%% file: output.tex
\documentclass[journal]{IEEEtran}
\input{preamble}

\usepackage{orcidlink}

\newcommand{\pfgm}{\rho_{x \leftrightarrow y}}
\newcommand{\pfgu}{\rho_{x \nleftrightarrow y}}

\newacronym{GNN}{GNN}{graph neural network}

\newcommand{\tildepfgm}{{\tilde \rho}_{x \leftrightarrow y}}
\newcommand{\tildepfgu}{{\tilde \rho}_{x \nleftrightarrow y}}

\begin{document}

\title{RobustMat: Neural Diffusion for Street Landmark Patch Matching under Challenging Environments}

\author{Rui~She*,
        Qiyu~Kang*,
        Sijie~Wang*, 
        Yu\'an-Ru\`i~Y\'ang,
        Kai~Zhao,
        Yang~Song,
        and
        Wee~Peng~Tay,~\IEEEmembership{Senior~Member,~IEEE}
%IEEE Publication Technology,~\IEEEmembership{Staff,~IEEE,}
        % <-this % stops a space
\thanks{*R. She, Q. Kang and S. Wang contributed equally to this work. 
This research is supported by A*STAR under its RIE2020 Advanced Manufacturing and Engineering (AME) Industry Alignment Fund – Pre Positioning (IAF-PP) (Grant No. A19D6a0053), and the Singapore Ministry of Education Academic Research Fund Tier 2 grant MOE-T2EP20220-0002.
% This research is supported by the Singapore Ministry of Education Academic Research Fund Tier 2 grant MOE-T2EP20220-0002 and the National Research Foundation, Singapore and Infocomm Media Development Authority under its Future Communications Research \& Development Programme.
}% <-this % stops a space
\thanks{
        R. She, Q. Kang, S. Wang, Y. Y\'ang, K. Zhao, Y. Song, and W. P. Tay are with Nanyang Technological University, Singapore.  
        {\{rui.she@; qiyu.kang@; wang1679@e.; yuanrui.yang@; kai.zhao@;  songy@; wptay@\}ntu.edu.sg}
        %$^{1}$Albert Author is with Faculty of Electrical Engineering, Mathematics and Computer Science,
        %University of Twente, 7500 AE Enschede, The Netherlands
        %{\tt\small albert.author@papercept.net}}%
        } % stops a space
%\thanks{Manuscript received April 19, 2021; revised August 16, 2021.}
%\thanks{Manuscript received April 19, 2021; revised August 16, 2021.}
}

% The paper headers
%\markboth{Journal of \LaTeX\ Class Files,~Vol.~14, No.~8, August~2021}%
%{Shell \MakeLowercase{\textit{et al.}}: A Sample Article Using IEEEtran.cls for IEEE Journals}

%\IEEEpubid{0000--0000/00\$00.00~\copyright~2021 IEEE}
% Remember, if you use this you must call \IEEEpubidadjcol in the second
% column for its text to clear the IEEEpubid mark.

\maketitle

\begin{abstract}
For autonomous vehicles (AVs), visual perception techniques based on sensors like cameras play crucial roles in information acquisition and processing. 
In various computer perception tasks for AVs, it may be helpful to match landmark patches taken by an onboard camera with other landmark patches captured at a different time or saved in a street scene image database. 
To perform matching under challenging driving environments caused by changing seasons, weather, and illumination, we utilize the spatial neighborhood information of each patch. 
We propose an approach, named \textit{RobustMat}, which derives its robustness to perturbations from neural differential equations. 
A convolutional neural ODE diffusion module is used to learn the feature representation for the landmark patches. A graph neural PDE diffusion module then aggregates information from neighboring landmark patches in the street scene. Finally, feature similarity learning outputs the final matching score. 
Our approach is evaluated on several street scene datasets and demonstrated to achieve state-of-the-art matching results under environmental perturbations. 
\end{abstract}


\section{Introduction}\label{sec:introduction}
\IEEEPARstart{I}{mage} matching \cite{wang2020deep,quan2021multi,she2023image} plays a vital role in autonomous driving and advanced driving assistant systems. 
In the context of street scene applications related to autonomous driving, landmarks such as traffic signs, traffic lights, and roadside poles can be identified as semantic objects \cite{wang2020robust}. Image patches captured at different locations that contain the same object can be matched using image patch matching techniques. This kind of matching is crucial in practical applications like loop-closure detection \cite{wang2020robust}, visual place recognition \cite{sunderhauf2015place}, camera relocalization for vehicles \cite{Wang2022robustloc}, and simultaneous localization and mapping (SLAM) \cite{Kang22itsc}.

\begin{figure}[!htb]
\centering
\includegraphics[width=\linewidth]{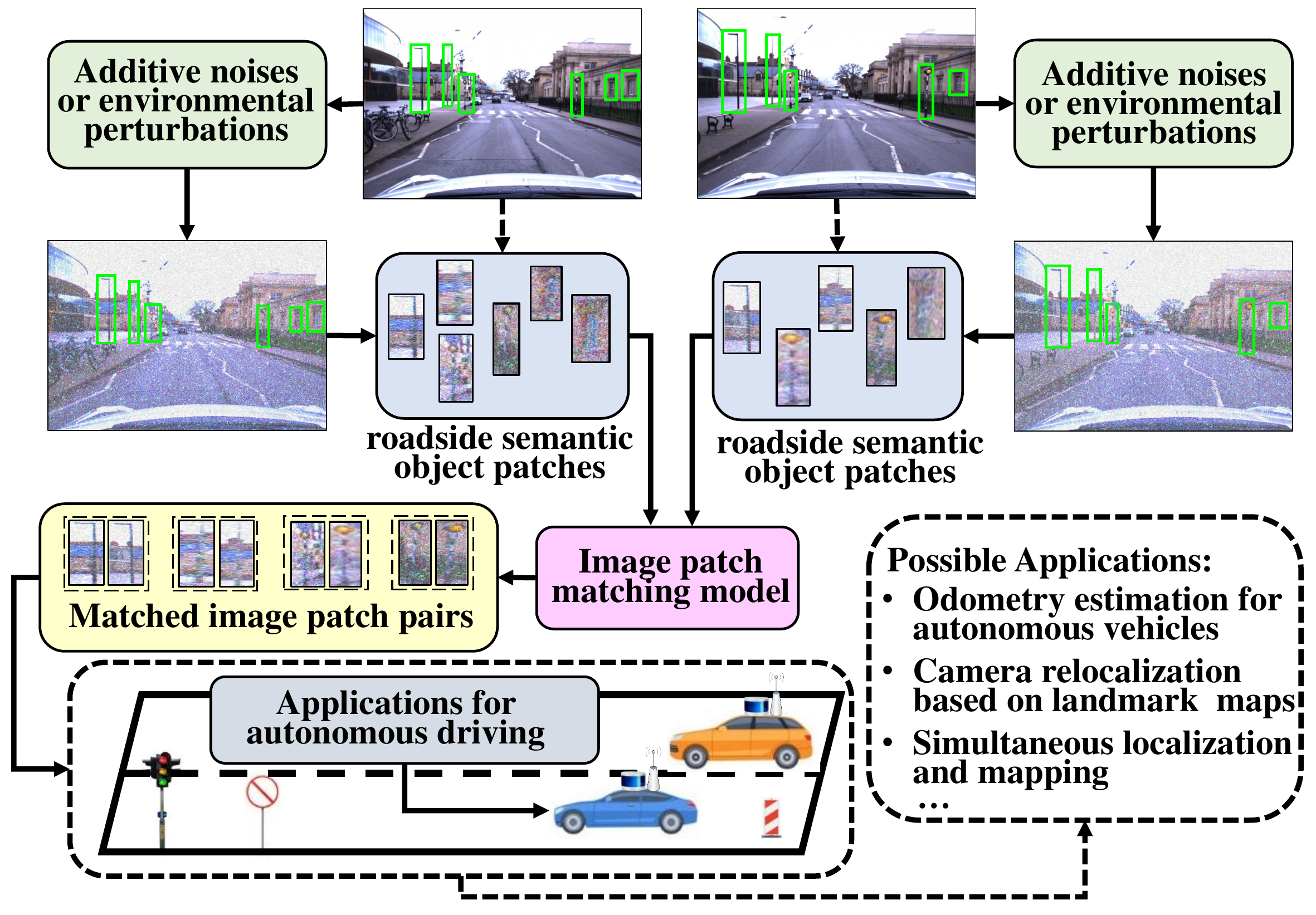}
\caption{The diagram of image patch matching for street scene landmarks under a noisy environment. The landmark patches in the green bounding boxes include traffic signs, traffic lights, poles, and windows.}
\label{fig:figure_overview}
\vspace{-0.5cm}
\end{figure}

In conventional matching approaches, handcrafted features based on local pixel statistics or gradient differences, such as SIFT, SURF, ORB and FREAK, are often used \cite{rublee2011orb,alahi2012freak}. The matching similarity is measured by the $\mathcal{L}_2$ distance, cosine distance, or other metrics. However, these approaches are not robust under varying environments, illuminations, and transformations, causing unstable matching performance \cite{tian2017l2,Wang2023HypLiLoc}. Deep learning methods such as \cite{han2015matchnet, quan2019afd, sarlin2020superglue, sun2021loftr} have been exploited to represent the image features but are still susceptible to noise and perturbations. 

In the joint feature and metric learning methods \cite{subramaniam2018ncc,quan2019afd}, the high-level feature representations and their similarity metrics are simultaneously learned. The feature descriptor learning methods \cite{tian2019sosnet,wang2019better,ng2020solar,miao2021learning} try to learn easily distinguishable high-level features under a predefined similarity metric. 
Keypoint correspondence learning methods like D2-Net \cite{dusmanu2019d2}, ASLFeat \cite{luo2020aslfeat} and SuperGlue \cite{sarlin2020superglue}, which achieve keypoint-level correspondence for the input image pairs, can use the keypoint matching score to perform image matching indirectly. 
In autonomous driving in cluttered environments, such as in the downtown area or central business districts (CBD), dynamic objects like vehicles and pedestrians can negatively impact pixel- or point-level matching. There may be more matched pixels from dynamic objects across different street scene images than from static landmarks. Such matching can introduce noise into applications such as place recognition. 
On the other hand, landmark patch-level matching can be more robust since one is performing matching based on identified static objects like traffic lights, signs, poles, and windows. 
We also note that matching based on single landmarks is unlikely to produce accurate results for applications like place recognition. Therefore, our approach is to make use of multiple landmarks and their \emph{relationships} with each other to learn a unique ``fingerprint''. Experimental results in \cref{sec:ablation} bear out this observation. 

In this paper, we focus on matching image patches of static semantic objects shown as \cref{fig:figure_overview}. 
Matching models in the literature achieve good performance in ideal environments, e.g., with clear visibility.
However, in practice, image matching often needs to deal with environmental perturbations such as varied weather or seasonal conditions, and dramatic illumination changes. To this end, we exploit neural ordinary differential equation (ODE) \cite{chen2018neural, kang2021Neurips} and partial differential equation (PDE) \cite{chamberlain2021grand, song2022robustness} modules to improve perception robustness. 
We construct a neighborhood graph for each landmark patch and use the graph neural PDE embedding to enhance the landmark patch representation. Each landmark patch is regarded as a vertex in the graph, which includes the $K$-nearest neighbors (\gls{wrt} the pixel positions) of the patch. The final matching score is a weighted average of the landmark feature graph vertex embedding similarities. 
Our contributions are summarized as follows:
\begin{itemize}
\item We propose a robust patch-matching method for challenging street scene landmark patches with the aid of neural ODE/PDE diffusion learning modules and neighboring patch information aggregation in a \gls{GNN}.
\item Theoretical analysis on the robustness of the feature embedding modules provides insights into the effect of perturbations. 
\item We empirically demonstrate that our landmark patch-matching method achieves \gls{SOTA} performance compared to other benchmarks under challenging noisy street environments. 
\end{itemize}

The rest of this paper is organized as follows. In \cref{sect:relatedworks}, related works are introduced. Our model and its corresponding details are presented in \cref{sect:model} and the theoretical basis for our model design is discussed in \cref{sect:theory}. 
In \cref{sect:exper}, we conduct extensive experiments and ablation studies. 
We discuss possible computer vision applications in \cref{sect:application} and finally conclude in \cref{sect:conc}. 

\section{Related Work}\label{sect:relatedworks}

\subsection{Image matching}

Since current state-of-the-art image patch-matching methods mainly rely on deep learning or neural networks, we give a brief overview of such learning-based methods in the literature. 

As a conventional learning-based matching method, MatchNet \cite{han2015matchnet} utilized deep convolutional neural networks (CNNs) for image feature extraction and exploits fully connected (FC) layers to measure the similarity of feature pairs. 
Moreover, different network architectures, such as SiameseNet, Pseudo-SiameseNet, and the 2-channel network \cite{melekhov2016siamese}, were studied to improve the matching performance.
NCC-Net \cite{subramaniam2018ncc} with the normalized cross correlation (NCC) measured the feature similarity for smaller neighborhood regions of pixels of patches to achieve matching robustness to illumination changes. 
AFD-Net \cite{quan2019afd} improved the discrimination of neural networks by aggregating multi-level features. 
DELG \cite{cao2020unifying} unified global and local features to obtain an efficient image feature representation.
CVNet \cite{lee2022correlation} introduced deep 4D convolutional layers and dense feature correlation to achieve image matching. 
Furthermore, feature matching based on images can also be utilized for visual localization, as demonstrated by CRBNet \cite{zhang2022leveraging}, which employs local and global descriptors in parallel to efficiently acquire correspondences.

Loss functions also have an impact on the feature descriptor learning for the image-matching task. Pairwise loss and triplet loss were introduced to train the descriptor networks \cite{kumar2016learning,balntas2016learning}.
The regularization of the loss function was studied in \cite{zhang2017learning}, which maximized the spread of local feature descriptors to obtain a better image-level embedding.
L2-Net \cite{tian2017l2} utilized a progressive sampling strategy to input more samples within a few epochs. 
To make the feature descriptor more efficient, HardNet \cite{mishchuk2017working} made use of the hard negative samples that are the closest negative samples far away from the closest positives one in a batch. 
Based on the exponential siamese and triplet losses, Exp-TLoss \cite{wang2019better} was also designed for hard sample learning. 
SOSNet \cite{tian2019sosnet} introduced second order similarity (SOS) to the loss function to refine the local feature learning. 
SOLAR \cite{ng2020solar} used the second-order spatial information and descriptor similarity to reweight feature maps and learn global descriptors.
HyNet \cite{tian2020hynet} presented a new local descriptor and a hybrid similarity measure for the triplet margin loss to improve the matching performance. 
The topology consistent descriptors (TCDesc) \cite{pan2021tcdesc} was designed based on the neighborhood information of descriptors, which could be trained by the triplet loss.

Keypoint and pixel-level correspondence methods such as LIFT \cite{yi2016lift} and SuperPoint \cite{detone2018superpoint} were also used for the matching task. 
When there are enough matched keypoints or pixels, a pair of images or patches are considered to be matched. 
Recently, D2-Net \cite{dusmanu2019d2} was designed based on a kind of CNN to detect and describe dense features simultaneously. 
Using the framework of D2-Net, ASLFeat \cite{luo2020aslfeat} provided effective modifications to estimate local shapes and keypoint locations more accurately.
SuperGlue \cite{sarlin2020superglue} made use of GNNs and the Sinkhorn algorithm to achieve feature matching based on keypoints. 
Moreover, LoFTR \cite{sun2021loftr} introduced self- and cross-attention layers into the transformers to obtain semi-dense matching. 
In addition, to improve image matching performance, spatial information is also considered in \cite{brogan2021fast,sarlin2020superglue,sun2021loftr}. 

The image patch matching task under additive perturbations, like white Gaussian noise and image
corruption, such as those due to rain and spatter, is still a challenging task that has not been addressed comprehensively. 
Several works in image matching have investigated the robustness of descriptors for image-level or local features \cite{chen2022robust,jiang2022robust,chen2021statenet,huang2020robust,jiang2020robust,yang2019robust} and that of keypoint-level matching \cite{mousavi2022two,wiles2021co,li2019rift} to changes in viewpoint, illumination, context and style.
The intrinsic robustness of patch-level matching models from the perspective of theoretical interpretability has not been sufficiently addressed.
To mitigate perturbations on images, several image denoising or restoration methods have been studied \cite{zamir2022restormer,wang2022uformer,YananZhao2023,tu2022maxim,zamir2021multi,chen2021hinet,gu2019self} as preprocessing steps. 
Although a two-stage approach can be a possible solution for image-matching in the presence of perturbations, the preprocessing step may introduce artifacts.
Therefore, in this paper, we investigate a single-stage patch-matching approach that is intrinsically robust, can be interpreted theoretically, and does not introduce unexpected artifacts.

\begin{figure*}[!htb]
\centering
\fbox{\includegraphics[width=0.98\linewidth, height=0.41\textheight]{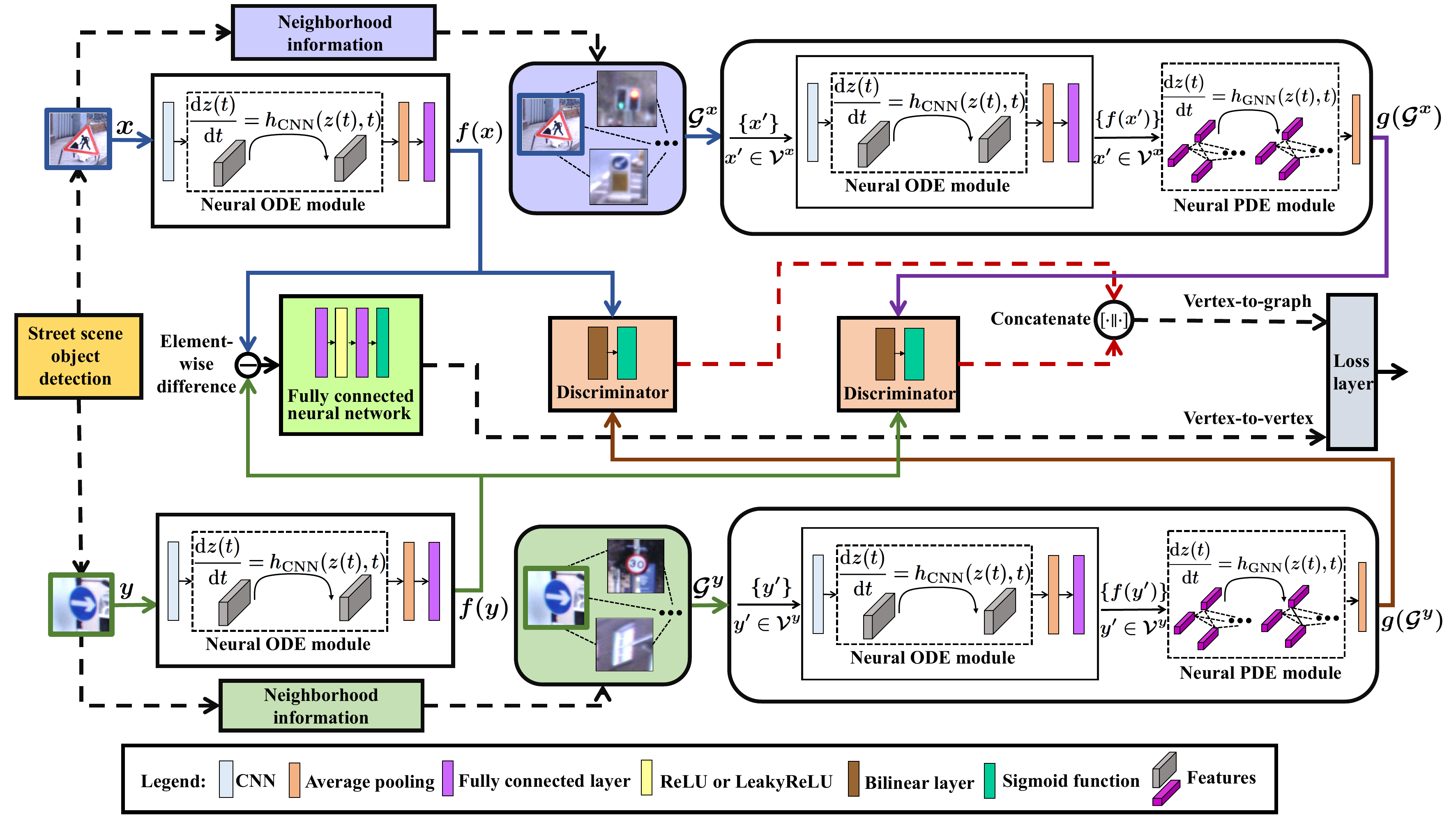}}
\caption{RobustMat: robust landmark patch matching model with the neural diffusion learning. 
The modules for neural ODE/PDE and the discriminators are shared networks, respectively.} 
\label{fig:model}
\label{fig_Matching_model}
\vspace{-0.5cm}
\end{figure*}

\subsection{Neural Diffusion and Robustness}

The neural diffusion models are based on ordinary or partial differential equations \cite{chen2018neural,song2022robustness}. 
In a neural ODE, the relationship between the input $z(0)$ and output $z(T)$ is given by 
\begin{align}
\ddfrac{z(t)}{t}=h_{\mathrm{ODE}}(z(t), t), \label{eq:ode}
\end{align}
where $h_{\mathrm{ODE}}:\Real^n \times [0, T)\to \Real^n$ is a trainable neural network, $z(t): [0, T) \to \Real^n$ denotes the state of the neural ODE at time $t$, and the terminal time $T \in [0,\infty)$.
The output $z(T)$ of the neural ODE layer is given by 
\begin{align}
z(T) = z(0) + \int_{0}^{T} h_{\mathrm{ODE}}(z(t), t) {\ud}t.
\end{align}
Graph neural PDEs are used for more stable representations of the graph features. In practice, neural ODE solvers in \cite{chen2018neural} are used for the graph neural PDEs to obtain approximate solutions \cite{chamberlain2021grand}.  

In the literature, neural ODEs have been demonstrated to be more robust against perturbations than vanilla deep neural networks like CNN, and it is possible to achieve Lyapunov stability \cite{kang2021Neurips}. 
Furthermore, graph neural PDE modules have their advantages in robustifying feature representations \cite{song2022robustness}. 

\section{Robust Landmark Image Patch Matching}\label{sect:model}

In this section, we describe in detail the different modules of our model.

A semantic object of the street scene, such as a traffic sign, traffic light, roadside pole, and window, is regarded as a landmark. An image (also called a \textit{frame} in this paper) is captured by a camera onboard a vehicle. An image patch corresponding to a semantic object is called a \textit{landmark patch}.
We assume that the detection and extraction for landmark patches from frames are available, e.g., by using techniques like Faster R-CNN \cite{ren2016faster}.  
A pair of image patches $(x,y)$ are said to be matched if they contain the \emph{same} landmark object in the field of view. This event is denoted as $x\leftrightarrow y$. The unmatched event is denoted as $x\nleftrightarrow y$.

We call our image patch matching approach \textit{Robust Matching (RobustMat)}. We first present the overall pipeline and then the details for each module. We also discuss the objective function used in RobustMat. 

\subsection{Overall Pipeline of RobustMat}\label{sect:pipeline}
The main modules in the RobustMat pipeline are as follows:
\begin{enumerate}
    \item Within a frame, the landmark patches are represented as vertices of a graph. The center pixel location of a landmark patch in the image plane is regarded as the position of the vertex. For an image patch or vertex $x$, its $K$ nearest neighbors (\gls{wrt} the $\calL_2$ distance of the positions) are acquired. A complete graph between the $K$ nearest neighbors is constructed as the neighborhood graph of $x$ and denoted by $\calG^x=\{\calV^x, \calE^x\}$, where $\calV^x$ contains the $K$ nearest neighbors and $\calE^x$ is the edge set.
    Correspondingly, for an image patch $y$ in another frame, the neighborhood graph $\mathcal{G}^{y}$ is constructed.
    \item To determine whether the two landmark patches $x$ and $y$ are matched, neural ODE and graph neural PDE layers are applied to obtain feature representations based on $\mathcal{G}^{x}$ and $\mathcal{G}^{y}$. The neural ODE is used for feature map self-updating within each vertex, while the graph neural PDE is used for aggregating neighboring patch information in $\mathcal{G}^{x}$ and $\mathcal{G}^{y}$. 
    \item  Using the learned feature representations for $x$ and $y$, we finally use similarity layers to determine whether the two landmark patches $x$ and $y$ are matched or not.
\end{enumerate}
RobustMat is illustrated in \cref{fig_Matching_model}.
More details are given as follows.

\subsection{Model Details}
In this subsection, we present details on each module in RobustMat, the loss function used for training, and the matching score used in testing.

\subsubsection{Vertex-diffusion Embedding Based on Neural ODE} 
Following \cite{chen2018neural} in processing images, we first downsample the input image patches to a feature map with smaller width and height using a downsampling CNN module. For each patch $x$, the feature map from the downsampling CNN is denoted by $h_\mathrm{DS}(x)\in \mathbb{R}^{c\times H\times W}$, where $c$ is the number of channels, $H$ is the feature map height, and $W$ is the feature map width. We then apply a CNN-based neural ODE to achieve stable feature map self-updating within each vertex. Formally, the vertex feature update is given by 
\begin{align}
\ddfrac{z(t)}{t}=h_\mathrm{CNN}(z(t),t), \label{eq:NODE}
\end{align}
where $h_\mathrm{CNN}(\cdot)$ denotes a CNN module applied to the feature map in $\mathbb{R}^{c\times H\times W}$. The above ODE has initial state $z(0)=h_\mathrm{DS}(x)$. The solution of \cref{eq:NODE} at time $T_f$, denoted by $z(T_f)\in\mathbb{R}^{c\times H\times W}$, can be obtained by using differential equation solvers \cite{chen2018neural} by integrating $h_\mathrm{CNN}(\cdot)$ from $t=0$ to $t=T_f$. 
Finally, an average pooling operation is applied to the output and subsequently passed into an FC layer to obtain the final vertex embedding.

To make it clearer, let $F_\mathrm{CNN}(z(0))$ denote the solution of \cref{eq:NODE} at $t=T_f$ with the initial condition $z(0)$.
We use the \emph{single notation}
\begin{align}\label{vertexembedding}
    f(x) = h_\mathrm{fc}(h_\mathrm{fp}(F_\mathrm{CNN}(h_\mathrm{DS}(x)))) \in \Real^{n_f}
\end{align}
to denote the full vertex-diffusion embedding module including the downsampling $h_\mathrm{DS}$, the CNN-based neural ODE $F_\mathrm{CNN}$, the average pooling $h_\mathrm{fp}$ and the FC layer $h_\mathrm{fc}$. 

\subsubsection{Graph-diffusion Embedding Based on Neural PDE} 
For a patch $x$ and its neighborhood graph $\calG^x$, we first apply the vertex-diffusion module  $f$ to all the vertices $\{x'\}_{x'\in \calV^x}$ in the graph to obtain the embeddings $\set{f(x') \given x'\in \calV^x}$. Then, to aggregate neighboring patch information in the graph $\calG^x$, we apply a graph neural PDE feature updating module \cite{chamberlain2021grand, song2022robustness} for $\calG^x$ with the initial features $z(0)=\set{f(x') \given x'\in \calV^x}$ as follows:
\begin{align}
\ddfrac{z(t)}{t}=h_\mathrm{GNN}(z(t),t), \label{eq:GRAPH}
\end{align}
where $h_\mathrm{GNN}(\cdot)$ is generally regarded as a layer of a GNN, e.g., graph attention network (GAT) or graph convolutional network (GCN). Essentially, \cref{eq:GRAPH} is an ODE. However, as pointed out by \cite{chamberlain2021grand}, many popular GNN architectures can be viewed as partial differential diffusion equations with different discretization schemes, and therefore \cref{eq:GRAPH} is generally considered as a neural PDE module. 
Similar to \cref{eq:NODE}, we solve \cref{eq:GRAPH} using \cite{chen2018neural}. 

The output of the graph neural PDE at time $T_g$ is given by $F_\mathrm{GNN}(\set*{f(x') \given x'\in \calV^x})$,
where $F_\mathrm{GNN}(\cdot)$ denotes the solution of \cref{eq:GRAPH} at $t=T_g$ by integrating $h_\mathrm{GNN}(\cdot)$ from $t=0$ to $t=T_g$.
Finally, the average pooling $h_\mathrm{gp}$ of the vertex features in $F_\mathrm{GNN}(\set{f(x') \given x'\in \calV^x})$ over all vertices is considered as the final graph embedding: 
\begin{align}
g(\calG^x) = h_\mathrm{gp}((F_\mathrm{GNN}(\set*{f(x') \given x'\in \calV^x}))) \in \Real^{n_f}. \label{eq:g_graph}
\end{align}
Analogously, the graph embedding $g(\calG^y)$ is obtained following the same procedure for the patch $y$ in another frame together with its neighborhood graph $\calG^y$.

\subsubsection{Correspondence Comparison}
In our model, we have two feature comparisons for \emph{vertex-to-vertex} correspondence and \emph{vertex-to-graph} correspondence. 

In the \emph{vertex-to-vertex} comparison module $r$, we use FC layers with a sigmoid function in the last layer. Its input is the element-wise squared difference between the vertex-diffusion features $f(\cdot)$ for the two patches $x$ and $y$ from different frames. The output of this module is
\begin{align}
   r\left((f(x)-f(y))^2\right)\in [0,1],
\end{align}
where $(\cdot)^2$ is element-wise squaring. 

In the \emph{vertex-to-graph} comparison module, for the patch pair $(x,y)$, the vertex-diffusion feature $f(x)$ is compared with the graph-diffusion embedding $g(\mathcal{G}^{y})$ using a discriminator $d$. The discriminator $d$ is based on a \emph{bilinear} layer, whose form is given by
\begin{align}
d(a,b)= \sigma(a^{\top}M b), \label{eq:ddd}
\end{align}
where $M\in \Real^{n_f\times n_f}$ and $\sigma(\cdot)$ denote the trainable matrix and the sigmoid function, respectively. We thereby obtain a score $d(f(x),g(\mathcal{G}^{y}))$.
The same procedure is applied for the comparison between $f(y)$ and $g(\mathcal{G}^{x})$ to yield the score $d(f(y),g(\mathcal{G}^{x}))$. 
The final similarity score for the patch pair $(x,y)$ is the output average of $r$ and $d$, which we describe in the following subsection.

\subsubsection{Loss Function and Matching Score}\label{sect:loss_score}
The \emph{vertex-to-vertex loss} function for the similarity model is the empirical cross-entropy loss
\begin{align}
    \ell_{\mathrm{vv}} 
    = \ofrac{|\calM|} {\sum_{(x,y)\in\calM}\ell_{\mathrm{ce}}\parens*{r\left((f(x)-f(y))^2\right),\indicator{x\leftrightarrow y}}},
\end{align}
where $\calM$ is the set of patch training pairs, $\ell_{\mathrm{ce}}(\cdot,\cdot)$ is the  cross entropy function, and $\indicator{\cdot}$ is the indicator function.
We also include a \emph{vertex-to-graph loss} defined as 
\begin{align}
    \ell_{\mathrm{v\calG}}
    & = - \frac{1}{2} \braces*{\hat{L}_d(f(x),g(\mathcal{G}^{y})) 
           + \hat{L}_d(f(y),g(\mathcal{G}^{x})) },  \label{eq.Lce_ID} 
\end{align}
where 
\begin{align}
& \hat{L}_d(f(x),g(\mathcal{G}^{y}))  \nn
& = \frac{1}{|\calM|} \sum_{(x,y)\in \calM} \Big\{ \indicator{x \leftrightarrow y} \log d(f(x),g(\mathcal{G}^{y})) \nn
& \qquad + \indicator{x \nleftrightarrow {y}} \log \parens*{1-d(f(x),g(\mathcal{G}^{y}))} \Big\}.
\end{align}
By minimizing $\ell_{\mathrm{v\calG}}$, we \emph{maximize} the information distance (cf.\ the supplementary material for further details). 
%(cf.\ \cref{sect:theory} for further details). 

The \emph{total loss} function is given by
\begin{align}\label{eq.L_loss}
    \ell_{\mathrm{total}}
    & = (1-\lambda)\ell_{\mathrm{vv}} + \lambda \ell_{\mathrm{v\calG}},
\end{align}
where $\lambda>0$ is a predefined weight parameter ($\lambda= 0.5$ is chosen in this paper).

In the testing phase, we set the final matching score as 
\begin{align}
    S_{\mathrm{match}}(x,y) 
    & = r(f(x),f(y)) \nn
    & \quad +\ofrac{2}\parens*{d(f(x),g(\calG^y))+d(f(y),g(\calG^x))}.\label{eq:ssss}
\end{align}
 A larger score than a predefined threshold can be regarded as a matched prediction.

\section{Theoretical Analysis}\label{sect:theory}

In this section, we theoretically analyze the effects of additive perturbations on image patches on our model, leading to motivations for our model design. 
Our results require Lipschitz continuity assumptions and other mild technical conditions. 
Specifically, if the input is perturbed, the embeddings from the neural vertex-diffusion and graph-diffusion modules have the same order of magnitude of perturbation as the input (cf.\ \cref{thm.Robustness_ode} and \cref{thm.Robustness_gde}). 

Let ${\tilde x}$ be the perturbed version of $x$ and $h_\mathrm{DS}({\tilde x}) = h_\mathrm{DS}(x) + \varepsilon_{DS}$, where $\varepsilon_{DS}$ denotes the  perturbation on the feature map. 
Denote $\varepsilon := \| \varepsilon_{DS} \|_2$, where $\| \cdot \|_2$ is the Frobenius norm. 
For a tensor $w\in\Real^{c\times H \times W}$, we index its elements using a predefined ordering. For $i=1,\dots,c\times H\times W$, let $w\tc{i}$ denote the $i$-th element in the tensor $w$.

\begin{Assumption}\label{ass:ass_ivp}
Suppose $\| h_\mathrm{DS}({\tilde x}) - h_\mathrm{DS}(x) \|_2 = \varepsilon$. 
Let $h_{\mathrm{CNN}}(\cdot, \cdot)$ in \cref{eq:NODE} be Lipschitz continuous. 
Assume that $z(t)$ in \cref{eq:NODE} is restricted to an open bounded set $\Omega \subset \Real^{c\times H \times W}$ for all $t\in[0,T_f]$ and
\begin{align}
    \frac{\partial h_{\mathrm{CNN}}^{(i)}(u, t)}{\partial u^{(i)}} \le -C^*_z, \label{eq.partial_f_CNN}
\end{align}
for all $u\in\Omega$, where $C_{z}^*$ is a positive parameter.
\end{Assumption}

\begin{Assumption}\label{ass:ass_fc}
We have $\| h_{\mathrm{fc}}(\omega_1)- h_{\mathrm{fc}}(\omega_2)  \|_2 \le C_{\mathrm{fc}} \| \omega_1-\omega_2 \|_2$ and 
$\| h_{\mathrm{fp}}(\omega_a)- h_{\mathrm{fp}}(\omega_b)  \|_2 \le C_{\mathrm{fp}} \| \omega_a-\omega_b \|_2$,
where $C_{\mathrm{fc}}$ and $C_{\mathrm{fp}}$ are positive constants. 
\end{Assumption}
\Cref{{ass:ass_fc}} is a common assumption \cite{yan2019robustness,dupont2019augmented}, which indicates that the FC layer and the pooling layer are Lipschitz continuous. 

\begin{Theorem}[Perturbation on the vertex-diffusion embedding]
\label{thm.Robustness_ode}
Suppose \cref{ass:ass_ivp,ass:ass_fc} hold. 
For a sufficiently small $\varepsilon>0$, the effect of the perturbation on the vertex-diffusion embedding $f$ in \cref{vertexembedding} is given by\footnote{For non-negative $f(u)$ and $g(u)$, we write $f(u)=o(g(u))$ if $\lim_{u \to 0} \frac{f(u)}{g(u)} = 0$, and $f(u)=\calO(g(u))$ if $\lim_{u \to 0} \frac{f(u)}{g(u)} < \infty$.} 
\begin{align}
    \norm*{f(\tilde x) - f(x)}_2 \le  \calO(\varepsilon\exp{(-C_{z}^* T_f)}). \label{eq.f_inequality}
\end{align}
\end{Theorem}
\begin{proof}
See Appendix~\ref{app:thm.Robustness_ode}.
\end{proof}

\Cref{thm.Robustness_ode} shows that the vertex-diffusion embedding does not amplify the order of magnitude of the perturbation in the input feature map. In other words, the vertex-diffusion module is stable \gls{wrt} its downsampled input. 

We next discuss the perturbation on the graph-diffusion embedding. 
Let $Z(t) = \parens{z(v,t)}_{v \in \calV^{x}} \in \Real^{|\calV^{x}|\times n_f}$ denote the vertex embeddings stacked together with $z(v,t)$ being the diffusion embedding for vertex $v$ at time $t\in [0,T_g]$. 
Taking the vertex-diffusion features $(f(v))_{v \in \calV^x}$, the graph-diffusion module applies $Z(0)=(f(v))_{v \in \calV^x}$ to \cref{eq:GRAPH} to obtain an output embedding $Z(T_g)$.

\begin{Assumption}\label{ass:ass_Amatrix}
Let $h_{\mathrm{GNN}}(\cdot, \cdot)$ in \cref{eq:GRAPH} and $h_\mathrm{gp}(\cdot)$ in \cref{eq:g_graph} be Lipschitz continuous. Assume the trajectory $Z(t)$ based on \cref{eq:GRAPH} is restricted to an open bounded set $\Theta \subset \Real^{|\calV^{x}|\times n_f}$ for all $t\in[0,T_g]$ and 
\begin{align}\label{Alip_condition}
    \frac{\partial h_{\mathrm{GNN}}^{(i)}(Z, t)}{\partial Z^{(i)}} \le -C^*_A, 
\end{align}
for all $Z\in\Theta$, where $C_{A}^*$ is a positive parameter.
\end{Assumption}

\begin{Theorem}[Perturbation on graph-diffusion embedding]
\label{thm.Robustness_gde}
For a patch $x$ and its perturbation $\tilde{x}$ corresponding to $v \in \calV^x$ and $\tilde{v} \in\calV^{\tilde{x}}$, suppose \cref{ass:ass_ivp,ass:ass_fc,ass:ass_Amatrix} hold.
Then, for a sufficiently small $\varepsilon$, the graph-diffusion embedding $g(\cdot)$ satisfies
\begin{align}
    \norm*{g(\calG^{\tilde x}) - g(\calG^x)}_2 \le  \calO(\varepsilon\exp{( -C_A^* T_g)}). \label{eq.g_inequality}
\end{align}
\end{Theorem}

\begin{proof}
See Appendix~\ref{app:thm.Robustness_gde}.
\end{proof}

From \cref{thm.Robustness_gde}, we see that the pipeline consisting of the vertex-diffusion module followed by the graph-diffusion module is stable \gls{wrt} its downsampled input.

\section{Experiments}\label{sect:exper}
\label{headings}
\subsection{Dataset preparation}\label{sect:Datasets}

\subsubsection{Synthetic Noise for Real Datasets}
We use the KITTI dataset \cite{geiger2012we} and the Oxford Radar RobotCar (abbreviated as ``Oxford'') dataset \cite{barnes2020oxford} to derive two landmark patch matching datasets, which consist of the image patches of roadside static landmark objects from sampled images (i.e., frames).
To imitate the challenging environmental perturbations, additive white Gaussian noise is added to each frame. Image corruption like rain and spatter, as well as brightness and saturation changes, are performed with the Python library ``imgaug''. 
Some examples are shown in \cref{fig_image_patches_pair}. 
To build the patch datasets, the object detection neural network Faster R-CNN \cite{ren2016faster} is exploited for landmark patch extraction, where the bounding box labels of the street scene landmark objects are manually included for training. 
More details of the patch-matching datasets are given as follows.\footnote{All codes and datasets are available at \url{https://github.com/AI-IT-AVs/RobustMat}.}

\textbf{KITTI Noisy Landmark Patches} or ``noisy KITTI''. 
The KITTI dataset is a public dataset with multi-sensor data for autonomous driving. 
The landmark objects consist of traffic lights, traffic signs, and poles. We add various types of noise, including additive white Gaussian noise, corruption, brightness and saturation changes, with the peak signal-to-noise ratio (PSNR) around $16$ dB. 
Furthermore, the minimum distance between two locations where image frames are captured is set to be $2$ m to avoid trivial matching due to insignificant landmark appearance changes. A maximum distance of $25$ m ensures the existence of common landmarks. 
In our experiments, $1500$ frames are selected for the landmark patch matching. The training and testing data are in the ratio $2:1$. 

\textbf{Oxford Noisy Landmark Patches} or ``noisy Oxford''. 
The Oxford dataset contains street scene images. 
The landmarks including traffic lights, traffic signs, poles, and windows based on building facades, are manually labeled for training.
We select $3000$ frames for the landmark patch matching in our experiment. 
The preparation of the dataset is similar to that for noisy KITTI.
%%%%-----------------------------------------------------
\begin{figure}[!htb]
\centering
\includegraphics[width=0.95\linewidth]{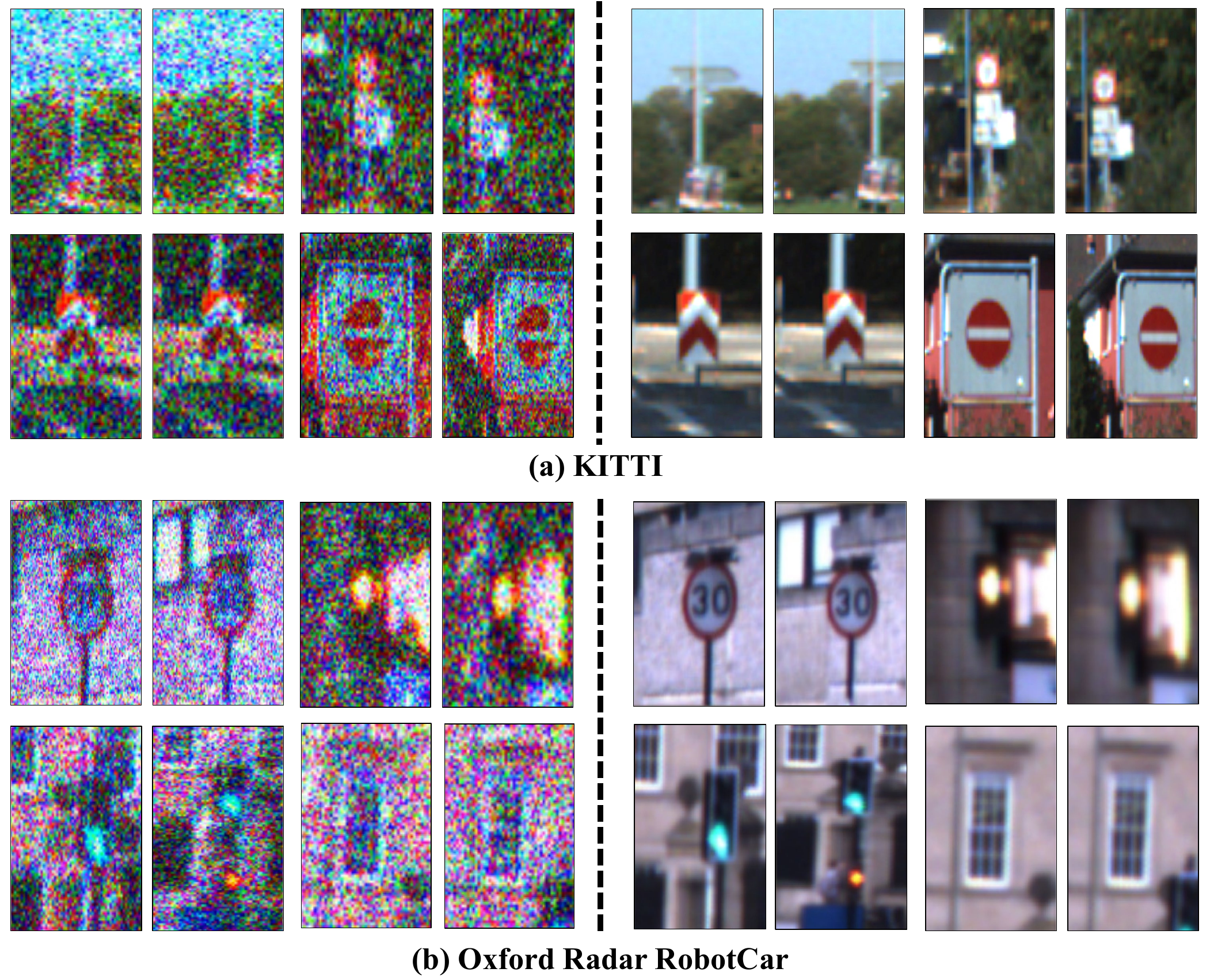}
% \vspace{-0.3cm}
\caption{Examples of noisy landmark patch pairs and their corresponding clean pairs from the KITTI and Oxford datasets.}
\label{fig_image_patches_pair}
% \vspace{-0.3cm}
\end{figure}
%%%%-----------------------------------------------------

\subsubsection{Authentic Noisy Datasets}
We utilize the Boreas datasets \cite{burnett2023boreas} to acquire authentic landmark patches with noise. Similar to the procedure for generating synthetic noisy data, we extract landmark patches from different driving environments within the Boreas datasets. 

\textbf{Boreas Sunny (Snowy/Rainy/Cloudy) Landmark Patches} or ``sunny (snowy/rainy/cloudy) Boreas''. The Boreas datasets contain diverse driving scenarios, including snowy, rainy, cloudy, and sunny conditions. Consequently, the snowy, rainy, and cloudy scenarios can be considered as instances with higher levels of noise compared to the sunny scenario. The selected landmarks consist of manually labeled objects such as traffic lights, traffic signs, poles, and windows.
In our experiment, we use $1000$ frames from the sunny scenario for training and $500$ frames from each of the snowy, rainy, and cloudy scenarios for testing. The landmark acquisition from the Boreas datasets is similar to that of the KITTI and Oxford datasets. Several examples are shown in \cref{fig_boreas_image_patches_pair}.

%%----------------------
\begin{figure}[!htb]
\centering
\includegraphics[width=0.95\linewidth]{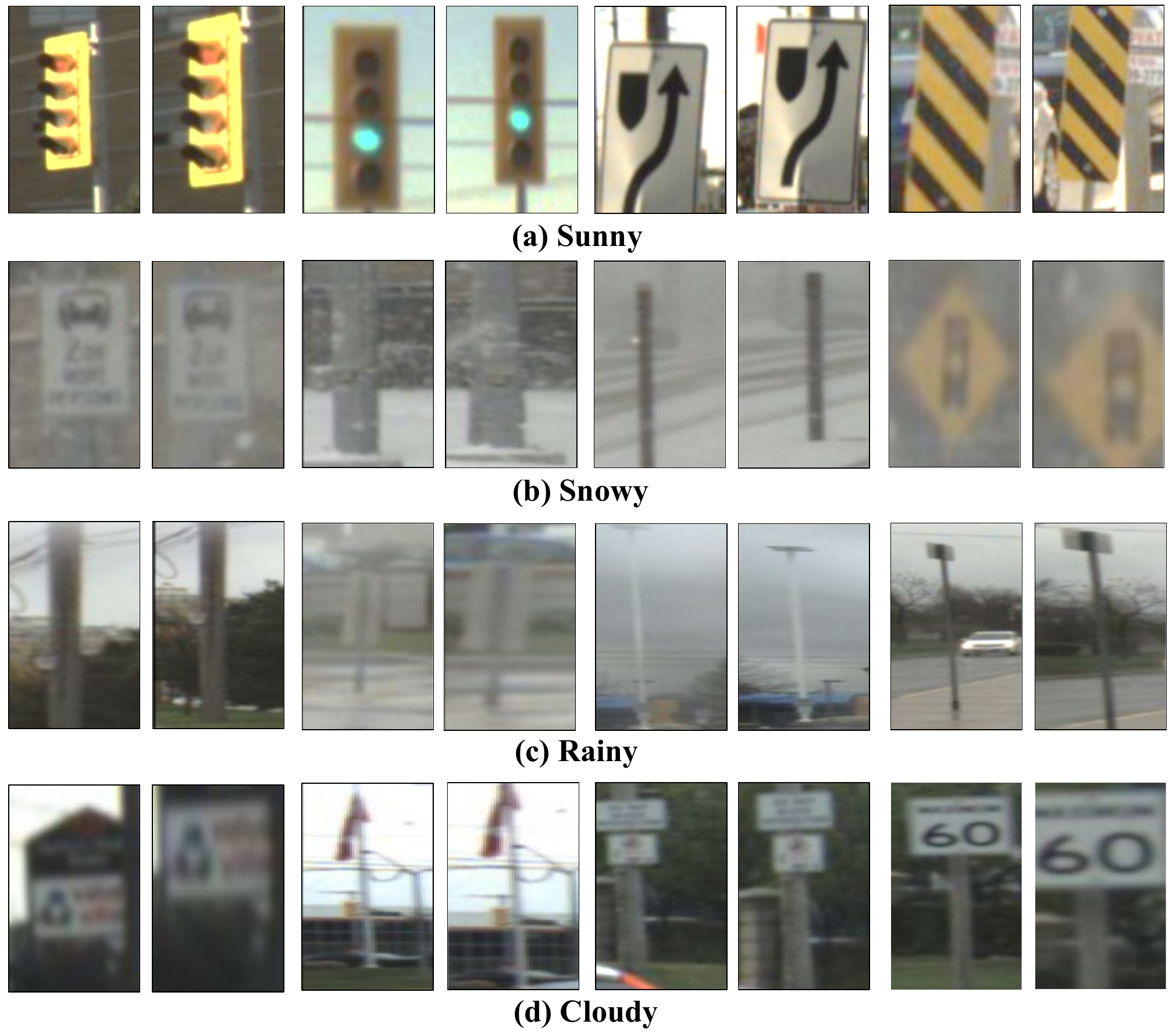}
% \vspace{-0.3cm}
\caption{
Examples of landmark patch pairs from the Boreas datasets.
}
\label{fig_boreas_image_patches_pair}
% \vspace{-0.3cm}
\end{figure}
%%----------------------
% %%%===================================================
\begin{table*}[!htb]
\scriptsize
%\tiny
\caption{
Matching performance on noisy KITTI. The best and the second-best results under different metrics are highlighted in \textbf{bold} and \underline{underlined}, respectively. Performance deterioration after denoising is highlighted in \blue{blue}.
} 
% \vspace{-0.2cm}
\label{Kitti-table}
\centering
\resizebox{0.98\textwidth}{!}{\setlength{\tabcolsep}{2pt} 
\newcommand{\tabincell}[2]{\begin{tabular}{@{}#1@{}}#2\end{tabular}}
\begin{tabular}{c|c|c|c|c|c|c|c|c|c |c|c}
\hline\hline
{\bf Methods}  & SiameseNet & MatchNet & HardNet & SOSNet & Exp-TLoss & HyNet & SOLAR & DELG & CVNet & \tabincell{c}{RobustMat \\ (GCN-PDE)\\} 
& \tabincell{c}{RobustMat \\ (GAT-PDE)\\} \\ 
\hline  
{\bf Precision}    & 0.8242 \tiny{$\pm$ 0.0066} & 0.9288 \tiny{$\pm$ 0.0032} & 0.8937 \tiny{$\pm$ 0.0024} & 0.9250 \tiny{$\pm$ 0.0012} & 0.9266 \tiny{$\pm$ 0.0015} & 0.9266 \tiny{$\pm$ 0.0013} & 0.9256 \tiny{$\pm$ 0.0011} 
& 0.9251 \tiny{$\pm$ 0.0054} & 0.9287 \tiny{$\pm$ 0.0026} & \textbf{0.9327} \tiny{$\pm$ 0.0026} & \underline{0.9308} \tiny{$\pm$ 0.0023} \\ 
{\bf Recall} & 0.7728 \tiny{$\pm$ 0.0298} & 0.8490 \tiny{$\pm$ 0.0222} & 0.8528 \tiny{$\pm$ 0.0212} & 0.8549 \tiny{$\pm$ 0.0151} & 0.8752 \tiny{$\pm$ 0.0187} & 0.8757 \tiny{$\pm$ 0.0170} & 0.8629 \tiny{$\pm$ 0.0141} 
& 0.7840 \tiny{$\pm$ 0.0138} & 0.7365 \tiny{$\pm$ 0.0150} & \underline{0.8811} \tiny{$\pm$ 0.0133} & \textbf{0.8970} \tiny{$\pm$ 0.0118} \\ 
{\bf $F_1$-Score}  & 0.7974 \tiny{$\pm$ 0.0169} & 0.8870 \tiny{$\pm$ 0.0130} & 0.8727 \tiny{$\pm$ 0.0123} & 0.8885 \tiny{$\pm$ 0.0088} & 0.9001 \tiny{$\pm$ 0.0106} & 0.9004 \tiny{$\pm$ 0.0096} & 0.8931 \tiny{$\pm$ 0.0080} 
& 0.8487 \tiny{$\pm$ 0.0103} & 0.8215 \tiny{$\pm$ 0.0098} & \underline{0.9061} \tiny{$\pm$ 0.9061} & \textbf{0.9136} \tiny{$\pm$ 0.0072} \\ 
{\bf AUC}          & 0.8040 \tiny{$\pm$ 0.0134} & 0.8269 \tiny{$\pm$ 0.0116} & 0.7744 \tiny{$\pm$ 0.0106} & 0.8234 \tiny{$\pm$ 0.0075} & 0.8336 \tiny{$\pm$ 0.0094} & 0.8339 \tiny{$\pm$ 0.0085} & 0.8275 \tiny{$\pm$ 0.0070} 
& 0.7968 \tiny{$\pm$ 0.0125} & 0.7835 \tiny{$\pm$ 0.0078} & \underline{0.8453} \tiny{$\pm$ 0.0088} & \textbf{0.8485} \tiny{$\pm$ 0.0079} \\ 
\hline\hline
{\bf Methods}  & \tabincell{c}{SiameseNet \\ +Denoising \\} & \tabincell{c}{MatchNet \\
+Denoising\\} & \tabincell{c}{HardNet \\ +Denoising\\} & \tabincell{c}{SOSNet\\+Denoising\\} & \tabincell{c}{Exp-TLoss\\+Denoising\\} & \tabincell{c}{HyNet \\+Denoising\\} & \tabincell{c}{SOLAR\\+Denoising\\} 
& \tabincell{c}{DELG \\ +Denoising\\} & \tabincell{c}{CVNet \\ +Denoising\\} & \tabincell{c}{RobustMat \\ (GCN-PDE)\\ +Denoising\\} 
& \tabincell{c}{RobustMat \\ (GAT-PDE)\\+Denoising\\} \\ 
\hline  
{\bf Precision}    & 0.8350 \tiny{$\pm$ 0.0118} & \blue{0.9136} \tiny{$\pm$ 0.0011} & 0.8968 \tiny{$\pm$ 0.0025} & \blue{0.9108} \tiny{$\pm$ 0.0018} & 0.9319 \tiny{$\pm$ 0.0014} & 0.9266 \tiny{$\pm$ 0.0018} & 0.9257 \tiny{$\pm$ 0.0012} 
& \underline{0.9334} \tiny{$\pm$ 0.0029} & 0.9305 \tiny{$\pm$ 0.0039} & \blue{0.9317} \tiny{$\pm$ 0.0033} & \textbf{0.9353} \tiny{$\pm$ 0.0031} \\ 
{\bf Recall} & 0.7760 \tiny{$\pm$ 0.0209} & 0.8624 \tiny{$\pm$ 0.0115} & 0.8811 \tiny{$\pm$ 0.0235} & 0.8720 \tiny{$\pm$ 0.0190} & 0.8763 \tiny{$\pm$ 0.0187} & 0.8763 \tiny{$\pm$ 0.0238} & 0.8640 \tiny{$\pm$ 0.0152} 
& 0.8224 \tiny{$\pm$ 0.0090} & 0.8496 \tiny{$\pm$ 0.0207} & \textbf{0.8949} \tiny{$\pm$ 0.0116} & \textbf{\blue{0.8949}} \tiny{$\pm$ 0.0120} \\ 
{\bf $F_1$-Score}  & 0.8042 \tiny{$\pm$  0.0121} & 0.8872 \tiny{$\pm$ 0.0062} & 0.8887 \tiny{$\pm$ 0.0132} & 0.8909 \tiny{$\pm$ 0.0108} & 0.9031 \tiny{$\pm$ 0.0106} & 0.9006 \tiny{$\pm$ 0.0134} & 0.8937 \tiny{$\pm$ 0.0086} 
& 0.8744 \tiny{$\pm$ 0.0060} & 0.8881 \tiny{$\pm$ 0.0119} & \underline{0.9129} \tiny{$\pm$ 0.0067} & \textbf{0.9147} \tiny{$\pm$ 0.0068} \\ 
{\bf AUC}          & 0.8112 \tiny{$\pm$ 0.0102} & \blue{0.8088} \tiny{$\pm$ 0.0050} & 0.7885 \tiny{$\pm$ 0.0118} & \blue{0.8080} \tiny{$\pm$ 0.0095} & 0.8421 \tiny{$\pm$ 0.0093} & 0.8341 \tiny{$\pm$ 0.0119} & 0.8280 \tiny{$\pm$ 0.0076} 
& 0.8232 \tiny{$\pm$ 0.0069} & 0.8296 \tiny{$\pm$ 0.0113} & \underline{0.8491} \tiny{$\pm$ 0.0078} & \textbf{0.8547} \tiny{$\pm$  0.0077} \\ 
\hline\hline
\end{tabular}}
% \vspace{-0.3cm}
\end{table*}
%%--------------------
\begin{figure*}[!htb]
\centering
\includegraphics[width=\linewidth]{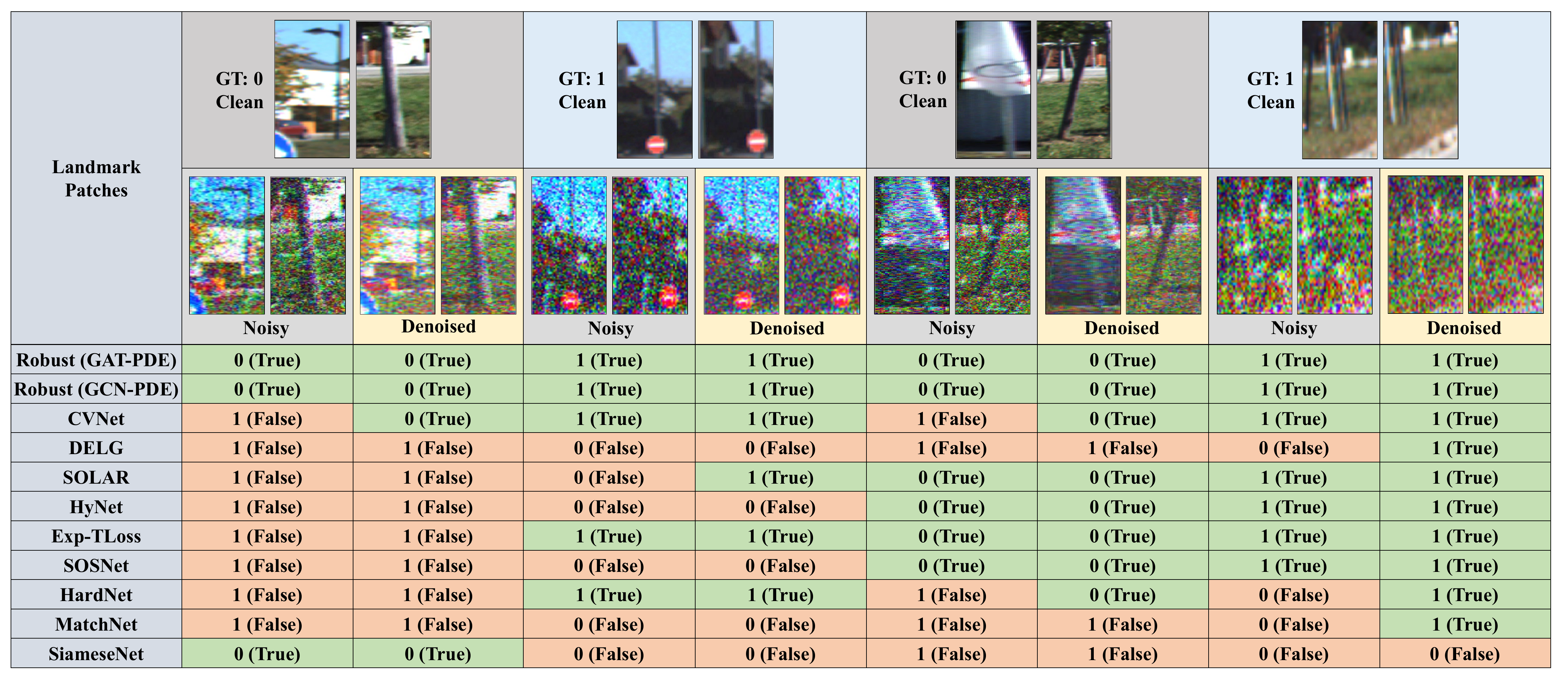}
\caption{
Matching prediction results for example landmark patch pairs from the noisy KITTI dataset. The prediction ``1'' or ``0'' indicates matched or unmatched. GT stands for ``ground truth''.
}
\label{fig_kitti_noise}
% \vspace{-0.5cm}
\end{figure*}
%%%===================================================

\subsection{Experimental Details}

\textbf{Model Setting.}
The size of an input landmark patch is normalized as $3 \times 256\times 256$.
We set $c=128$, $n_f=512$, and $H=W=64$ in \cref{eq:NODE} and \cref{eq:GRAPH}.
We set the GNN layer $h_\mathrm{GNN}$ in the graph neural PDE \cref{eq:GRAPH} to be GAT \cite{chamberlain2021grand} or GCN \cite{eliasof2021pde} and use $K=3$  nearest neighbors for each patch. 
The GAT-based graph neural PDE uses $4$ attention heads, each with $512$ hidden features in the GAT block. We call this model RobustMat (GAT-PDE).
For the GCN-based graph neural PDE, $512$ hidden features are set in the GCN block. This variant is called RobustMat (GCN-PDE).
The ``dopri5'' solver \cite{torchdiffeq} is used to solve the ODEs with $T_f=T_g=1.0$.
The Adam optimizer is adopted in our training. Its learning rate is set as $0.0001$. The number of training epochs is set as $150$.
More details of the model settings are provided in the supplementary. 

\textbf{Baseline Methods.}
We compare our model with several baseline methods including MatchNet \cite{han2015matchnet}, SiameseNet \cite{melekhov2016siamese}, HardNet \cite{mishchuk2017working}, SOSNet \cite{tian2019sosnet}, Exp-TLoss \cite{wang2019better}, HyNet \cite{tian2020hynet}, SOLAR \cite{ng2020solar}, DELG \cite{cao2020unifying} and CVNet \cite{lee2022correlation}.
MatchNet and SiameseNet are joint-feature-and-metric learning methods, while HardNet, SOSNet and Exp-TLoss make full use of similarity metrics to distinguish the learned features. 
SOLAR, HyNet, DELG, and CVNet combine feature descriptor learning and metric design for the matching task.  
We also exploit an image denoising method self-guided network (SGN) \cite{gu2019self}, to filter out noise in the input first and then perform image matching in a two-stage approach. Note that the denoising procedure is not perfect and often introduces artifacts like blurring into the output, as seen in \cref{fig_kitti_noise,fig_oxford_noise}.

\subsection{Performance Evaluation}\label{sec.Performance_Evaluation}
\textbf{Performance on noisy KITTI.}
To evaluate the methods, the statistics include mean values and standard deviations from five experiments.
From \cref{Kitti-table}, we observe that in the case of no denoising, RobustMat (GAT-PDE) outperforms the other baselines under all four measures. 
RobustMat (GCN-PDE) also outperforms the other baselines but is generally slightly worst than RobustMat (GAT-PDE).
This implies that the neural ODE/PDE modules have a positive impact on the robustness \cite{kang2021Neurips,song2022robustness}. 

Similar observations are made for the two-stage approach using denoising as the first stage. We observe that denoising does not always improve performance (these cases are highlighted in blue) and in general RobustMat without denoising is competitive with the best baseline with denoising. 
The denoising procedure does not remove all noise and may instead introduce artifacts. In this regard, RobustMat with denoising also performs better than the other baselines with denoising.
This indicates that despite image denoising methods that can be used to preprocess the input, RobustMat has the advantage of being a single-stage approach.

\textbf{Performance on noisy Oxford.}
From \cref{Oxford-table}, we observe that RobustMat (GAT-PDE) outperforms the other baseline methods across all four criteria. Notably, the Oxford and KITTI datasets possess distinct image resolutions and contain varying environmental conditions, thereby rendering the effects of additive noise unique to each patch-matching dataset. This difference may influence the overall robustness of the compared methods.

Comparing \cref{Kitti-table} and \cref{Oxford-table}, we see that all methods perform better on noisy KITTI than on noisy Oxford. This could be because the window patches employed in the noisy Oxford dataset exhibit greater similarities with each other, thereby making the distinction between windows in this dataset more challenging.

%%%%=================================
\begin{table*}[!htb]
\scriptsize
%\tiny
\caption{
Matching performance on noisy Oxford. The best and the second-best results under different metrics are highlighted in \textbf{bold} and \underline{underlined}, respectively. Performance deterioration after denoising is highlighted in \blue{blue}.
}
% \vspace{-0.2cm}
\label{Oxford-table}
\centering
\resizebox{0.98\textwidth}{!}{\setlength{\tabcolsep}{2pt} 
\newcommand{\tabincell}[2]{\begin{tabular}{@{}#1@{}}#2\end{tabular}}
%\resizebox{\textwidth}{18mm}{
\begin{tabular}{c|c|c|c|c|c|c|c|c|c|c|c}
\hline\hline
{\bf Methods}  & SiameseNet & MatchNet & HardNet & SOSNet & Exp-TLoss & HyNet & SOLAR & DELG & CVNet & \tabincell{c}{RobustMat \\ (GCN-PDE)\\} 
& \tabincell{c}{RobustMat \\ (GAT-PDE)\\} \\ 
\hline  
{\bf Precision}    & 0.8778 \tiny{$\pm$ 0.0189} & 0.9267 \tiny{$\pm$ 0.0041} & 0.9248 \tiny{$\pm$  0.0010} & 0.9184 \tiny{$\pm$ 0.0031} & 0.9024 \tiny{$\pm$ 0.0014} & 0.9314 \tiny{$\pm$ 0.0015} & 0.9201 \tiny{$\pm$ 0.0024} 
& 0.9223 \tiny{$\pm$ 0.0043} & 0.9302 \tiny{$\pm$ 0.0047} & \underline{0.9421} \tiny{$\pm$ 0.0040} & \textbf{0.9486} \tiny{$\pm$ 0.0037} \\ 
{\bf Recall} & 0.6000 \tiny{$\pm$ 0.0360} & 0.8363 \tiny{$\pm$ 0.0079} & 0.8360 \tiny{$\pm$ 0.0125} & 0.8416 \tiny{$\pm$ 0.0353} & 0.8384 \tiny{$\pm$ 0.0129} & \underline{0.8696} \tiny{$\pm$ 0.0198} & 0.8603 \tiny{$\pm$ 0.0278} 
& 0.8232 \tiny{$\pm$ 0.0082} & 0.8101 \tiny{$\pm$ 0.0105} & 0.8595 \tiny{$\pm$ 0.0041} & \textbf{0.8811} \tiny{$\pm$ 0.0076} \\ 
{\bf $F_1$-Score}  & 0.7120 \tiny{$\pm$ 0.0259} & 0.8792 \tiny{$\pm$ 0.0047} & 0.8781 \tiny{$\pm$ 0.0073} & 0.8780  \tiny{$\pm$ 0.0207} & 0.8692 \tiny{$\pm$ 0.0076} & \underline{0.8993} \tiny{$\pm$ 0.0113} & 0.8890 \tiny{$\pm$  0.0162} 
& 0.8699 \tiny{$\pm$ 0.0054} & 0.8660 \tiny{$\pm$ 0.0075} & 0.8989 \tiny{$\pm$ 0.0038} & \textbf{0.9136} \tiny{$\pm$ 0.0035} \\ 
{\bf AUC}          & 0.7580 \tiny{$\pm$ 0.0160} & 0.8189 \tiny{$\pm$ 0.0067} & 0.8160 \tiny{$\pm$ 0.0062} & 0.8088 \tiny{$\pm$ 0.0176} & 0.7832 \tiny{$\pm$ 0.0065} & 0.8389 \tiny{$\pm$ 0.0099} & 0.8181 \tiny{$\pm$ 0.0139} 
& 0.8076 \tiny{$\pm$ 0.0076} & 0.8139 \tiny{$\pm$ 0.0099} & \underline{0.8505} \tiny{$\pm$ 0.0071} & \textbf{0.8689} \tiny{$\pm$ 0.0046} \\ 
\hline\hline
{\bf Methods}  & \tabincell{c}{SiameseNet \\ +Denoising \\} & \tabincell{c}{MatchNet \\
+Denoising\\} & \tabincell{c}{HardNet \\ +Denoising\\} & \tabincell{c}{SOSNet\\+Denoising\\} & \tabincell{c}{Exp-TLoss\\+Denoising\\} & \tabincell{c}{HyNet \\+Denoising\\} & \tabincell{c}{SOLAR\\+Denoising\\} & \tabincell{c}{DELG \\ +Denoising\\} & \tabincell{c}{CVNet \\ +Denoising\\}  & \tabincell{c}{RobustMat \\ (GCN-PDE)\\ +Denoising\\} 
& \tabincell{c}{RobustMat \\ (GAT-PDE)\\+Denoising\\} \\ 
\hline  
{\bf Precision}    & \blue{0.8435} \tiny{$\pm$ 0.0112} & 0.9273 \tiny{$\pm$ 0.0050} & 0.9271 \tiny{$\pm$ 0.0010} & \blue{0.9068} \tiny{$\pm$ 0.0020} & \blue{0.8887} \tiny{$\pm$ 0.0020} & 0.9325 \tiny{$\pm$ 0.0010} & \blue{0.8969} \tiny{$\pm$ 0.0020} 
& 0.9281 \tiny{$\pm$ 0.0021} & 0.9328 \tiny{$\pm$ 0.0033} & \underline{\blue{0.9345}} \tiny{$\pm$ 0.0030} & \textbf{\blue{0.9350}} \tiny{$\pm$ 0.0027} \\ 
{\bf Recall} & 0.6536 \tiny{$\pm$ 0.0227} & 0.8499 \tiny{$\pm$ 0.0153} & 0.8645 \tiny{$\pm$ 0.0132} & 0.8563 \tiny{$\pm$ 0.0204} & 0.8624 \tiny{$\pm$ 0.0174} & 0.8848 \tiny{$\pm$ 0.0135} & 0.8819 \tiny{$\pm$ 0.0188} 
& 0.8395 \tiny{$\pm$ 0.0152} & 0.8627 \tiny{$\pm$ 0.0145} & \underline{0.8936} \tiny{$\pm$ 0.0054} & \textbf{0.8939} \tiny{$\pm$ 0.0065} \\ 
{\bf $F_1$-Score}  & 0.7362 \tiny{$\pm$ 0.0133} & 0.8868 \tiny{$\pm$ 0.0081} & 0.8947 \tiny{$\pm$ 0.0076} & 0.8807 \tiny{$\pm$ 0.0116} & 0.8753 \tiny{$\pm$ 0.0099} & 0.9080 \tiny{$\pm$ 0.0076} & 0.8892 \tiny{$\pm$ 0.0104} 
& 0.8815 \tiny{$\pm$ 0.0085} & 0.8963 \tiny{$\pm$ 0.0076} & \underline{0.9136} \tiny{$\pm$ 0.0035} & \textbf{0.9140} \tiny{$\pm$ 0.0041} \\ 
{\bf AUC}          & 0.7660 \tiny{$\pm$ 0.0083} & 0.8249 \tiny{$\pm$ 0.0083} & 0.8303 \tiny{$\pm$ 0.0066} & \blue{0.7961} \tiny{$\pm$ 0.0102} & 0.7692 \tiny{$\pm$ 0.0086} & 0.8464 \tiny{$\pm$ 0.0067} &  \blue{0.7889} \tiny{$\pm$ 0.0094} 
& 0.8221 \tiny{$\pm$ 0.0068} & 0.8381 \tiny{$\pm$ 0.0063} & \underline{0.8528} \tiny{$\pm$ 0.0056} & \textbf{\blue{0.8537}} \tiny{$\pm$ 0.0057} \\ 
\hline\hline
\end{tabular}}
% \vspace{-0.3cm}
\end{table*}
%%----------------------------
\begin{figure*}[!htb]
\centering
\includegraphics[width=\linewidth]{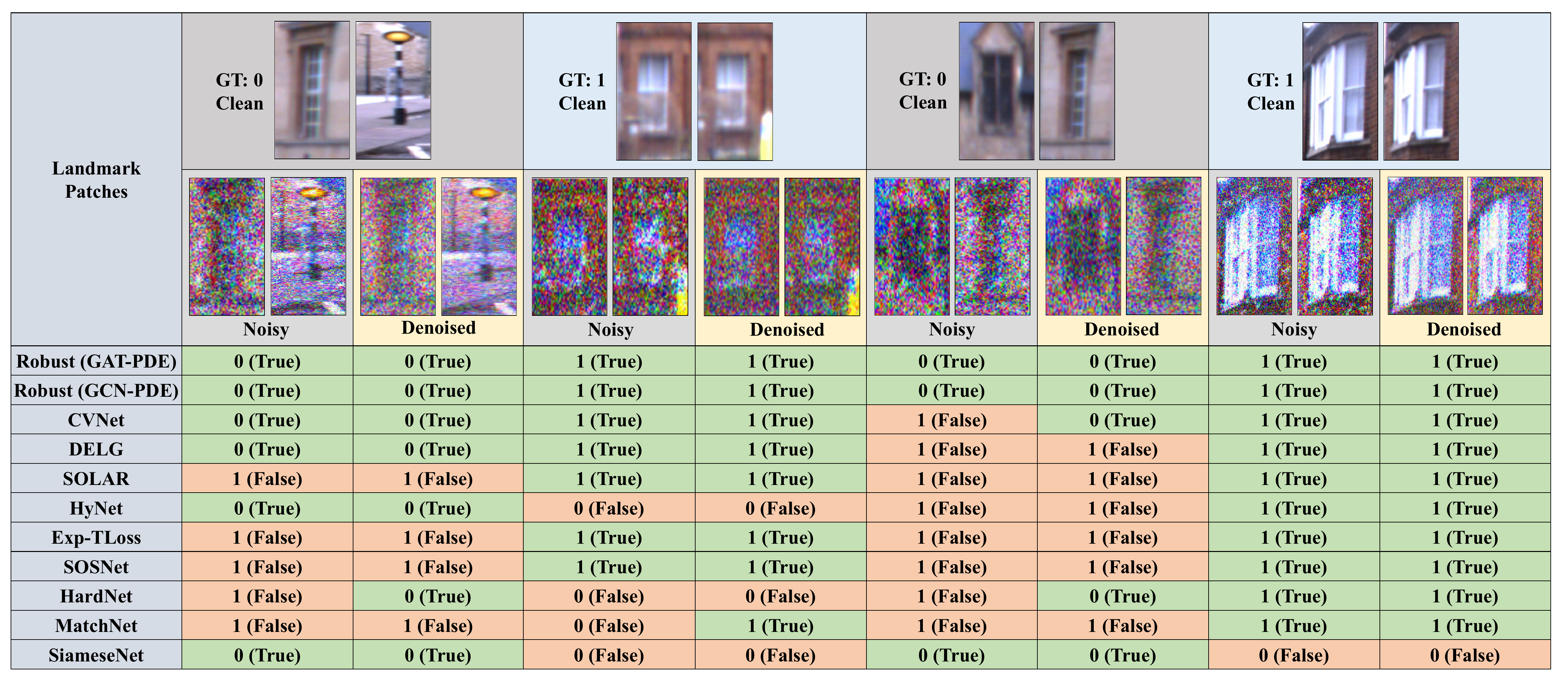}
\caption{
Matching prediction results for example landmark patch pairs from the noisy Oxford dataset. The prediction ``1'' or ``0'' indicates matched or unmatched. GT stands for ``ground truth''.
}
\label{fig_oxford_noise}
% \vspace{-0.5cm}
\end{figure*}
%%%%=================================

\textbf{Performance on Snowy/Rainy/Cloudy Boreas.}
From \cref{Boreas-table} and \cref{fig_boreas_noise}, we observe that compared to other baseline methods, RobustMat (GAT-PDE) outperforms other baselines in almost all criteria in the three noisy environments.
Specifically, RobustMat (GAT-PDE) outperforms the other methods in the rainy environment, and it continues to demonstrate superior performance under at least two criteria in snowy and cloudy environments.
Compared to the rainy and cloudy Boreas datasets, the snowy Boreas dataset significantly influences the matching performance of all methods under most criteria. However, RobustMat (GAT-PDE) exhibits less performance degradation.

%%%%=================================
\begin{table*}[!htb]
\scriptsize
%\tiny
\caption{
Matching performance for the snowy/rainy/cloudy Boreas. The best and the second-best results under different metrics are highlighted in \textbf{bold} and \underline{underlined}, respectively.
}
% \vspace{-0.2cm}
\label{Boreas-table}
\centering
\resizebox{0.98\textwidth}{!}{\setlength{\tabcolsep}{1.5pt} 
\newcommand{\tabincell}[2]{\begin{tabular}{@{}#1@{}}#2\end{tabular}}
%\resizebox{\textwidth}{18mm}{
\begin{tabular}{c|c|c|c|c|c|c|c|c|c|c|c|c}
\hline\hline
{\tabincell{c}{{\textbf{\emph{Noise}}} \\ {\textbf{\emph{Types}}} \\}} & {\bf Methods}  & SiameseNet &  MatchNet & HardNet & SOSNet & Exp-TLoss & HyNet & SOLAR & DELG & CVNet & \tabincell{c}{RobustMat \\ (GCN-PDE)\\} 
& \tabincell{c}{RobustMat \\ (GAT-PDE)\\} \\ 
\hline  
\multirow{4}{*}{\tabincell{c}{{\textbf{\emph{Snowy}}}  \\}} 
& {\bf Precision}  & 0.8654 \tiny{$\pm$ 0.0042} & 0.9004 \tiny{$\pm$ 0.0004} & 0.9136 \tiny{$\pm$ 0.0007} & 0.9273 \tiny{$\pm$ 0.0014} & 0.9118 \tiny{$\pm$ 0.0006} & \textbf{0.9406} \tiny{$\pm$ 0.0001} & 0.9256 \tiny{$\pm$ 0.0003} 
& 0.9347 \tiny{$\pm$ 0.0008} & 0.9223 \tiny{$\pm$ 0.0008} & \underline{0.9395} \tiny{$\pm$ 0.0004} & 0.9358 \tiny{$\pm$ 0.0004} \\ 
& {\bf Recall}     & 0.7728 \tiny{$\pm$ 0.0265} & 0.8677 \tiny{$\pm$ 0.0040} & 0.7893 \tiny{$\pm$ 0.0074} & 0.7488  \tiny{$\pm$ 0.0156} & 0.8827 \tiny{$\pm$ 0.0067} & 0.8027 \tiny{$\pm$ 0.0017} & 0.8635 \tiny{$\pm$ 0.0039} 
& 0.7642 \tiny{$\pm$ 0.0095} & \underline{0.8869} \tiny{$\pm$ 0.0102} & 0.8283 \tiny{$\pm$ 0.0062} & \textbf{0.8939} \tiny{$\pm$ 0.0057} \\ 
& {\bf $F_1$-Score} & 0.8163 \tiny{$\pm$ 0.0170} & 0.8838 \tiny{$\pm$ 0.0023} & 0.8469 \tiny{$\pm$ 0.0045} & 0.8285 \tiny{$\pm$ 0.0100} & 0.8970 \tiny{$\pm$ 0.0037} & 0.8662 \tiny{$\pm$ 0.0010} & 0.8935 \tiny{$\pm$ 0.0022} 
& 0.8409 \tiny{$\pm$ 0.0061} & \underline{0.9042} \tiny{$\pm$ 0.0057} & 0.8804 \tiny{$\pm$ 0.0037} & \textbf{0.9143} \tiny{$\pm$ 0.0032} \\ 
& {\bf AUC}        & 0.8264 \tiny{$\pm$ 0.0133} & 0.7899 \tiny{$\pm$ 0.0020} & 0.7827 \tiny{$\pm$ 0.0037} & 0.7864 \tiny{$\pm$ 0.0078} & 0.8133 \tiny{$\pm$ 0.0034} & 0.8253 \tiny{$\pm$ 0.0008} & 0.8277 \tiny{$\pm$ 0.0020} 
& 0.8021 \tiny{$\pm$ 0.0047} & 0.8315 \tiny{$\pm$ 0.0051} & \underline{0.8341} \tiny{$\pm$ 0.0031} & \textbf{0.8549} \tiny{$\pm$ 0.0028} \\ 
\hline
\multirow{4}{*}{\tabincell{c}{{\textbf{\emph{Rainy}}} \\}} 
& {\bf Precision}  & 0.8091 \tiny{$\pm$ 0.0006} & 0.9197 \tiny{$\pm$ 0.0004} & 0.9084 \tiny{$\pm$ 0.0003} & 0.9269 \tiny{$\pm$ 0.0002} & 0.9290 \tiny{$\pm$  0.0004} & 0.9288 \tiny{$\pm$ 0.0005} & 0.9294 \tiny{$\pm$ 0.0003} 
& 0.9205 \tiny{$\pm$ 0.0005} & 0.9179 \tiny{$\pm$ 0.0005} & \underline{0.9375} \tiny{$\pm$ 0.0003} & \textbf{0.9380} \tiny{$\pm$ 0.0006} \\ 
& {\bf Recall}     & 0.8816 \tiny{$\pm$ 0.0032} & 0.8853 \tiny{$\pm$ 0.0045} & 0.8987 \tiny{$\pm$ 0.0038} & 0.8795  \tiny{$\pm$ 0.0020} & 0.9067 \tiny{$\pm$ 0.0051} & 0.7648 \tiny{$\pm$ 0.0059} & 0.9131 \tiny{$\pm$ 0.0040} 
& \underline{0.9264} \tiny{$\pm$ 0.0057} & 0.9237 \tiny{$\pm$ 0.0060} & 0.9205 \tiny{$\pm$ 0.0052} & \textbf{0.9275} \tiny{$\pm$ 0.0088} \\ 
& {\bf $F_1$-Score} & 0.8438 \tiny{$\pm$ 0.0017} & 0.9022 \tiny{$\pm$ 0.0025} & 0.9035 \tiny{$\pm$ 0.0021} & 0.9026 \tiny{$\pm$ 0.0011} & 0.9177 \tiny{$\pm$ 0.0028} & 0.8388 \tiny{$\pm$ 0.0038} & 0.9212 \tiny{$\pm$ 0.0022} 
& 0.9234 \tiny{$\pm$ 0.0031} & 0.9208 \tiny{$\pm$ 0.0032} & \underline{0.9289} \tiny{$\pm$ 0.0028} & \textbf{0.9327} \tiny{$\pm$ 0.0047} \\ 
& {\bf AUC}        & 0.8368 \tiny{$\pm$ 0.0016} & 0.8267 \tiny{$\pm$ 0.0022} & 0.8133 \tiny{$\pm$ 0.0019} & 0.8357 \tiny{$\pm$ 0.0010} & 0.8493 \tiny{$\pm$ 0.0025} & 0.7944 \tiny{$\pm$ 0.0030} & 0.8525 \tiny{$\pm$ 0.0020} 
& 0.8432 \tiny{$\pm$ 0.0029} & 0.8379 \tiny{$\pm$ 0.0030} & \underline{0.8683} \tiny{$\pm$ 0.0026} & \textbf{0.8717} \tiny{$\pm$ 0.0044} \\ 
\hline
\multirow{4}{*}{\tabincell{c}{{\textbf{\emph{Cloudy}}} \\}} 
& {\bf Precision}  & 0.9369 \tiny{$\pm$ 0.0014} & 0.9560 \tiny{$\pm$ 0.0008} & 0.9562 \tiny{$\pm$ 0.0007} & 0.9424 \tiny{$\pm$ 0.0008} & 0.9509 \tiny{$\pm$ 0.0007} & 0.9427 \tiny{$\pm$ 0.0007} & \underline{0.9579} \tiny{$\pm$ 0.0005} 
& 0.9350 \tiny{$\pm$ 0.0004} & 0.9549 \tiny{$\pm$ 0.0002} & 0.9485 \tiny{$\pm$ 0.0001} & \textbf{0.9581} \tiny{$\pm$ 0.0004} \\ 
& {\bf Recall}     & 0.7135 \tiny{$\pm$ 0.0171} & 0.8107 \tiny{$\pm$ 0.0151} & 0.8160 \tiny{$\pm$ 0.0133} & 0.8299  \tiny{$\pm$ 0.0117} & 0.8261 \tiny{$\pm$  0.0120} & 0.8341 \tiny{$\pm$ 0.0109} & 0.8507 \tiny{$\pm$ 0.0100} 
& 0.8437 \tiny{$\pm$ 0.0049} & 0.8464 \tiny{$\pm$ 0.0036} & \textbf{0.8832} \tiny{$\pm$ 0.0026} & \underline{0.8528} \tiny{$\pm$ 0.0095} \\ 
& {\bf $F_1$-Score} & 0.8101 \tiny{$\pm$ 0.0116} & 0.8773 \tiny{$\pm$ 0.0092} & 0.8805 \tiny{$\pm$ 0.0080} & 0.8825 \tiny{$\pm$ 0.0070} & 0.8841 \tiny{$\pm$ 0.0072} & 0.8851 \tiny{$\pm$ 0.0064} & 0.9011 \tiny{$\pm$ 0.0058} 
& 0.8870 \tiny{$\pm$ 0.0029} & 0.8974 \tiny{$\pm$ 0.0021} & \textbf{0.9147} \tiny{$\pm$ 0.0015} & \underline{0.9023} \tiny{$\pm$ 0.0055} \\ 
& {\bf AUC}        & 0.8328 \tiny{$\pm$ 0.0085} & 0.8493 \tiny{$\pm$ 0.8493} & 0.8520 \tiny{$\pm$ 0.0066} & 0.8389 \tiny{$\pm$ 0.0059} & 0.8491 \tiny{$\pm$ 0.0060} & 0.8411 \tiny{$\pm$ 0.0054} & 0.8693 \tiny{$\pm$ 0.0050} 
& 0.8339 \tiny{$\pm$ 0.0025} & 0.8632 \tiny{$\pm$ 0.0018} & \underline{0.8696} \tiny{$\pm$ 0.0013} & \textbf{0.8704} \tiny{$\pm$ 0.0047} \\ 
\hline
\end{tabular}}
\end{table*}
%%----------------------------
\begin{figure*}[!htb]
\centering
\includegraphics[width=0.75\linewidth]{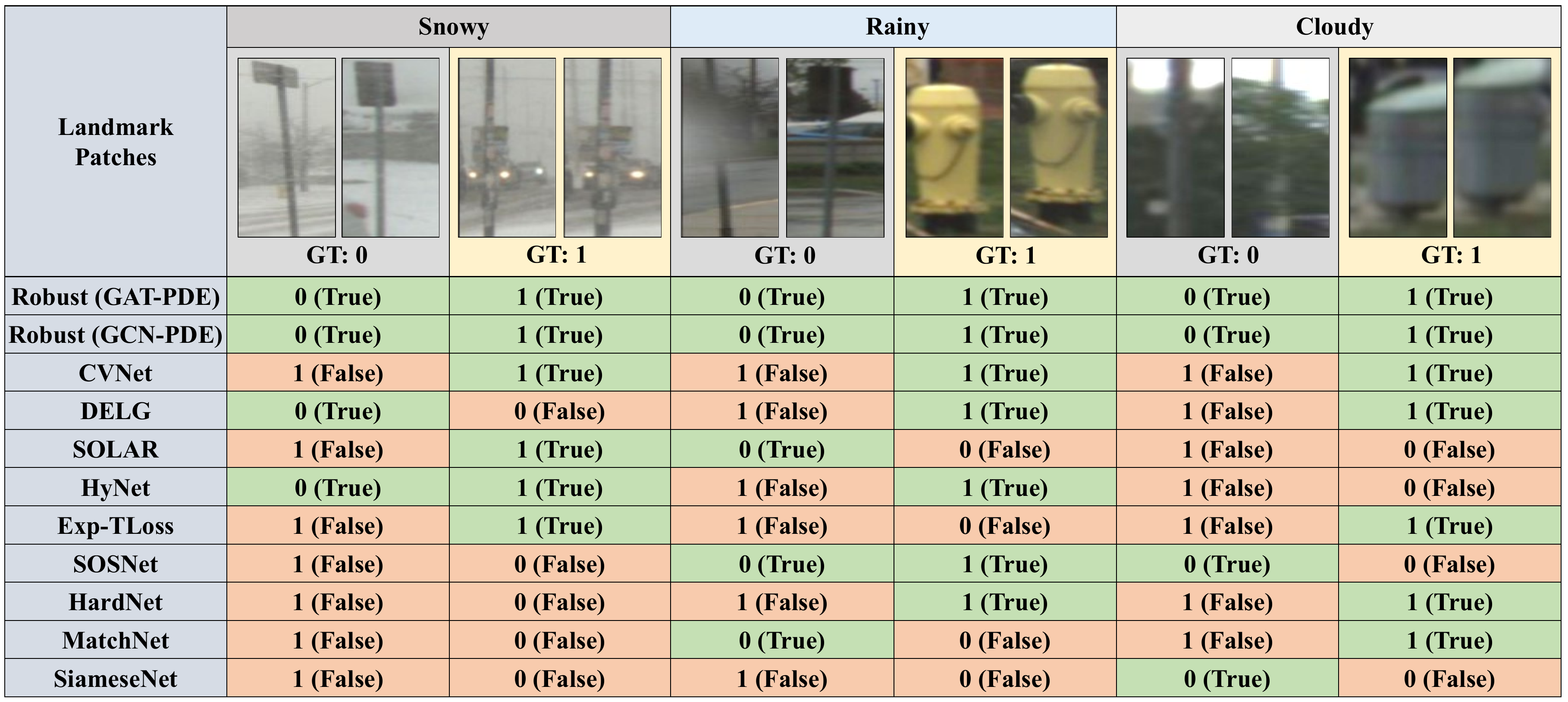}
\caption{
Matching prediction results for example landmark patch pairs from the snowy/rainy/cloudy Boreas datasets. The prediction ``1'' or ``0'' indicates matched or unmatched. GT stands for ``ground truth''.
}
\label{fig_boreas_noise}
% \vspace{-0.5cm}
\end{figure*}
%% ============================

\subsection{Ablation Studies}\label{sec:ablation}
In this section, we investigate the necessity of the different modules in RobustMat. 

\begin{table}[!htb]
\scriptsize
\caption{Ablation study for including neural ODE/PDE.}
% \vspace{-0.2cm}
\label{Table.ablation}
\centering
\newcommand{\tabincell}[2]{\begin{tabular}{@{}#1@{}}#2\end{tabular}}
\begin{tabular}{ccccc}
\hline\hline
\multirow{2}{*}{\bf Model}  & \multicolumn{2}{c}{\tabincell{c}{RobustMat with neural ODE + \\ GAT-based neural PDE \\}} & \multicolumn{2}{c}{Resnet+GAT } \\
                            & \emph{noisy} & \emph{clean} & \emph{noisy} & \emph{clean} \\
\hline
\tabincell{c}{{\bf Precision}\ \  \\}      & 0.9418 & 0.9488 & 0.9084 & 0.9146 \\
\hline
\tabincell{c}{{\bf Recall}\ \  \\}         & 0.8853 & 0.9147 & 0.8467 & 0.9133 \\
\hline
\tabincell{c}{{\bf $F_1$-Score}\ \  \\}    & 0.9127 & 0.9314 & 0.8765 & 0.9139 \\
\hline
\tabincell{c}{{\bf AUC}\ \  \\}            & 0.8606 & 0.8833 & 0.7953 & 0.8286 \\
\hline\hline
\end{tabular}%}
\end{table}

\subsubsection{Advantage of the neural ODE/PDE modules}
We replace the CNN-based neural ODE and GAT-based graph neural PDE modules in our model with vanilla Resnet and GAT, respectively. 
We compare RobustMat (with the neural ODE/PDE modules) with this model with Resnet and GAT under both clean and noisy conditions on the Oxford dataset. 
From \cref{Table.ablation}, we observe that the neural ODE/PDE modules not only improve the matching performance but also the robustness against the noise.

\begin{table}[!htb]
\scriptsize
\caption{Ablation study for only using the neural ODE or the graph neural PDE.}
% \vspace{-0.2cm}
\label{Table.ablation_only_ODE_PDE}
\centering
\newcommand{\tabincell}[2]{\begin{tabular}{@{}#1@{}}#2\end{tabular}}
\begin{tabular}{cccc}
\hline\hline
{\bf Model}  & \tabincell{c}{Neural \\ ODE \\}  &  \tabincell{c}{GAT-based \\ neural PDE \\} &  \tabincell{c}{Neural ODE \\ + PDE \\}\\
\hline
\tabincell{c}{{\bf Precision}\ \  \\}      & \textbf{0.9609} & 0.7745 & {0.9418}  \\
\hline
\tabincell{c}{{\bf Recall}\ \  \\}         & 0.8200 & \textbf{0.9800} & {0.8853}  \\
\hline
\tabincell{c}{{\bf $F_1$-Score}\ \  \\}    & {0.8849} & 0.8652 & \textbf{0.9127}   \\
\hline
\tabincell{c}{{\bf AUC}\ \  \\}            & {0.8600} & 0.5620 & \textbf{0.8606}   \\
\hline\hline
\end{tabular}
\end{table}
%%%----------------------

\subsubsection{Combination of the neural ODE and PDE modules}
Since RobustMat combines the neural ODE and the graph neural PDE, we compare with the models using either only the neural ODE or the graph neural PDE. 
The results in \cref{Table.ablation_only_ODE_PDE} indicates the advantage of combining the neural ODE and PDE modules when considering the $F_1$-score and AUC. When using only the neural ODE, we tend to have high precision but lower recall. When using only the GAT-based neural PDE, we tend to have a high recall but low precision. This could be due to the former having more false negative matchings as it does not utilize neighborhood information, while the latter has more false positive matchings due to aggregating non-robust vertex features. RobustMat combines the best of both worlds and achieves the best $F_1$-score and AUC.
This also validates our choice of the loss function in \cref{sect:loss_score}.

\begin{table}[!htb]
\scriptsize
\caption{Ablation for different terminal time $T_f$ and $T_g$ for the neural ODE and PDE modules on RobustMat.}
% \vspace{-0.2cm}
\label{Table.ablation_diff_T}
\centering
\newcommand{\tabincell}[2]{\begin{tabular}{@{}#1@{}}#2\end{tabular}}
\begin{tabular}{c c c c c}
\hline\hline
{\bf Terminal Time}  & {\bf Precision}  &  {\bf Recall} &  {\bf $F_1$-Score} & {\bf AUC} \\
\hline
\tabincell{c}{$T_f=2.0$ \& $T_g=1.0$ }  & 0.9273      & {0.9026}    & 0.9149      & 0.8453 \\
\tabincell{c}{$T_f=3.0$ \& $T_g=1.0$ }  & {0.9411}    & 0.8960      & 0.9180      & {0.8640} \\
\tabincell{c}{$T_f=4.0$ \& $T_g=1.0$ }  & 0.9176      & {0.9360}    & {0.9267}    & 0.8420 \\ 
\tabincell{c}{$T_f=5.0$ \& $T_g=1.0$ }  & {0.9488}    & 0.8906      & {0.9188}    & {0.8733} \\ 
\hline
\tabincell{c}{$T_f=2.0$ \& $T_g=2.0$ }  & 0.9358      & 0.8947      & 0.9148      & 0.8553 \\
\tabincell{c}{$T_f=3.0$ \& $T_g=3.0$ }  & 0.9353      & {0.9253}    & {0.9303}    & 0.8667 \\
\tabincell{c}{$T_f=4.0$ \& $T_g=4.0$ }  & {0.9401}    & {0.9213}    & {0.9306}    & {0.8727} \\ 
\tabincell{c}{$T_f=5.0$ \& $T_g=5.0$ }  & {0.9513}    & 0.8867      & 0.9179      & {0.8753} \\ 
\hline\hline
\end{tabular}
\end{table}

\subsubsection{Terminal time for neural ODE/PDE modules}
We investigate how different terminal times $T_f$ and $T_g$ influence the matching performance. 
From \cref{Table.ablation_diff_T}, we observe that when the terminal times $T_f$ and $T_g$ increase, the performance of RobustMat improves in general, validating \cref{thm.Robustness_ode,thm.Robustness_gde}. The trade-off of having a larger $T_f$ or $T_g$ is the runtime complexity of the ODE solver. 

\begin{table}[!htb]
\scriptsize
\caption{Ablation study for different regularizations.}
% \vspace{-0.2cm}
\label{Table.ablation_diff_reg}
\centering
\newcommand{\tabincell}[2]{\begin{tabular}{@{}#1@{}}#2\end{tabular}}
%\resizebox{0.48\textwidth}{!}{\setlength{\tabcolsep}{2pt} 
\begin{tabular}{ c c c c}
\hline\hline
{\textbf{Model}}  & \tabincell{c}{RobustMat \\ (no Reg) \\}  & \tabincell{c}{RobustMat \\ (ODE Reg) \\} &  \tabincell{c}{RobustMat \\ (ODE \& PDE Reg) \\}  \\
\hline
{\bf Precision}     & \underline{0.9486} \tiny{$\pm$ 0.0037}  & 0.9426 \tiny{$\pm$ 0.0041}             & \textbf{0.9500} \tiny{$\pm$ 0.0032}    \\ \hline
{\bf Recall}        & 0.8811 \tiny{$\pm$ 0.0076}              & \textbf{0.8888} \tiny{$\pm$ 0.0069}    & \underline{0.8827} \tiny{$\pm$ 0.0054} \\ \hline
{\bf $F_1$-Score}   & 0.9136 \tiny{$\pm$ 0.0035}              & \underline{0.9149} \tiny{$\pm$ 0.0034} & \textbf{0.9151} \tiny{$\pm$ 0.0041}    \\ \hline
{\bf AUC}           & \underline{0.8689} \tiny{$\pm$ 0.0046}  & 0.8632 \tiny{$\pm$ 0.0056}             & \textbf{0.8717} \tiny{$\pm$ 0.0065}    \\ 
\hline\hline
\end{tabular}%}
\end{table}

\subsubsection{Constraints for neural ODE/PDE modules}
We now test loss regularizations to achieve the neural ODE/PDE constraints in \cref{ass:ass_ivp,ass:ass_Amatrix}. 
Specifically, we first obtain the feature difference for each patch pair in the input and output of a neural ODE/PDE module. 
We use the output-input quotient minimization for the feature difference as the regularization for the neural ODE/PDE module.
We train and test these methods based on the Oxford dataset. 
From \cref{Table.ablation_diff_reg}, we see that the method with the constraint regularizations for both neural ODE and PDE almost outperforms the other variants. 
The experimental results also validate the theoretical analysis in \cref{thm.Robustness_ode,thm.Robustness_gde}. 

%%%%==========================================================================
\begin{table}[!htb]
%\scriptsize
%\footnotesize
%\vspace{-0.8cm}
\caption{Matching performance comparison with the keypoint-level baselines on noisy KITTI.} 
% \vspace{-0.2cm}
\label{table-keypoint}
\centering
\resizebox{0.48\textwidth}{!}{\setlength{\tabcolsep}{3.5pt} 
\newcommand{\tabincell}[2]{\begin{tabular}{@{}#1@{}}#2\end{tabular}}
%\resizebox{\textwidth}{18mm}{ 
\begin{tabular}{c|c|c|c|c}
\hline\hline
%{\bf Test Datasets} & 
{\bf Methods}  & D2-Net & SuperGlue & LoFTR & \tabincell{c}{RobustMat \\ \&  Trained on Oxford \\}
\\ \hline
{\bf Precision}    & \underline{0.9318} \tiny{$\pm$ 0.0029} & 0.8509 \tiny{$\pm$ 0.0025}             & 0.9243 \tiny{$\pm$ 0.0018} & {\textbf{0.9408}} \tiny{$\pm$ 0.0024} \\ 
{\bf Recall}       & 0.7723 \tiny{$\pm$ 0.0188}             & \underline{0.8496} \tiny{$\pm$ 0.0177} & 0.8405 \tiny{$\pm$ 0.0103} & {\textbf{0.8560}} \tiny{$\pm$ 0.0174} \\  
{\bf $F_1$-Score}  & 0.8444 \tiny{$\pm$ 0.0120}             & 0.8502 \tiny{$\pm$ 0.0099}             & \underline{0.8804} \tiny{$\pm$ 0.0065} & {\textbf{0.8963}} \tiny{$\pm$ 0.0101} \\ 
{\bf AUC}          & 0.8013 \tiny{$\pm$ 0.0102}             & 0.7016 \tiny{$\pm$ 0.0081}             & \underline{0.8171} \tiny{$\pm$ 0.0066} & {\textbf{0.8472}} \tiny{$\pm$ 0.0093} \\  
\hline\hline
%%------------------------------------------------------------------------------------------------
%\multirow{5}{*}{Oxford} & 
{\bf Methods}  & \tabincell{c}{D2-Net \\+ Denoising\\} & \tabincell{c}{SuperGlue \\+ Denoising\\} & \tabincell{c}{LoFTR \\+ Denoising\\} & \tabincell{c}{RobustMat \\ + Denoising \\ \&  Train on Oxford \\} 
\\ \hline 
{\bf Precision}   & \underline{0.9241} \tiny{$\pm$ 0.0029} & 0.8824 \tiny{$\pm$ 0.0036}   & 0.9167 \tiny{$\pm$ 0.0019}  & \textbf{0.9327} \tiny{$\pm$ 0.0037} \\ 
{\bf Recall}      & 0.8181 \tiny{$\pm$ 0.0190}             & 0.8603 \tiny{$\pm$ 0.0078}   & \underline{0.8629} \tiny{$\pm$ 0.0130}  & \textbf{0.8869} \tiny{$\pm$ 0.0191} \\ 
{\bf $F_1$-Score} & 0.8678 \tiny{$\pm$ 0.0116}             & 0.8712 \tiny{$\pm$ 0.0054}   & \underline{0.8890} \tiny{$\pm$ 0.0065} & \textbf{0.9091} \tiny{$\pm$ 0.0113}  \\ 
{\bf AUC}         & 0.8083 \tiny{$\pm$ 0.0106}             & 0.7581 \tiny{$\pm$ 0.0079} & \underline{0.8139} \tiny{$\pm$ 0.0043}  & \textbf{0.8475} \tiny{$\pm$ 0.0120} \\  
%\tabincell{l}{* \\ *}
\hline\hline
\end{tabular}}
% \vspace{-0.3cm}
\end{table}

%%%%=================================
\begin{table*}[!htb]
\scriptsize
%\tiny
\caption{
Matching performance under different noises for the Oxford landmark patches.
}
% \vspace{-0.2cm}
\label{denoise-table}
\centering
\resizebox{0.98\textwidth}{!}{\setlength{\tabcolsep}{2pt} 
\newcommand{\tabincell}[2]{\begin{tabular}{@{}#1@{}}#2\end{tabular}}
%\resizebox{\textwidth}{18mm}{
\begin{tabular}{c|c|c|c|c|c|c|c|c|c|c|c|c}
\hline\hline
{\tabincell{c}{{\textbf{\emph{Noise}}} \\ {\textbf{\emph{Level}}} \\}} & {\bf Methods}  & SiameseNet &  MatchNet & HardNet & SOSNet & Exp-TLoss & HyNet & SOLAR & DELG & CVNet & \tabincell{c}{RobustMat \\ (GCN-PDE)\\} 
& \tabincell{c}{RobustMat \\ (GAT-PDE)\\} \\ 
\hline  
\multirow{4}{*}{\tabincell{c}{{\textbf{\emph{Low-}}} \\ {\textbf{\emph{Level}}} \\ {\textbf{\emph{Noise}}} \\}} 
& {\bf Precision}  & 0.8405 \tiny{$\pm$ 0.0162} & 0.9228 \tiny{$\pm$ 0.0048} & 0.9284 \tiny{$\pm$ 0.0013} & 0.9215 \tiny{$\pm$ 0.0008} & 0.9074 \tiny{$\pm$ 0.0014} & 0.9398 \tiny{$\pm$ 0.0012} & 0.9226 \tiny{$\pm$ 0.0020} 
& 0.9341 \tiny{$\pm$ 0.0056} & 0.9400 \tiny{$\pm$ 0.0030} & \textbf{0.9458} \tiny{$\pm$ 0.0042} & \underline{0.9418} \tiny{$\pm$ 0.0055} \\ 
& {\bf Recall}     & 0.6224 \tiny{$\pm$ 0.0395} & 0.8597 \tiny{$\pm$ 0.0059} & 0.8827 \tiny{$\pm$ 0.0175} & 0.8768  \tiny{$\pm$ 0.0102} & 0.8883 \tiny{$\pm$ 0.0147} & \underline{0.8957} \tiny{$\pm$ 0.0185} & 0.8909 \tiny{$\pm$ 0.0253} 
& 0.8616 \tiny{$\pm$ 0.0134} & 0.8557 \tiny{$\pm$ 0.0066} & 0.8795 \tiny{$\pm$ 0.0089} & \textbf{0.9021} \tiny{$\pm$ 0.0042} \\ 
& {\bf $F_1$-Score} & 0.7144 \tiny{$\pm$ 0.0269} & 0.8901 \tiny{$\pm$ 0.0035} & 0.9049 \tiny{$\pm$ 0.0098} & 0.8986 \tiny{$\pm$ 0.0057} & 0.8977 \tiny{$\pm$ 0.0082} & \underline{0.9172} \tiny{$\pm$ 0.0103} & 0.9063 \tiny{$\pm$ 0.0141} 
& 0.8963 \tiny{$\pm$ 0.0092} & 0.8959 \tiny{$\pm$ 0.0035} & 0.9114 \tiny{$\pm$ 0.0060} & \textbf{0.9215} \tiny{$\pm$ 0.0033} \\ 
& {\bf AUC}        & 0.7520 \tiny{$\pm$ 0.0173} & 0.8218 \tiny{$\pm$ 0.0073} & 0.8393 \tiny{$\pm$ 0.0088} & 0.8264 \tiny{$\pm$ 0.0051} & 0.8081 \tiny{$\pm$ 0.0073} & 0.8619 \tiny{$\pm$ 0.0093} & 0.8335 \tiny{$\pm$ 0.0127} 
& 0.8396 \tiny{$\pm$ 0.0125} & 0.8459 \tiny{$\pm$ 0.0045} & \underline{0.8641} \tiny{$\pm$ 0.0088} & \textbf{0.8675} \tiny{$\pm$ 0.0085} \\ 
\hline 
\multirow{4}{*}{\tabincell{c}{{\textbf{\emph{High-}}} \\ {\textbf{\emph{Level}}} \\ {\textbf{\emph{Noise}}} \\}} 
& {\bf Precision}  & 0.8763 \tiny{$\pm$ 0.0189} & \underline{0.9410} \tiny{$\pm$ 0.0018} & 0.9129 \tiny{$\pm$ 0.0012} & 0.9167 \tiny{$\pm$ 0.0011} & 0.8999 \tiny{$\pm$ 0.0023} & 0.9306 \tiny{$\pm$ 0.0021} & 0.9153 \tiny{$\pm$ 0.0014} 
& 0.9100 \tiny{$\pm$ 0.0060} & 0.9121 \tiny{$\pm$ 0.0063} & 0.9333 \tiny{$\pm$ 0.0048} & \textbf{0.9430} \tiny{$\pm$ 0.0034} \\  
& {\bf Recall}     & 0.5776 \tiny{$\pm$ 0.0427} & 0.7733 \tiny{$\pm$ 0.0145} & 0.7125 \tiny{$\pm$ 0.0107} & 0.8221 \tiny{$\pm$ 0.0127} & 0.8160 \tiny{$\pm$ 0.0204} & 0.7696 \tiny{$\pm$ 0.0247} & 0.8072 \tiny{$\pm$ 0.0152} 
& 0.6872 \tiny{$\pm$ 0.0112} & 0.6117 \tiny{$\pm$ 0.0154} & \textbf{0.8400} \tiny{$\pm$ 0.0058} & \underline{0.8256} \tiny{$\pm$ 0.0104} \\ 
& {\bf $F_1$-Score} & 0.6954 \tiny{$\pm$ 0.0323} & 0.8489 \tiny{$\pm$ 0.0084} & 0.8003 \tiny{$\pm$ 0.0072} & 0.8668 \tiny{$\pm$ 0.0075} & 0.8558 \tiny{$\pm$ 0.0124} & 0.8422 \tiny{$\pm$ 0.0159} & 0.8578 \tiny{$\pm$ 0.0093} 
& 0.7830 \tiny{$\pm$ 0.0074} & 0.7322 \tiny{$\pm$ 0.0125} & \textbf{0.8842} \tiny{$\pm$ 0.0048} & \underline{0.8804} \tiny{$\pm$ 0.0064} \\ 
& {\bf AUC}        & 0.7480 \tiny{$\pm$ 0.0206} & 0.8139 \tiny{$\pm$ 0.0053} & 0.7543 \tiny{$\pm$ 0.0054} & 0.7991 \tiny{$\pm$ 0.0063} & 0.7720 \tiny{$\pm$ 0.0102} & 0.7988 \tiny{$\pm$ 0.0124} & 0.7916 \tiny{$\pm$ 0.0076} 
& 0.7416 \tiny{$\pm$ 0.0081} & 0.7175 \tiny{$\pm$ 0.0113} & \underline{0.8300} \tiny{$\pm$ 0.0084} & \textbf{0.8380} \tiny{$\pm$ 0.0072} \\ 
\hline\hline
{\tabincell{c}{{\textbf{\emph{Noise}}} \\ {\textbf{\emph{Level}}} \\}} & {\bf Methods}  & \tabincell{c}{SiameseNet \\ +Denoising \\} & \tabincell{c}{MatchNet \\
+Denoising\\} & \tabincell{c}{HardNet \\ +Denoising\\} & \tabincell{c}{SOSNet\\+Denoising\\} & \tabincell{c}{Exp-TLoss\\+Denoising\\} & \tabincell{c}{HyNet \\+Denoising\\} & \tabincell{c}{SOLAR\\+Denoising\\} & \tabincell{c}{DELG \\ +Denoising\\}  & \tabincell{c}{CVNet \\ +Denoising\\} & \tabincell{c}{RobustMat \\ (GCN-PDE)\\ +Denoising\\} 
& \tabincell{c}{RobustMat \\ (GAT-PDE)\\+Denoising\\} \\ 
\hline 
\multirow{4}{*}{\tabincell{c}{{\textbf{\emph{Low-}}} \\ {\textbf{\emph{Level}}} \\ {\textbf{\emph{Noise}}} \\}} 
& {\bf Precision}  & 0.8337 \tiny{$\pm$ 0.0159} & 0.9164 \tiny{$\pm$ 0.0058} & 0.9287 \tiny{$\pm$ 0.0015} & 0.9110 \tiny{$\pm$ 0.0013} & 0.8936 \tiny{$\pm$ 0.0010} & 0.9362 \tiny{$\pm$ 0.0004} & 0.9002 \tiny{$\pm$ 0.0015} 
& 0.9362 \tiny{$\pm$ 0.0042} & \underline{0.9432} \tiny{$\pm$ 0.0041} & \textbf{0.9445} \tiny{$\pm$ 0.0037} & 0.9380 \tiny{$\pm$ 0.0027} \\ 
& {\bf Recall} & 0.6840 \tiny{$\pm$ 0.0402} & 0.8880 \tiny{$\pm$ 0.0114} & 0.8861 \tiny{$\pm$ 0.0198} & 0.9013 \tiny{$\pm$ 0.0153} & 0.9075 \tiny{$\pm$ 0.0104} & \textbf{0.9389} \tiny{$\pm$ 0.0065} & 0.9144 \tiny{$\pm$ 0.0152} 
& 0.8621 \tiny{$\pm$ 0.0175} & 0.8680 \tiny{$\pm$ 0.0144} & 0.9035 \tiny{$\pm$ 0.0066} & \underline{0.9157} \tiny{$\pm$ 0.0081} \\ 
& {\bf $F_1$-Score} & 0.7507 \tiny{$\pm$ 0.0253} & 0.9019 \tiny{$\pm$ 0.0069} & 0.9068 \tiny{$\pm$ 0.0110} & 0.9061 \tiny{$\pm$ 0.0084} & 0.9004 \tiny{$\pm$ 0.0057} & \textbf{0.9375} \tiny{$\pm$ 0.0035} & 0.9072 \tiny{$\pm$ 0.0082} 
& 0.8976 \tiny{$\pm$ 0.0112} & 0.9040 \tiny{$\pm$ 0.0075} & 0.9235 \tiny{$\pm$ 0.0047} & \underline{0.9267} \tiny{$\pm$ 0.0051} \\ 
& {\bf AUC}        & 0.7736 \tiny{$\pm$ 0.0176} & 0.8224 \tiny{$\pm$ 0.0108} & 0.8411 \tiny{$\pm$ 0.0099} & 0.8187 \tiny{$\pm$ 0.0076} & 0.7917 \tiny{$\pm$ 0.0052} & \textbf{0.8735} \tiny{$\pm$ 0.0032} & 0.8051 \tiny{$\pm$  0.0076} 
& 0.8431 \tiny{$\pm$ 0.0129} & 0.8556 \tiny{$\pm$ 0.0072} & \underline{0.8721} \tiny{$\pm$ 0.0076} & 0.8671 \tiny{$\pm$ 0.0069} \\ 
\hline
\multirow{4}{*}{\tabincell{c}{{\textbf{\emph{High-}}} \\ {\textbf{\emph{Level}}} \\ {\textbf{\emph{Noise}}} \\}} 
& {\bf Precision}  & 0.8630 \tiny{$\pm$ 0.0162} & 0.9260 \tiny{$\pm$ 0.0040} & 0.9166 \tiny{$\pm$ 0.0034} & 0.9047 \tiny{$\pm$ 0.0014} & 0.8868  \tiny{$\pm$ 0.0007} & 0.9239 \tiny{$\pm$ 0.0005} & 0.8948 \tiny{$\pm$ 0.0019} 
& 0.8882 \tiny{$\pm$ 0.0052} & 0.9006 \tiny{$\pm$ 0.0037} & \underline{0.9311} \tiny{$\pm$ 0.0054} & \textbf{0.9418} \tiny{$\pm$ 0.0027} \\ 
& {\bf Recall}     & 0.6088 \tiny{$\pm$ 0.0269} & 0.7843 \tiny{$\pm$ 0.0139} & 0.7493 \tiny{$\pm$ 0.0319} & 0.8355 \tiny{$\pm$ 0.0131} & 0.8464 \tiny{$\pm$ 0.0057} & 0.7768 \tiny{$\pm$ 0.0050} & \underline{0.8621} \tiny{$\pm$ 0.0183} 
& 0.8392 \tiny{$\pm$ 0.0073} & 0.8483 \tiny{$\pm$ 0.0087} & \textbf{0.8648} \tiny{$\pm$ 0.0110} & 0.8328 \tiny{$\pm$ 0.0046} \\ 
& {\bf $F_1$-Score} & 0.7136 \tiny{$\pm$ 0.0203} & 0.8492 \tiny{$\pm$ 0.0087} & 0.8243 \tiny{$\pm$ 0.0212} & 0.8687 \tiny{$\pm$ 0.0077} & 0.8661 \tiny{$\pm$ 0.0033} & 0.8440 \tiny{$\pm$ 0.0031} & 0.8781 \tiny{$\pm$ 0.0104} 
& 0.8630 \tiny{$\pm$ 0.0036} & 0.8736 \tiny{$\pm$ 0.0047} & \textbf{0.8967} \tiny{$\pm$ 0.0078} & \underline{0.8839} \tiny{$\pm$ 0.0031} \\ 
& {\bf AUC}        & 0.7560 \tiny{$\pm$ 0.0140} & 0.7981 \tiny{$\pm$ 0.0082} & 0.7727 \tiny{$\pm$ 0.0160} & 0.7857 \tiny{$\pm$ 0.0065} & 0.7612 \tiny{$\pm$ 0.0028} & 0.7924 \tiny{$\pm$ 0.0025} & 0.7791 \tiny{$\pm$ 0.0092} 
& 0.7612 \tiny{$\pm$ 0.0078} & 0.7837 \tiny{$\pm$ 0.0061} & \underline{0.8364} \tiny{$\pm$ 0.0113} & \textbf{0.8392} \tiny{$\pm$ 0.0046} \\
\hline\hline
\end{tabular}}
\end{table*}
%%%%=================================

\subsubsection{Comparison with keypoint-level matching}\label{sect:keypoint}
We compare RobustMat with keypoint-level matching methods, including D2-Net \cite{dusmanu2019d2}, SuperGlue \cite{sarlin2020superglue} and LoFTR \cite{sun2021loftr}. 
We use the pre-trained models of these baselines given by the corresponding literature \cite{dusmanu2019d2,sarlin2020superglue,sun2021loftr}. 
To ensure a fair comparison, we train and test RobustMat on \emph{different} datasets: we train on the Oxford dataset but test on the noisy KITTI dataset. 
The landmarks from the Oxford dataset include windows, which are not included in the KITTI dataset. 
From \cref{table-keypoint}, we observe that RobustMat (GAT-PDE) is superior to these keypoint-level matching baselines with and without denoising.  
The reason may be that noise influences the detection of keypoints, causing keypoint-level matching methods to fail to obtain accurate keypoint correspondence. Further results are provided in the supplementary.

\subsubsection{Effect of different noise intensities} 
We vary the additive Gaussian noise power for the noisy Oxford dataset to obtain perturbed input patches with different qualities.
Low-level noise refers to a PSNR of around $19$ dB, and high-level noise has a PSNR of around $13$ dB.

From \cref{denoise-table}, we observe RobustMat is competitive under high-level noise with or without denoising, and under low-level noise without denoising. With denoising under low-level noise, RobustMat is not the best performer. This is because the denoising process can effectively remove low-level noise and the neural ODE/PDE modules lose their advantages.

\subsubsection{Neighborhood information}\label{sect:neighborhood}
We compare RobustMat to other baselines that also use neighborhood information. Since the vanilla baselines we use do not originally utilize this information, we incorporate such information into them for a fair comparison. To accomplish this, we sort the neighbors of a given patch according to its center pixel coordinates and then compare the patch pair and their corresponding neighboring pairs in each baseline. We calculate the average of the predicted scores and decide the matching based on a threshold that we set as a hyperparameter, \emph{tuned to obtain the best performance for each baseline method}. The results in \Cref{neighborhood-table} demonstrate that RobustMat outperforms the baselines even when taking neighborhood information into consideration. 

%%%%=================================
\begin{table}[!htb]
%\scriptsize
%\tiny
\caption{
Matching performance on the noisy Oxford for the comparison of the models with neighborhood information.
}
% \vspace{-0.2cm}
\label{neighborhood-table}
\centering
\resizebox{0.48\textwidth}{!}{\setlength{\tabcolsep}{1.8pt} 
\newcommand{\tabincell}[2]{\begin{tabular}{@{}#1@{}}#2\end{tabular}}
%\resizebox{\textwidth}{18mm}{
\begin{tabular}{c|c|c|c|c|c|c}
\hline\hline
{\bf Methods}  
& \tabincell{c}{HardNet \\ (neighbor)\\} & \tabincell{c}{HyNet \\ (neighbor)\\} & \tabincell{c}{SOLAR \\ (neighbor) \\} 
& \tabincell{c}{DELG \\ (neighbor)\\} & \tabincell{c}{CVNet \\ (neighbor)\\}
& \tabincell{c}{RobustMat \\ (GAT-PDE)\\} \\ 
\hline  
{\bf Precision}    & \underline{0.9321} \tiny{$\pm$  0.0010}   & 0.9010 \tiny{$\pm$ 0.0014}  & 0.9209 \tiny{$\pm$ 0.0013}  
& 0.9263 \tiny{$\pm$  0.0055} & 0.9183 \tiny{$\pm$ 0.0058} & \textbf{0.9486} \tiny{$\pm$  0.0037} \\ 
{\bf Recall}       & 0.7877 \tiny{$\pm$ 0.0131} & \underline{0.8613} \tiny{$\pm$ 0.0130} & 0.7923 \tiny{$\pm$ 0.0143} 
& 0.7379 \tiny{$\pm$  0.0107} & 0.7736 \tiny{$\pm$ 0.0183} & \textbf{0.8811} \tiny{$\pm$ 0.0076} \\
{\bf $F_1$-Score}  & 0.8538\tiny{$\pm$ 0.0081}  & \underline{0.8807} \tiny{$\pm$ 0.0075}  & 0.8517 \tiny{$\pm$ 0.0088} 
& 0.8214 \tiny{$\pm$  0.0085} & 0.8397 \tiny{$\pm$ 0.0126} & \textbf{0.9136} \tiny{$\pm$ 0.0035}  \\
{\bf AUC}          & \underline{0.8079} \tiny{$\pm$ 0.0066}  & 0.7887 \tiny{$\pm$ 0.0065} & 0.7941\tiny{$\pm$  0.0071} 
& 0.7809 \tiny{$\pm$  0.0109} & 0.7836 \tiny{$\pm$ 0.0136} & \textbf{0.8689} \tiny{$\pm$ 0.0046} \\
\hline\hline
{\bf Methods}  
& \tabincell{c}{HardNet \\ (neighbor)\\ +Denoising\\} & \tabincell{c}{HyNet \\ (neighbor)\\ +Denoising\\} & \tabincell{c}{SOLAR \\ (neighbor)\\ +Denoising\\} 
& \tabincell{c}{DELG \\ (neighbor)\\ +Denoising\\} & \tabincell{c}{CVNet \\ (neighbor)\\ +Denoising\\} & \tabincell{c}{RobustMat \\ (GAT-PDE)\\+Denoising\\} \\ 
\hline  
{\bf Precision}    & \underline{0.9346} \tiny{$\pm$ 0.0015} & 0.8820 \tiny{$\pm$ 0.0006} & 0.9149 \tiny{$\pm$ 0.0023} 
& 0.9281 \tiny{$\pm$ 0.0044} & 0.9323 \tiny{$\pm$ 0.0033} & \textbf{0.9350} \tiny{$\pm$ 0.0027} \\ 
{\bf Recall}       & 0.8197 \tiny{$\pm$ 0.0198} & \underline{0.8771}  \tiny{$\pm$ 0.0047} & 0.8035 \tiny{$\pm$ 0.0232} 
& 0.8427 \tiny{$\pm$ 0.0122} & 0.8707 \tiny{$\pm$ 0.0141} & \textbf{0.8939} \tiny{$\pm$ 0.0065} \\
{\bf $F_1$-Score}  & 0.8733 \tiny{$\pm$ 0.0120}  & 0.8795 \tiny{$\pm$ 0.0027} & 0.8554 \tiny{$\pm$ 0.0142} 
& 0.8833 \tiny{$\pm$  0.0065} & \underline{0.9004} \tiny{$\pm$ 0.0076} & \textbf{0.9140} \tiny{$\pm$ 0.0041} \\
{\bf AUC}          & 0.8239 \tiny{$\pm$ 0.0099}  & 0.7625 \tiny{$\pm$ 0.0024} & 0.7897 \tiny{$\pm$ 0.0116} 
& 0.8233 \tiny{$\pm$  0.0069} & \underline{0.8405} \tiny{$\pm$ 0.0070} & \textbf{0.8537} \tiny{$\pm$ 0.0057} \\
\hline\hline
\end{tabular}}
\end{table}

\subsubsection{Relation with Theory}\label{sect:relation_theory}
RobustMat exhibits robustness in performing the matching task for the synthetic and authentic noisy datasets in \cref{sec.Performance_Evaluation}, indirectly verifying the theoretical findings presented in \cref{thm.Robustness_ode,thm.Robustness_gde}. 
To see this more specifically, recall that $h_\mathrm{DS}(x)$ is the output of the downsampling CNN module, which is then input into the vertex-diffusion module to obtain $f(x)$. Finally, based on $f(x)$ and the embeddings of the neighbors of $x$, we obtain the graph-diffusion embedding $g(\calG^x)$. For each of these three embeddings, we provide the kernel density estimate plots and box plots of the Frobenius norm of the difference between the embedding of a noisy sample and its corresponding clean sample. 

In \cref{fig:h_feature}, we observe that $\| h_\mathrm{DS}({\tilde x}) - h_\mathrm{DS}(x) \|_2$ tends to increase with the level of additive noise. On the other hand, \cref{fig:f_feature,fig:g_feature} demonstrate that the vertex-diffusion embedding $f$ and graph-diffusion embedding $g$ do not differ very much between the clean and noisy samples. This observation indicates that both the vertex-diffusion and graph-diffusion embeddings of RobustMat are robust to input noise perturbations, which is consistent with the theoretical findings in \cref{thm.Robustness_ode,thm.Robustness_gde}.

\begin{figure}[!htb]
\centering
\includegraphics[width=\linewidth]{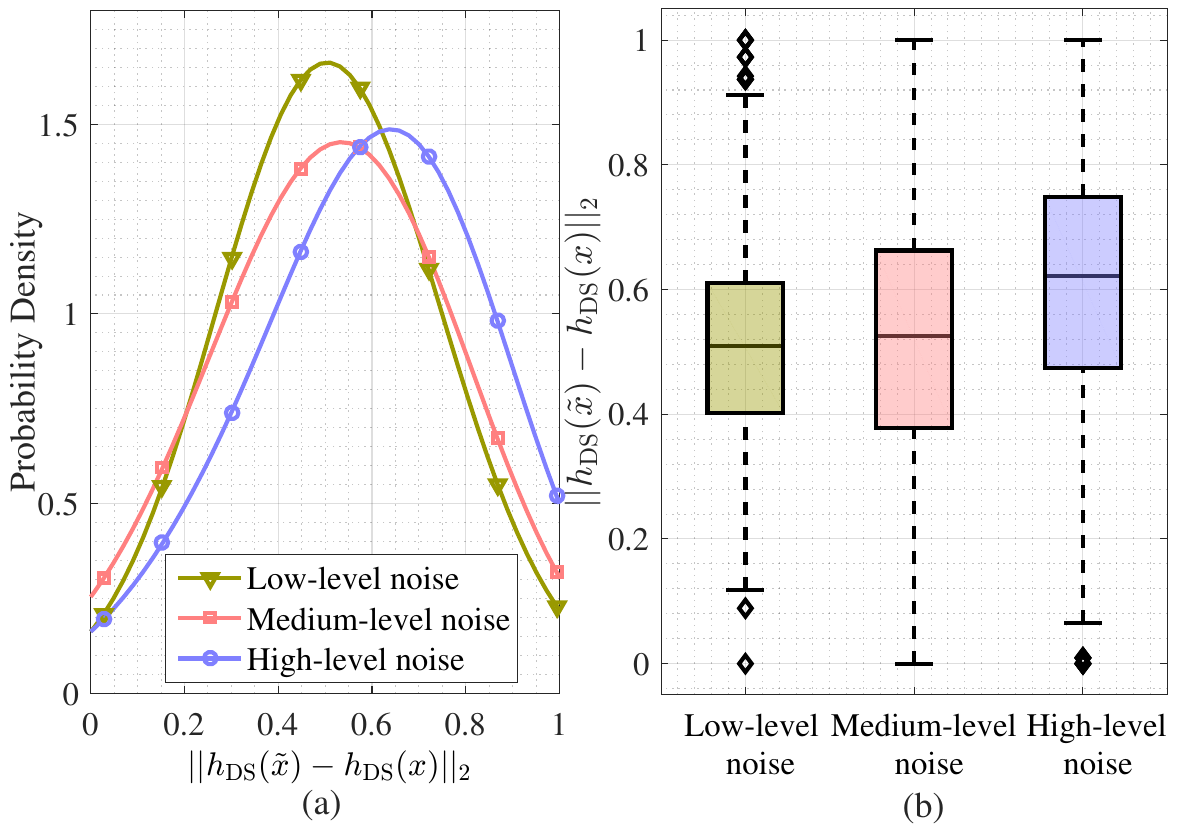}
% \vspace{-0.5cm}
\caption{
(a) Kernel density estimate plots and (b) box plots of $\| h_\mathrm{DS}({\tilde x}) - h_\mathrm{DS}(x) \|_2$ based on the RobustMat (GAT-PDE) using the Oxford dataset. The low-level, medium-level and high-level noises refer to a PSNR of around $13$ dB, $16$ dB and $19$ dB respectively, when comparing noisy samples with their corresponding clean samples.
}
\label{fig:h_feature}
\end{figure}

\begin{figure}[!htb]
\centering
\includegraphics[width=\linewidth]{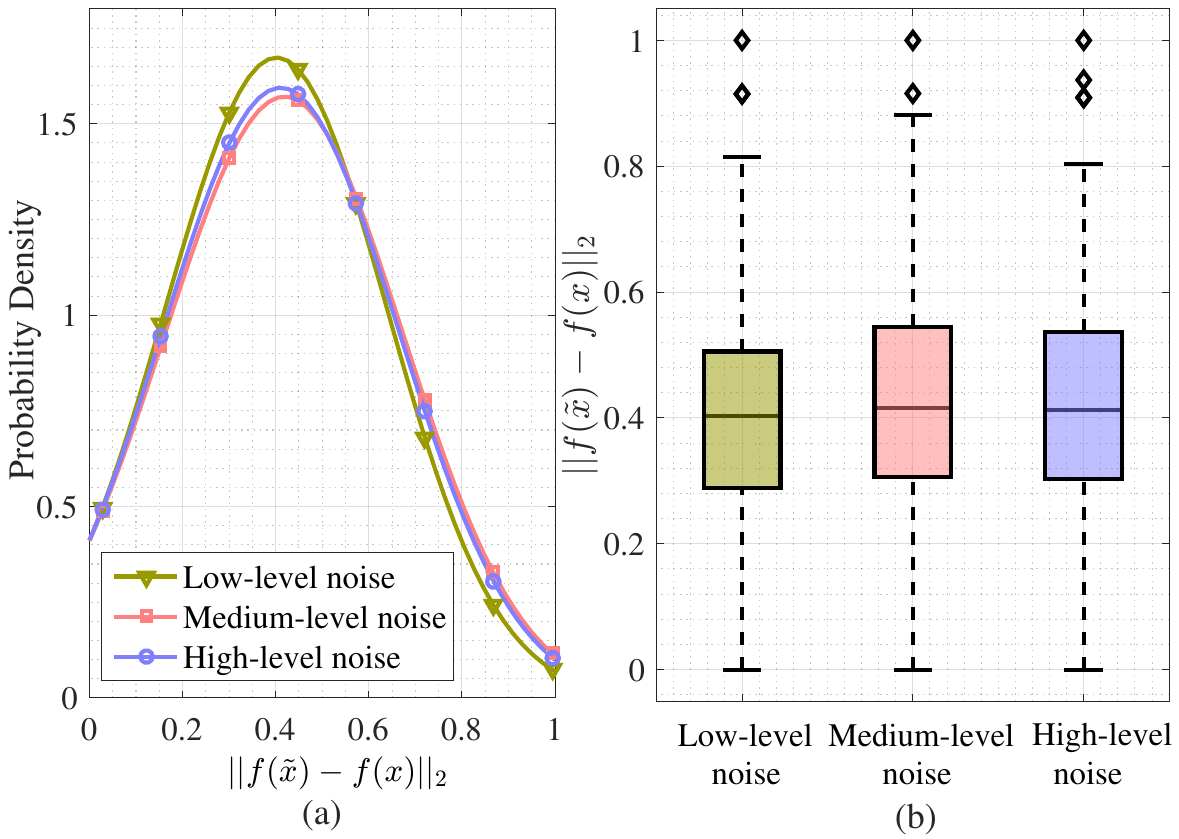}
% \vspace{-0.5cm}
\caption{
(a) Kernel density estimate plots and (b) box plots of $\| f({\tilde x}) - f(x) \|_2$ based on the RobustMat (GAT-PDE) using the Oxford dataset.
}
\label{fig:f_feature}
\end{figure}

\begin{figure}[!htb]
\centering
\includegraphics[width=\linewidth]{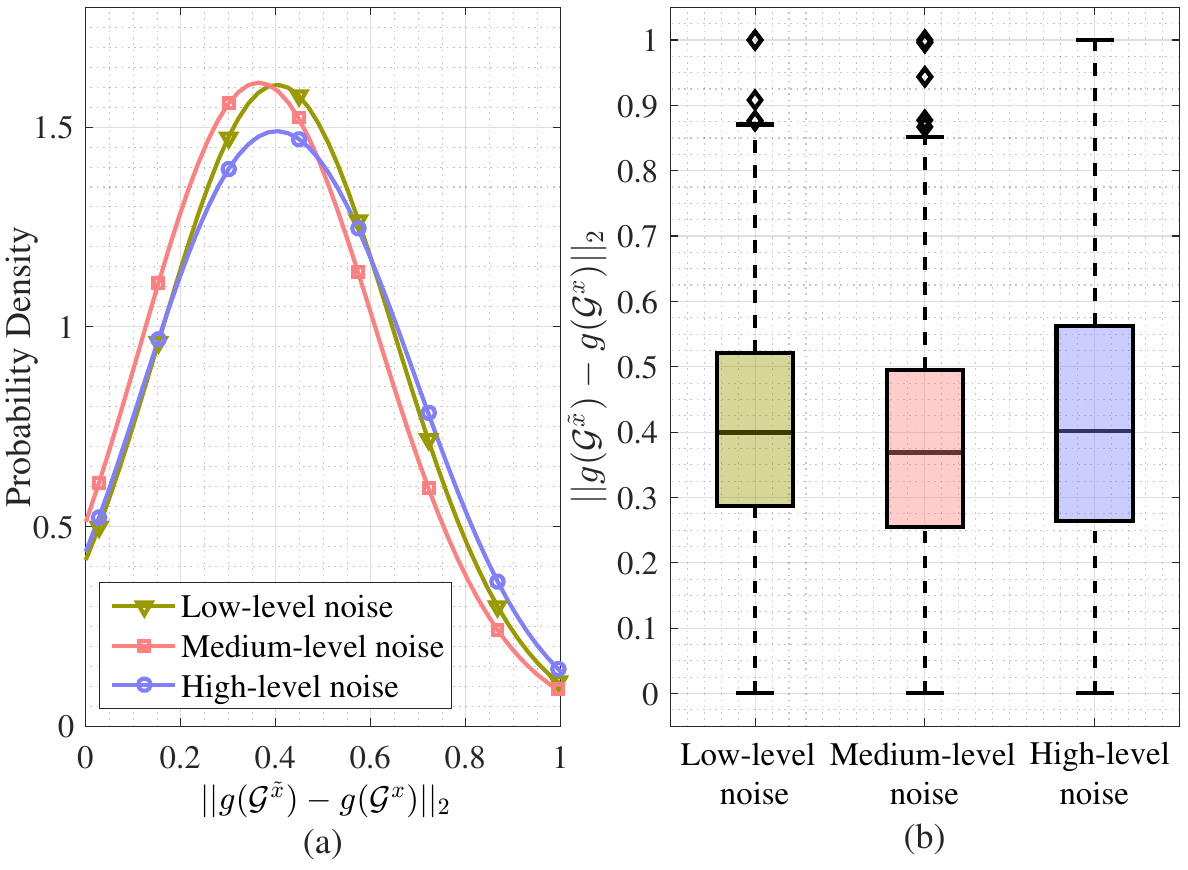}
% \vspace{-0.5cm}
\caption{
(a) Kernel density estimate plots and (b) box plots of $\| g(\calG^{\tilde x}) - g(\calG^{x}) \|_2$ based on the RobustMat (GAT-PDE) using the Oxford dataset.
}
\label{fig:g_feature}
\end{figure}

\subsection{Possible Applications}\label{sect:application}

In this section, some autonomous vehicle (AV) applications, as illustrated in \cref{fig:figure_vvapp_model}, are discussed.   

\textbf{Vehicular place recognition based on landmark maps:} 
RobustMat can be used to achieve local patch matching to determine the corresponding global image frame matching between the current frame extracted from a vehicle and keyframes from landmark maps \cite{sunderhauf2015place}. 
In the landmark maps, the landmark patches captured from a section of a road serve as the keyframe representatives for the beginning location of the road section. 
If a vehicular frame matches a particular keyframe representative, the location of the corresponding road section is regarded as the predicted location for the vehicle.
Furthermore, scenarios involving perturbations, such as variations in view angles, weather conditions, and illuminations, present common challenges in place recognition. To address these issues, RobustMat provides a promising solution based on robust landmark patch matching.
As demonstrated in \cref{Boreas-table} and \cref{fig_boreas_noise}, RobustMat performs robustly against various real environmental noises, such as weather and illuminations. This adaptability to noisy scenarios ensures the effectiveness and efficiency of RobustMat in place recognition tasks.

In addition, when dealing with a substantial landmark dataset for matching, RobustMat may require a longer processing time. 
To mitigate this issue, we can employ model optimization and pruning techniques \cite{liu2018rethinking,akiba2019optuna,fang2023structural} or explore other acceleration methods during the testing process.
To limit the number of landmark patches, we can select more appropriate landmark patches by setting pixel sizes and resolutions.  We can also leverage parallel computing for the matching inference to further reduce computational complexity.

\textbf{Odometry assistance:} For a moving vehicle, there may be shared landmarks between two frames captured from adjacent locations. RobustMat can match the landmarks in a current frame to those in the previous one. 
The matched landmarks can then serve as references for the LiDAR points or key visual pixels which are used in pixel- or point-level matching. More LiDAR points or key visual pixels on the matched landmarks can be used. This can be regarded as a kind of importance sampling, which makes LiDAR or keypoint odometry more accurate and efficient.
To further improve the accuracy of odometry estimation, it is promising to combine pixel- or point-level matching methods \cite{dusmanu2019d2,luo2020aslfeat,sarlin2020superglue,sun2021loftr} with our patch-level matching.
Specifically, the patch-level matching method can provide an attention region for the keypoints in pixel- or point-level methods, serving as a post-processing step to filter out weak corresponding keypoints.
One can also use the image patches of the shared landmarks and the LiDAR points on these landmarks as the input to a neural network to predict the odometry \cite{Kang22itsc}.  

Our method exploits only landmark patches rather than the whole image frame, thus saving storage memory and communication resources. 

\begin{figure}[!htb]
\centering
\includegraphics[width=\linewidth]{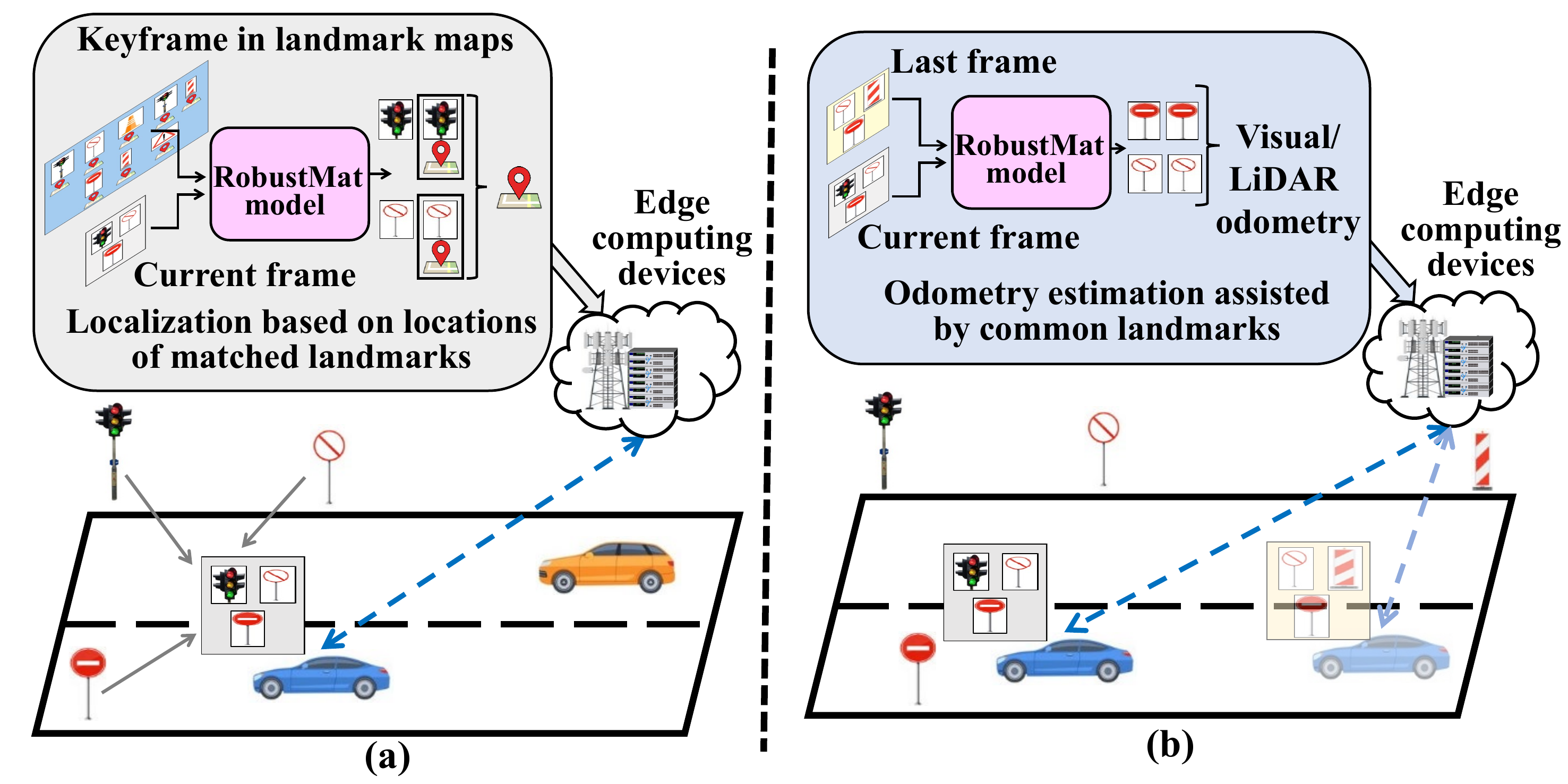}
\caption{The diagram of applications of RobustMat for AVs, where (a) is for vehicular place recognition based on landmark maps and (b) is for odometry assistance.}
\label{fig:figure_vvapp_model}
% \vspace{-0.5cm}
\end{figure}

\section{Conclusion}\label{sect:conc}

To improve the robustness of landmark patch matching, which is an essential component of autonomous driving systems, we have incorporated neural ODE/PDE modules for patch self-representation and patch neighborhood representation. Theoretical analysis shows the resilience of our approach against additive perturbations. Empirical evaluation on benchmark datasets has demonstrated the superiority of our approach over baseline methods, especially in the presence of noisy perturbations. Our results confirm the efficacy of using neural ODE/PDE modules for robust patch-matching.

An interesting future work is to adapt our approach to point-level matching methods or use our method as a post-processing step to emphasize regions where more important keypoints are. Methods mixing patch-level and point-level matching may further improve the accuracy and robustness.

\section*{Acknowledgments}
To improve the readability, parts of this paper have been grammatically revised using ChatGPT \cite{OpenAI}.

\appendices

\section{Proof of \cref{thm.Robustness_ode}}\label{app:thm.Robustness_ode}

Let $\tilde{z}(t)$ denote the trajectory $z(t)$ in \cref{eq:NODE} for the perturbed input $\tilde{z}(0)=h_\mathrm{DS}({\tilde x})$.
From \cref{ass:ass_ivp} and Gronwall’s Lemma \cite{Gronwall1998,yan2019robustness,dupont2019augmented}, we have $\norm{{\tilde z}(t)-z(t)} \le \exp{(C_0 t)}\norm{\tilde{z}(0)-z(0)}=\varepsilon \exp{(C_0 t)}$, where $C_0 \ge 0$ is a constant. 
For sufficiently small $\varepsilon$, $\norm{{\tilde z}^{(i)}(t)-z^{(i)}(t)}$ can be made arbitrarily small for all $t\in[0,T_f]$.
We then have 
\begin{align}
    \frac{h_{\mathrm{CNN}}^{(i)}({\tilde z}(t), t) - h_{\mathrm{CNN}}^{(i)}(z(t), t)}{{\tilde z}^{(i)}(t)-z^{(i)}(t)} = -C_{\mathrm{CNN}}^{(i)}(t),
\end{align}
where $C_{\mathrm{CNN}}^{(i)}(t)>0$ is an element of $C_{\mathrm{CNN}}(t)$, a $c\times H \times W$ matrix. 
I.e., we have 
\begin{align}
h_{\mathrm{CNN}}({\tilde z}(t), t) - h_{\mathrm{CNN}}(z(t), t) = -C_{\mathrm{CNN}}(t) \odot [{\tilde z}(t) - z(t) ],
\end{align}
where $\odot$ denotes the Hadamard product. 

For a sufficiently small time slot ${\Delta t}$ such that $ C_z^*\Delta t < 1$, we obtain
\begin{align}
& \big\| [h_{\mathrm{CNN}}({\tilde z}(t), t) - h_{\mathrm{CNN}}(z(t), t)]{\Delta t} + [{\tilde z}(t) - z(t) ] \big\|_2 \nn
& \le  (1 - C_{z}^*{\Delta t}) \|{\tilde z}(t) - z(t) \|_2. \label{eq.CNN_c_z*}
\end{align}
When $\Delta t \to 0$, we have 
\begin{align}
& \| {\tilde z}(t+{\Delta t}) - z(t+{\Delta t}) \|_2 - \|{\tilde z}(t) - z(t) \|_2 \nn
&= \big\| [{\tilde z}(t) + h_{\mathrm{CNN}}({\tilde z}(t), t){\Delta t}] - [ z(t) + h_{\mathrm{CNN}}(z(t), t){\Delta t} ] \big\|_2 \nn
& \quad - \|{\tilde z}(t) - z(t) \|_2 + o(\Delta t) \nn
& \le -C_{z}^*{\Delta t} \|{\tilde z}(t) - z(t) \|_2, \label{eq.Delta_0_CNN}
\end{align}
and 
\begin{align}
& \ddfrac{\|{\tilde z}(t) - z(t) \|_2}{t} \nn 
& = \lim_{{\Delta t}\to 0} \frac{\| {\tilde z}(t+{\Delta t}) - z(t+{\Delta t}) \|_2 - \|{\tilde z}(t) - z(t) \|_2}{\Delta t} \nn
&\le -C_{z}^* \|{\tilde z}(t) - z(t) \|_2. \label{eq.dd_z_CNN}
\end{align}
Furthermore, by integrating \cref{eq.dd_z_CNN} from $0$ to $T_f$, we have
\begin{align}
& \|{\tilde z}(T_f) - z(T_f) \|_2 \le \exp(-C_{z}^*T_f) \|{\tilde z}(0) - z(0) \|_2. \label{eq.dd_z_T_1}
\end{align}
Combining \cref{eq.dd_z_T_1} and \cref{ass:ass_fc}, we have
\begin{align}
    \| f(\tilde x) - f(x) \|_2 
    & \le C_{\mathrm{fc}} C_{\mathrm{fp}} \| {\tilde z}(T_f) -z(T_f) \|_2 \nn
    & \le C_{\mathrm{fc}}C_{\mathrm{fp}}  \exp{(-C_{z}^* T_f)} \| {\tilde z}(0) - z(0) \|_2 \nn
    & = \calO(\varepsilon\exp{(-C_{z}^* T_f)}),
\end{align}
which completes the proof.

\section{Proof of \cref{thm.Robustness_gde}}\label{app:thm.Robustness_gde}
Let $\tilde{Z}(t)$ denote the trajectory $Z(t)$ based on \cref{eq:GRAPH} for the perturbed input $\tilde{Z}(0)=(f(\tilde{v}))_{\tilde{v} \in \calV^{\tilde{x}}}$.
Similar to \cref{thm.Robustness_ode}, we have 
\begin{align}
    & \|\tilde Z(T_g) - Z(T_g) \|_2 \le \exp{(-C_A^* T_g)}  \|\tilde Z(0) - Z(0) \|_2.
\end{align}

From \cref{thm.Robustness_ode}, it follows that $\|\tilde Z(0) - Z(0) \|_2 \le \calO(\varepsilon)$.
From the Lipschitz continuous $h_\mathrm{gp}(\cdot)$ in \cref{ass:ass_Amatrix}, we have 
\begin{align}
    \| g(\calG^{\tilde x}) - g(\calG^x) \|_2 
    & \le C_{\mathrm{gp}} \|\tilde Z(T_g) - Z(T_g) \|_2 \nn
    & \le C_{\mathrm{gp}} \exp{(-C_A^* T_g)}  \|\tilde Z(0) - Z(0) \|_2 \nn
    & \le \calO(\varepsilon\exp{(-C_A^* T_g)}),
\end{align}
where $C_{\mathrm{gp}}$ is a Lipschitz constant. The proof is now complete. 

\bibliographystyle{IEEEtran}
\bibliography{IEEEabrv,StringDefinitions,manuscript_reference}

% \bf{If you will not include a photo:}\vspace{-33pt}
% \begin{IEEEbiographynophoto}{John Doe}
% Use $\backslash${\tt{begin\{IEEEbiographynophoto\}}} and the author name as the argument followed by the biography text.
% \end{IEEEbiographynophoto}

\newpage
\text{ }
\newpage

\section*{Supplementary Material}

\subsection{Dataset Preparation Details}\label[Appendix]{app.Landmarkpre}

\subsubsection{Landmark Patch Preparation}

To provide the training and testing data for the landmark patch matching task, we make use of the full-sized street image datasets including KITTI dataset\footnote{http://www.cvlibs.net/datasets/kitti/} and the Oxford Radar RobotCar dataset\footnote{http://ori.ox.ac.uk/datasets/radar-robotcar-dataset}, where $3$D LiDAR points are used for the ground truth matching.

In the literature of landmark-based applications, Edge Boxes are used to detect a bounding box around a patch that contains a large number of internal contours compared to the number of contours exiting from the box, which indicates the presence of an object in the enclosed patch. The DeepLabV3+ is used to extract significant landmark regions. However, all of the aforementioned patch extraction or landmark detection approaches are not stable when removing dynamic objects and many noisy regions are presented. 
By contrast, in our datasets, we exploit Faster R-CNN as the stable landmark object detector to locate the region of interest for static roadside objects including traffic lights, traffic signs, poles, and windows. 

In Faster R-CNN, we choose Resnet50 with Feature Pyramid Network (FPN) as the backbone, which is already pre-trained on the Imagenet dataset. During training, we use Adam optimizer with learning rate $0.0002$ and weight decay $0.0001$ to train the detector for $50$ epochs. The training batch size is set as $2$ and random horizontal flipping is used for data augmentation. For the KITTI and Oxford Radar RobotCar datasets, the stereo cameras capture full-sized image frames, and we use only \emph{the left frames} to extract landmark patches. The ground truth landmark object bounding box labels are from the segmentation masks for the KITTI dataset and manually collected for the Oxford Radar RobotCar dataset, respectively. 

To add environmental perturbations or noises to the patches, we exploit the Python library ``imgaug''\footnote{https://imgaug.readthedocs.io/en/latest/}. 
In particular, the perturbations include additive white Gaussian noise, image corruption (like rain and spatter), and brightness and saturation changes.  

To obtain the matching ground truth for patch pairs, vehicle ground truth locations and collected LiDAR scans are used. Similar to the processing in other patch-matching datasets, a $3$D LiDAR reference map is built to find the landmark locations after projecting $3$D LiDAR points to the frames. Then, the $\mathcal{L}_2$ distance between two landmark patches is computed, which determines the ground truth of the patch matching. 

More details for the KITTI and Oxford Noisy Landmark Patches are given as follows.

\begin{table*}[!htb]
%\scriptsize
%\footnotesize
\caption{Neural network modules and the parameters.}
\label{model-setting}
\begin{center}
\newcommand{\tabincell}[2]{\begin{tabular}{@{}#1@{}}#2\end{tabular}}
%\resizebox{\textwidth}{18mm}{
\begin{tabular}{cclc}
\toprule\toprule
\multicolumn{1}{c}{\bf Mapping}  & {\bf Models} & {\bf Layers (model parameters) } 
& \tabincell{c}{{\bf Dimension} {\bf of Outputs} }
\\ \midrule
$f$ & \tabincell{c}{Vetex-diffusion \\ embedding\\}      
                        & \tabincell{l}{A CNN (with 7 convolutional layers) \& \\
                        a CNN-based neural ODE (with the two-convolutional layer CNN \\
                        as the ODE function, the integration time as [0,1], the relative and absolute \\
                        tolerances both as 0.01, in the module ``odeint'') \& \\ 
                        an average pooling and an FC layer\\}
                        & \tabincell{c}{128*64*64 \\ 128*64*64\\ \ \\ \  \\ 512\\} \\ \midrule
\multirow{1}{*}{$g$} 
    & \tabincell{c}{Graph-diffusion \\ embedding (with \\ GAT-based/ \\ GCN-based neural PDE) \\}       
                        & \tabincell{l}{The ``odeint'' module for ODE solution with the integration time as [0,1], \\ 
                                        the relative and absolute tolerances both as 0.001, the method as ``dopri5'', \\
                                        the ODE function as the GAT/GCN including\\
                                        GAT block 1 (4 attention heads, 
                                        4*128 hidden features \& ELU) \\
                                        GAT block 2 (4 attention heads, 
                                        4*128 hidden features \& ELU)/ \\
                                        GCN block 1 (512 hidden features \& ReLU) \\
                                        GCN block 2 (512 hidden features \& ReLU)\\
                                        } 
                        & \tabincell{l}{512}\\ \midrule
$r$ & \tabincell{l}{FC Neural Network\\}
                        & \tabincell{l}{Linear layers ([1024,512,256] 
                                        hidden features \& ReLU) \\
                        Linear layer (1 hidden feature \& Sigmoid function) \\}
                        & \tabincell{c}{256\\ 1\\} \\ \midrule
$d$ & Discriminator      & \tabincell{l}{Bilinear layer (one 512*512 hidden matrix 
                                        \& Sigmoid function)\\} 
                        & \tabincell{l}{1\\} \\
\bottomrule\bottomrule
\end{tabular}%}
\end{center}
\end{table*}

\begin{figure}[!htb]
\centering
\includegraphics[width=0.5\linewidth, height=0.06\textheight]{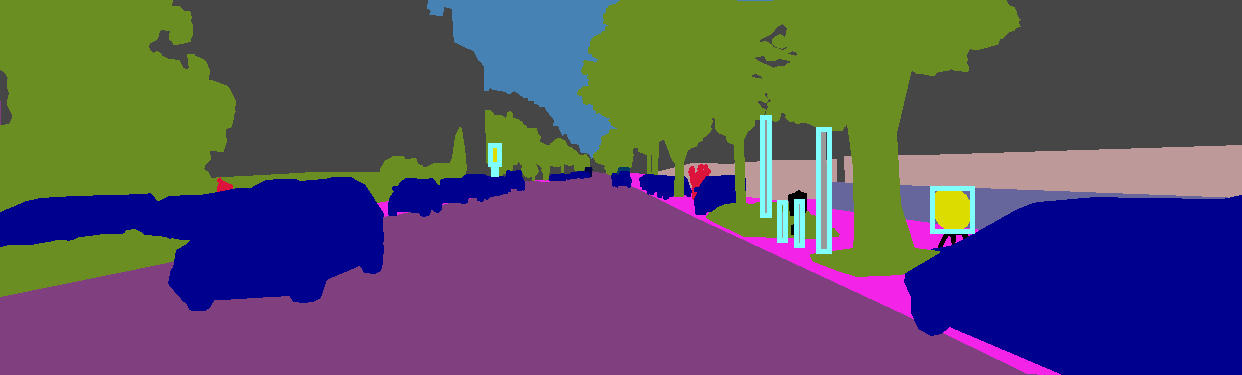}\hfill
\includegraphics[width=0.5\linewidth, height=0.06\textheight]{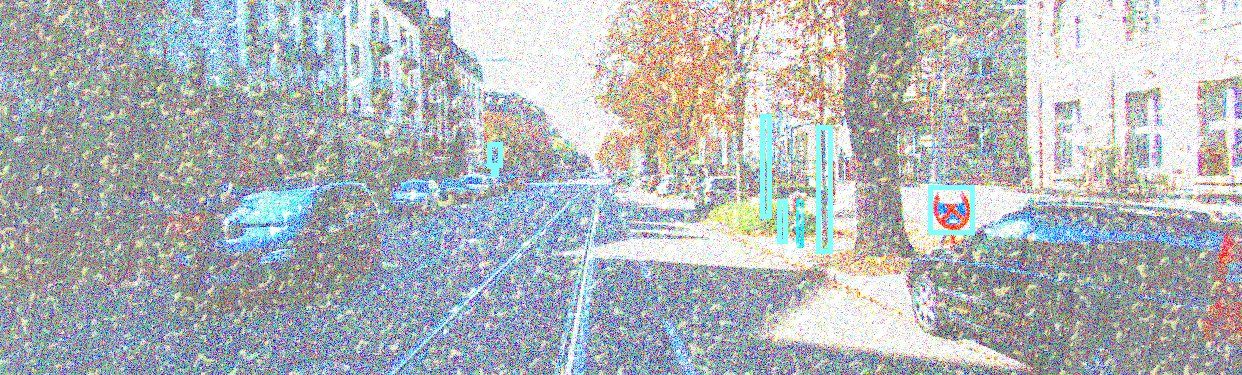}\hfill
\includegraphics[width=0.5\linewidth, height=0.06\textheight]{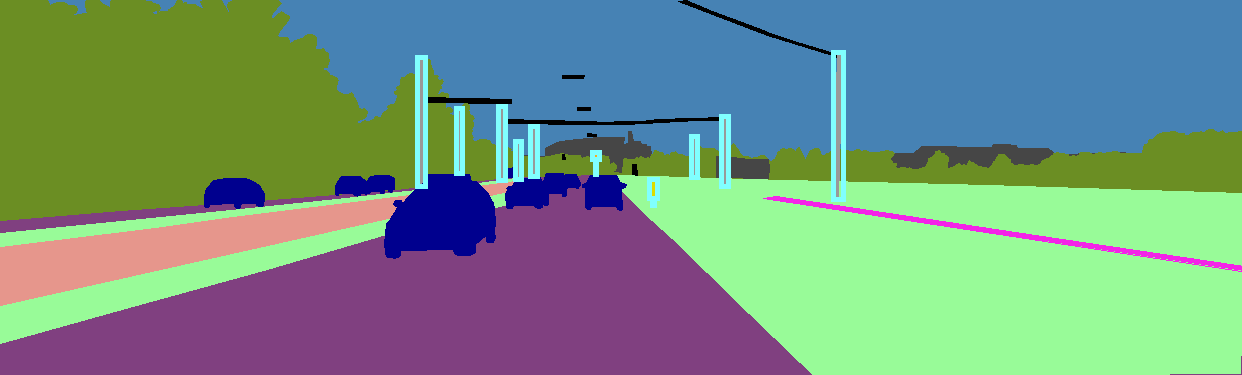}\hfill
\includegraphics[width=0.5\linewidth, height=0.06\textheight]{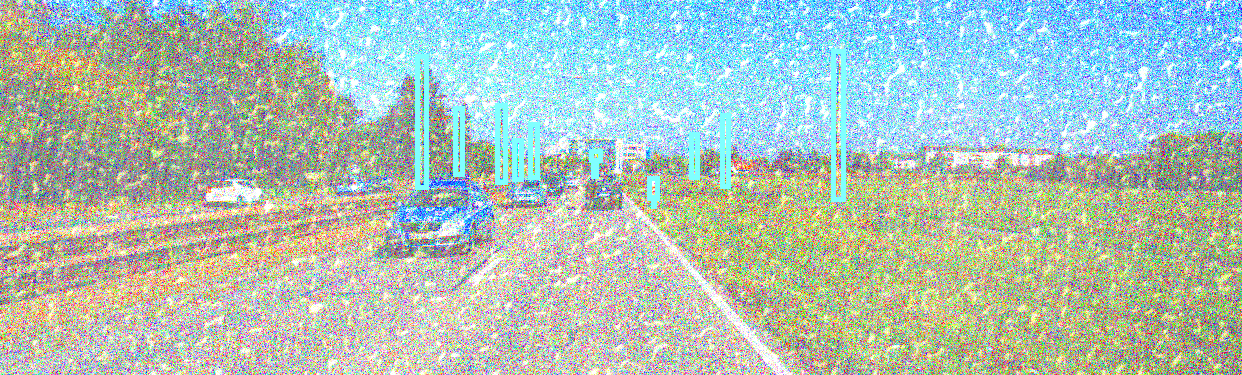}\hfill
\includegraphics[width=0.5\linewidth, height=0.06\textheight]{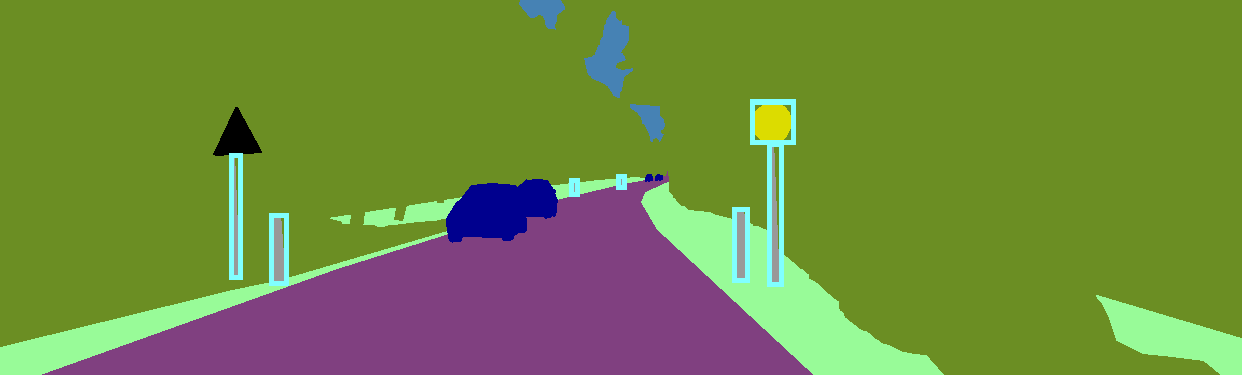}\hfill
\includegraphics[width=0.5\linewidth, height=0.06\textheight]{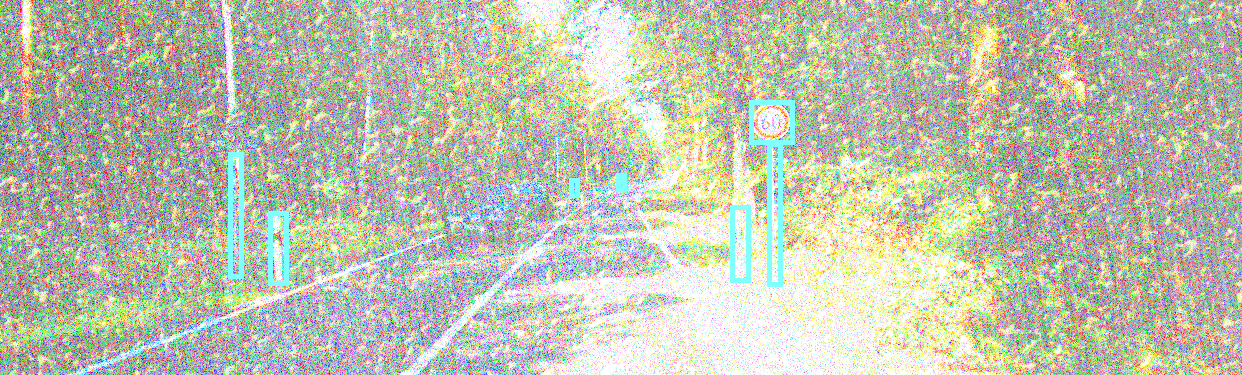}\hfill
\includegraphics[width=0.5\linewidth, height=0.06\textheight]{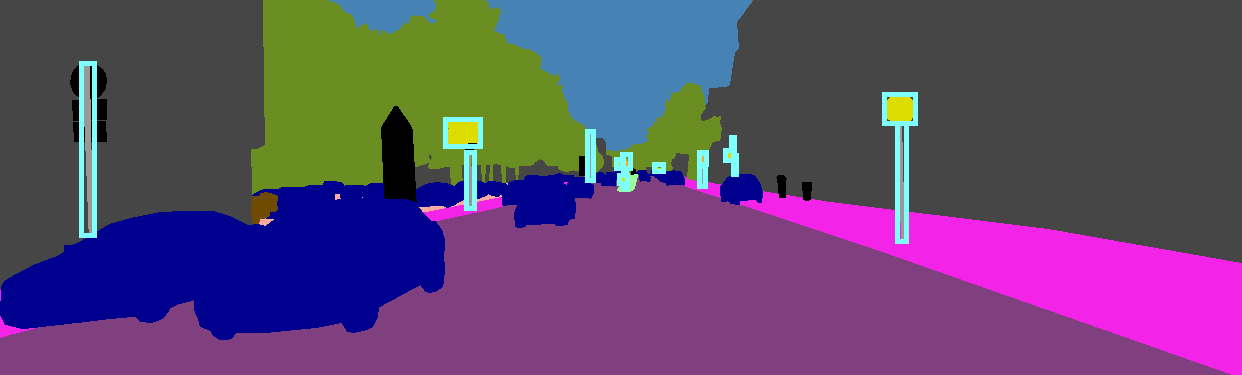}\hfill
\includegraphics[width=0.5\linewidth, height=0.06\textheight]{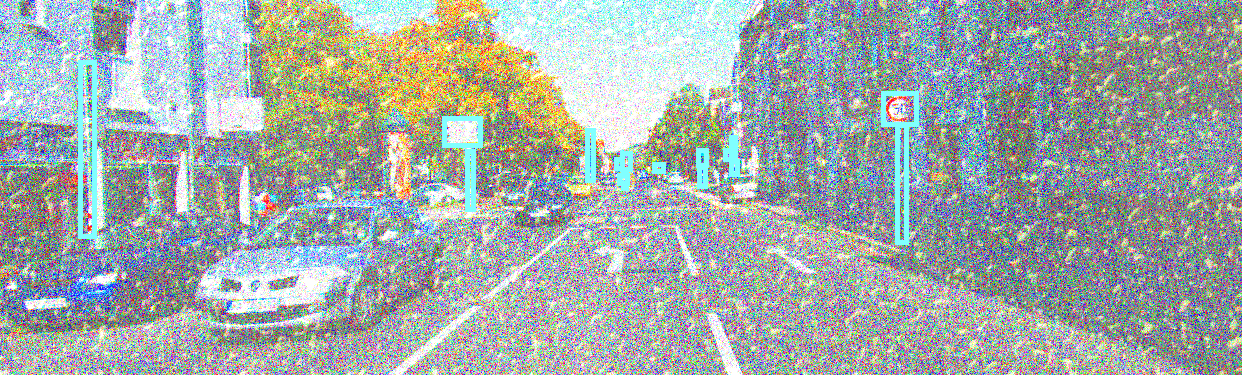}\hfill
%\includegraphics[width=0.5\linewidth, height=0.06\textheight]{000085_10_sem.pdf}\hfill
%\includegraphics[width=0.5\linewidth, height=0.06\textheight]{000085_10_bbox.pdf}\hfill
%%\vspace{-0.3cm}
\caption{Examples of ground truth landmark bounding box labels based on semantic segmentation masks in the KITTI dataset, w.r.t.\ the semantic segmentation images and the noisy images.
}
\label{fig:label_kitti}
%%\vspace{-0.5cm}
\end{figure}

\textbf{KITTI Noisy Landmark Patches.}
%
%The object labels given by \cite{Alhaija2018IJCV} are used for the landmark patch acquisition. 
The object labels are semantic segmentation masks, which can be used for landmark patch acquisition. 
To perform landmark object detection, we need to first convert the semantic segmentation labels to object bounding box labels. We use the ``skimage.measure.label'' to label connected regions for pixel classes including traffic lights, traffic signs and poles. See \cref{fig:label_kitti} for an example. In some rare cases, multiple poles overlap, and the connected region algorithm outputs an inaccurate bounding box. We exclude these overlapped objects in the generated bounding box labels. We project the surrounding LiDAR points, as mentioned in the section on dataset preparation, onto the image frame plane using the intrinsic camera matrix and extrinsic camera matrix. Here, we have used the sensors' information (i.e., vehicle global ground truth locations) to accumulate collected LiDAR scans to build the $3$D LiDAR reference map. Due to the limited LiDAR field of view, a single LiDAR scan may not have any LiDAR point corresponding to some landmarks. To avoid this, we build a $3$D LiDAR reference map similar to that in PointNetVlad. 
We have also gone through all the frames manually to remove or correct a few landmark objects with low quality. Finally, for each detected landmark object, we intentionally expand its bounding box by $15$ pixels on each side to include some background information. 

\begin{figure}[!htb]
\centering
\includegraphics[width=0.5\linewidth, height=0.12\textheight]{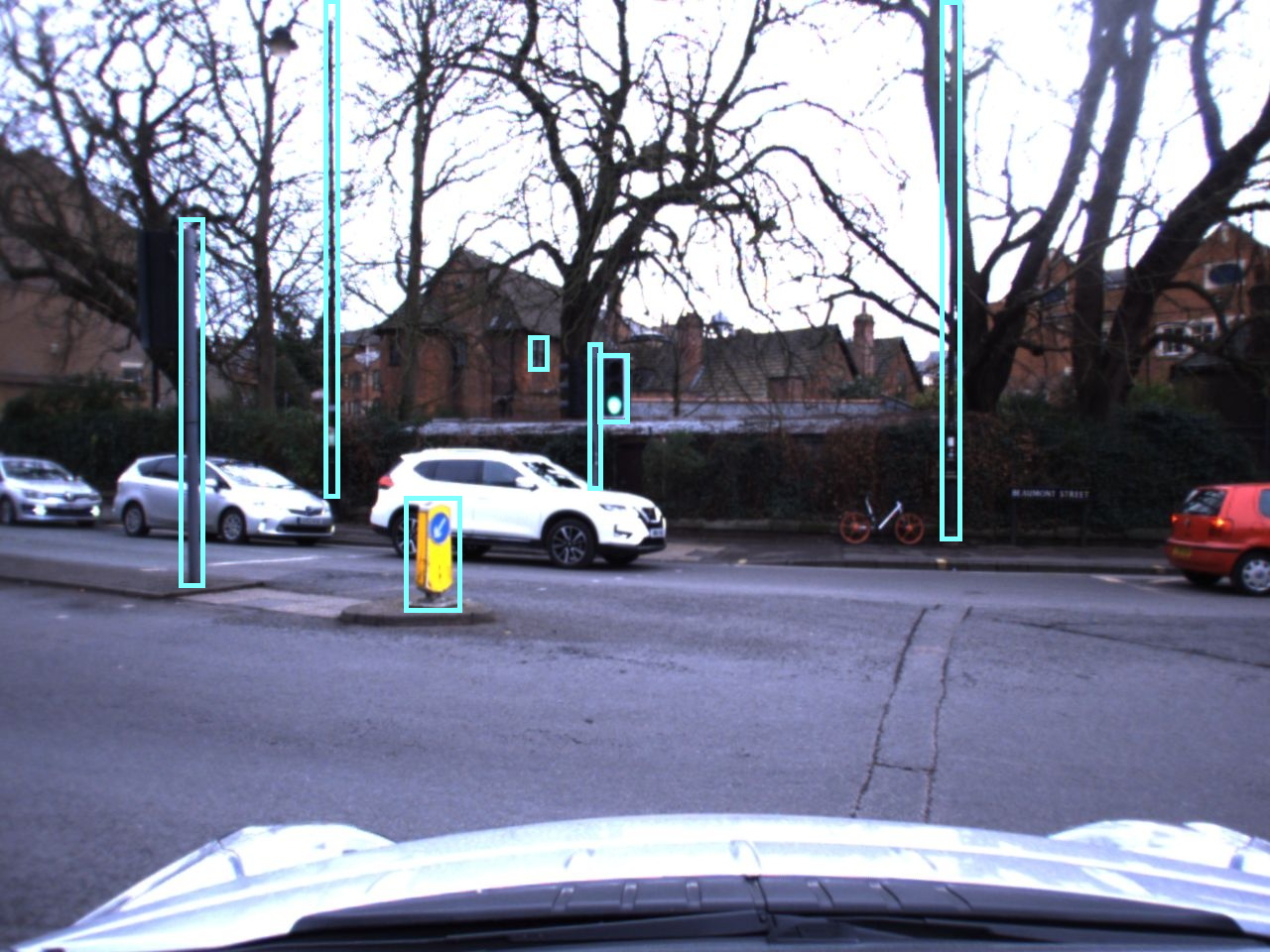}\hfill
\includegraphics[width=0.5\linewidth, height=0.12\textheight]{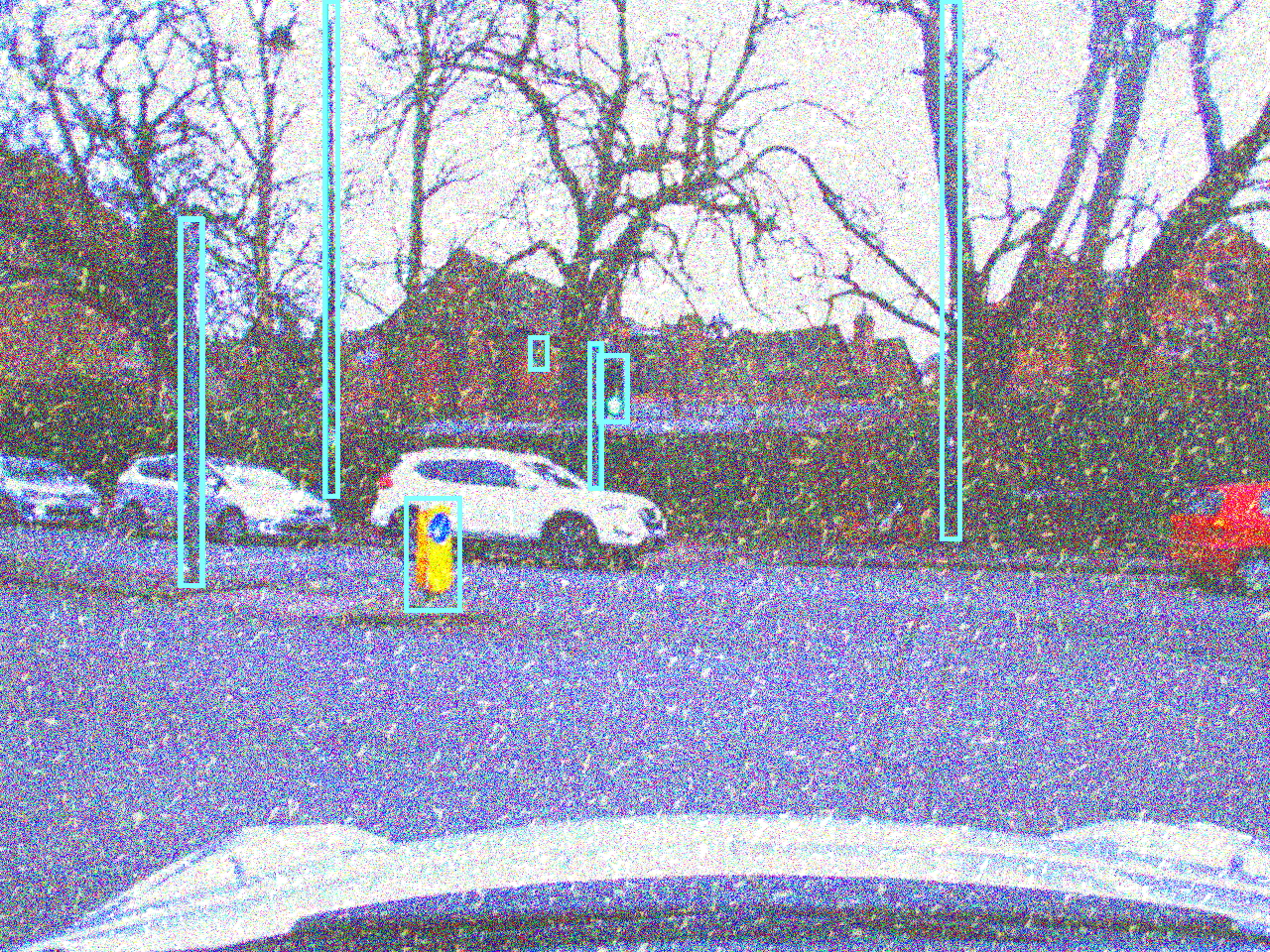}\hfill
\includegraphics[width=0.5\linewidth, height=0.12\textheight]{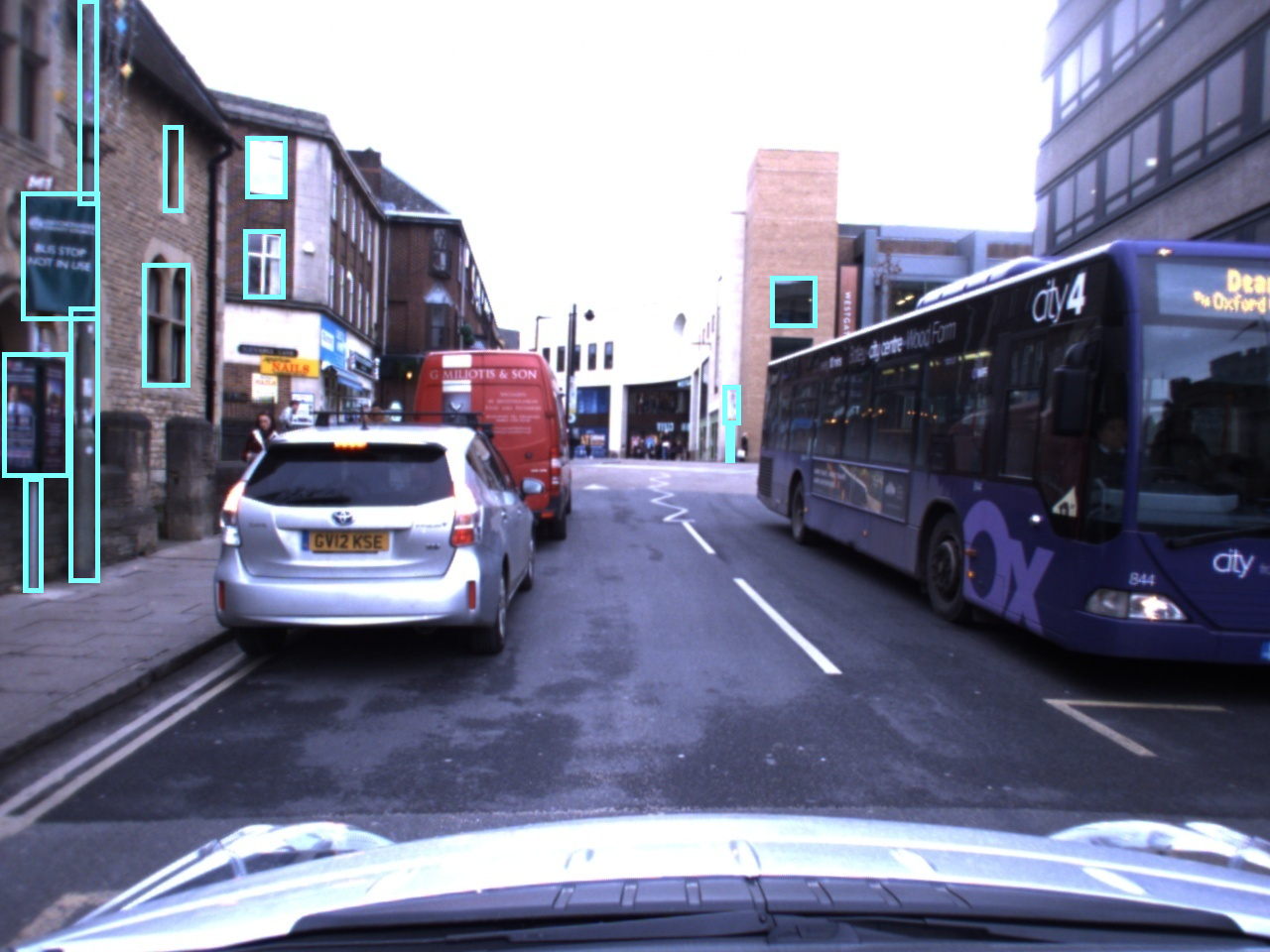}\hfill
\includegraphics[width=0.5\linewidth, height=0.12\textheight]{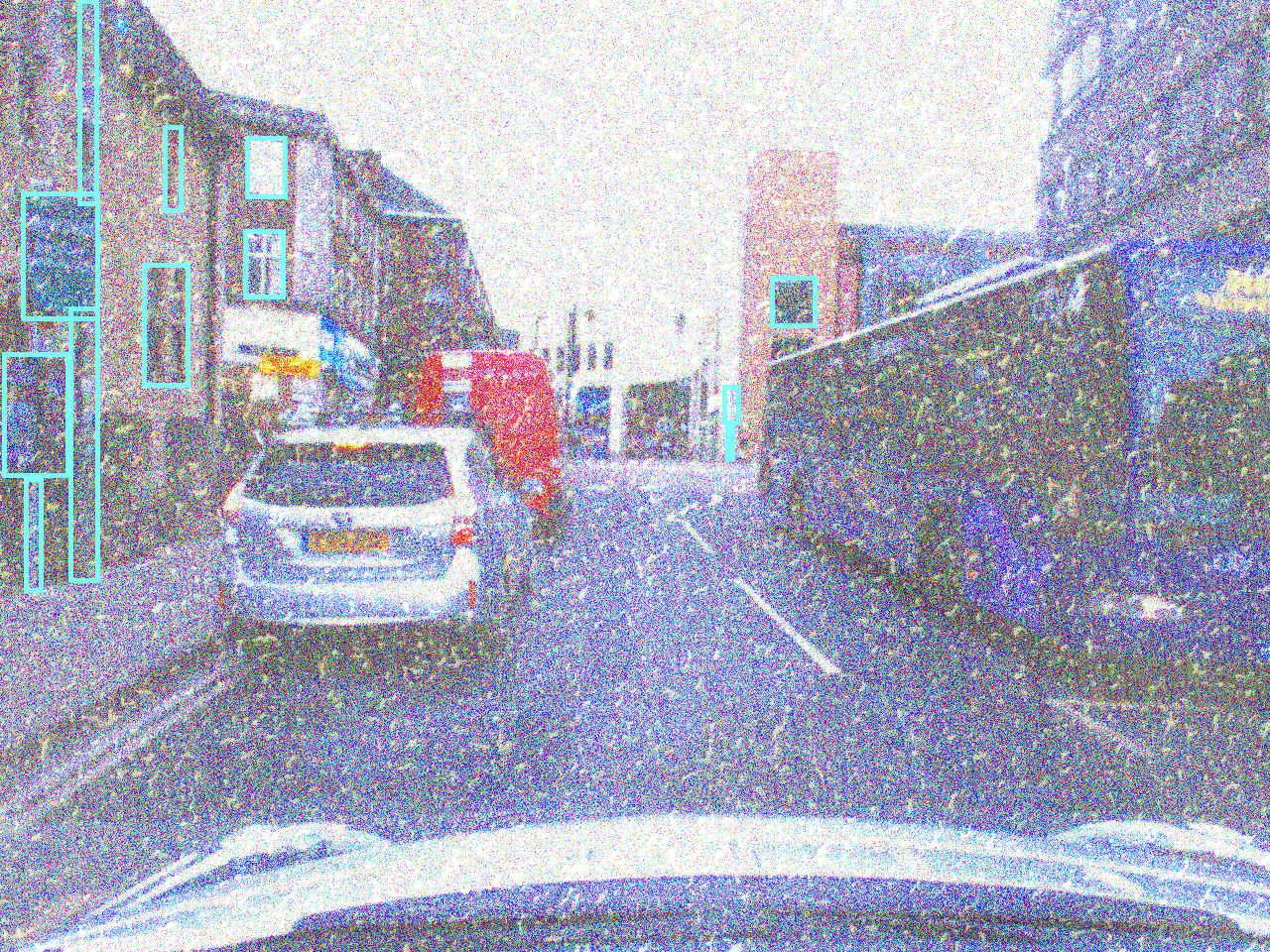}\hfill
\includegraphics[width=0.5\linewidth, height=0.12\textheight]{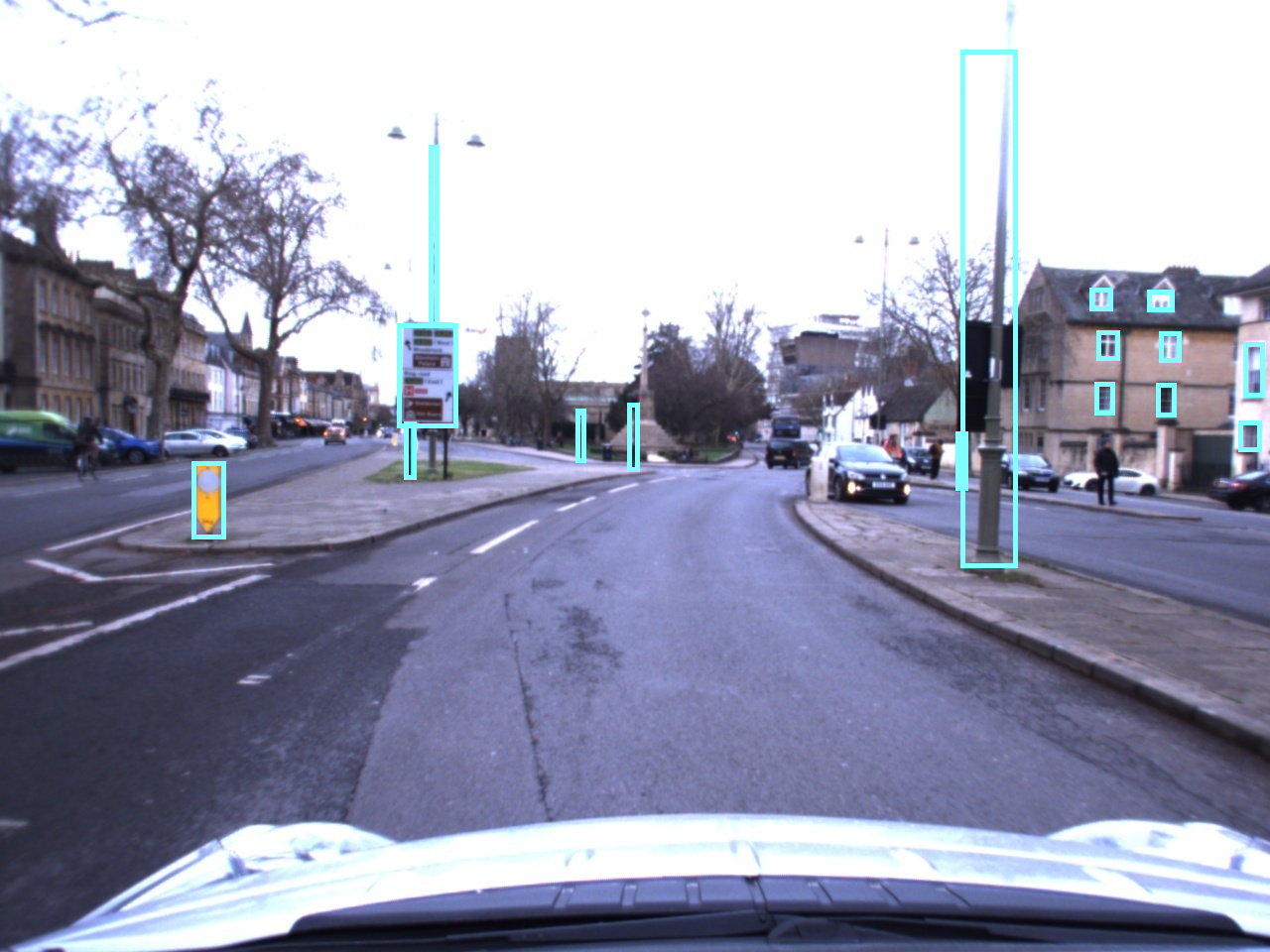}\hfill
\includegraphics[width=0.5\linewidth, height=0.12\textheight]{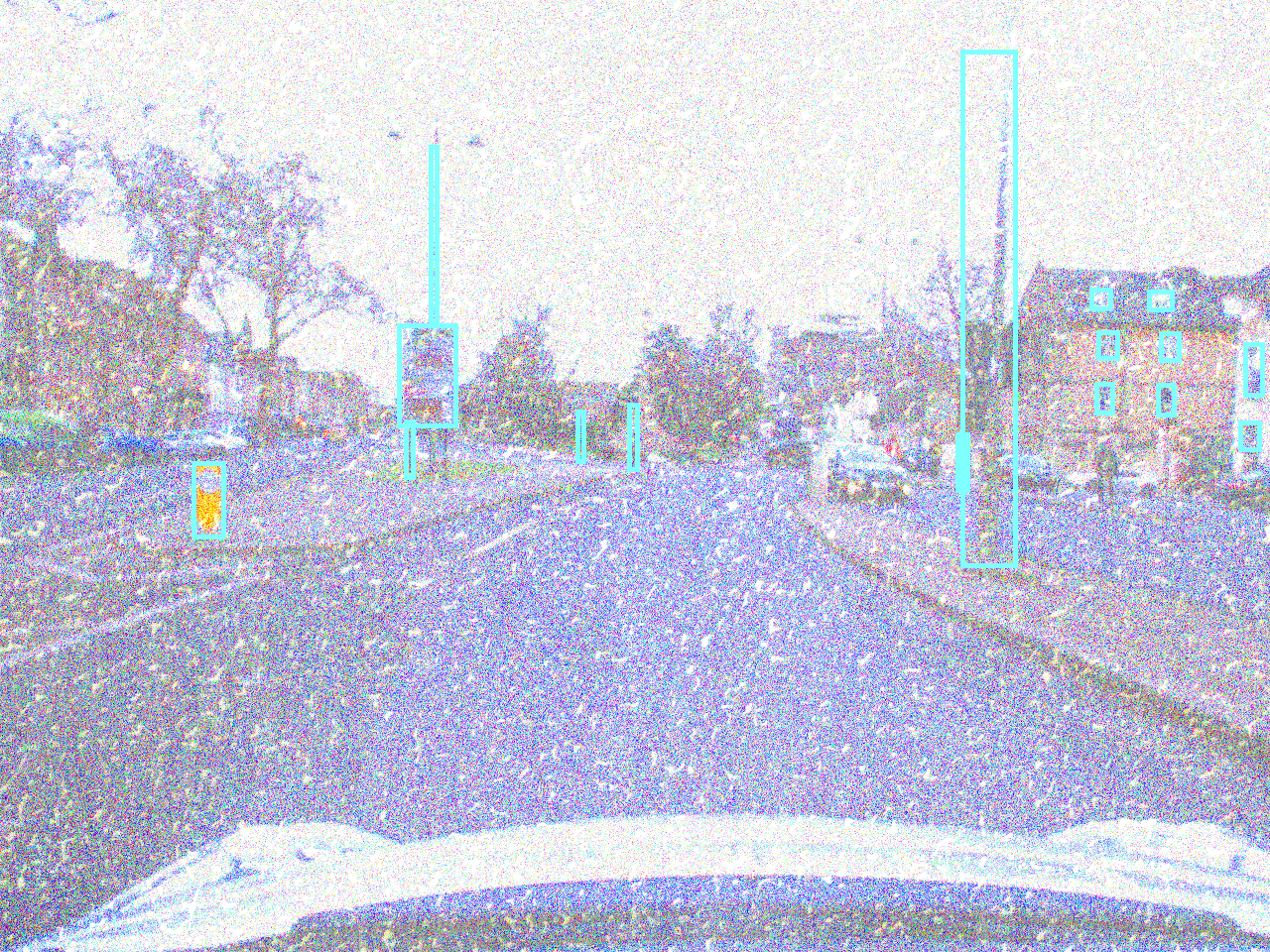}\hfill
\caption{Examples of the ground truth landmark bounding box labels for the Oxford Radar RobotCar dataset, w.r.t.\ the clear and the noisy images.}
\label{fig:label_oxford}
\end{figure}

\textbf{Oxford Noisy Landmark Patches.} 
To build the dataset of Oxford Noisy Landmark Patches, we manually labeled landmarks including traffic lights, traffic signs, poles, and windows for $500$ frames. See \cref{fig:label_oxford} for an example. Compared with the landmark patches from the KITTI dataset, we additionally include the window class in this dataset. 
(Window labels are not available for the KITTI dataset yet. We will enrich the KITTI dataset with window labels in future work. Other building facade elements  like doors and balconies will also be added to the two landmark patch datasets.) 
Similar to the operations for the KITTI dataset, we obtain the final landmark patches for the Oxford Radar RobotCar dataset. 

\begin{figure}[!htb]
\centering
\includegraphics[width=\linewidth]{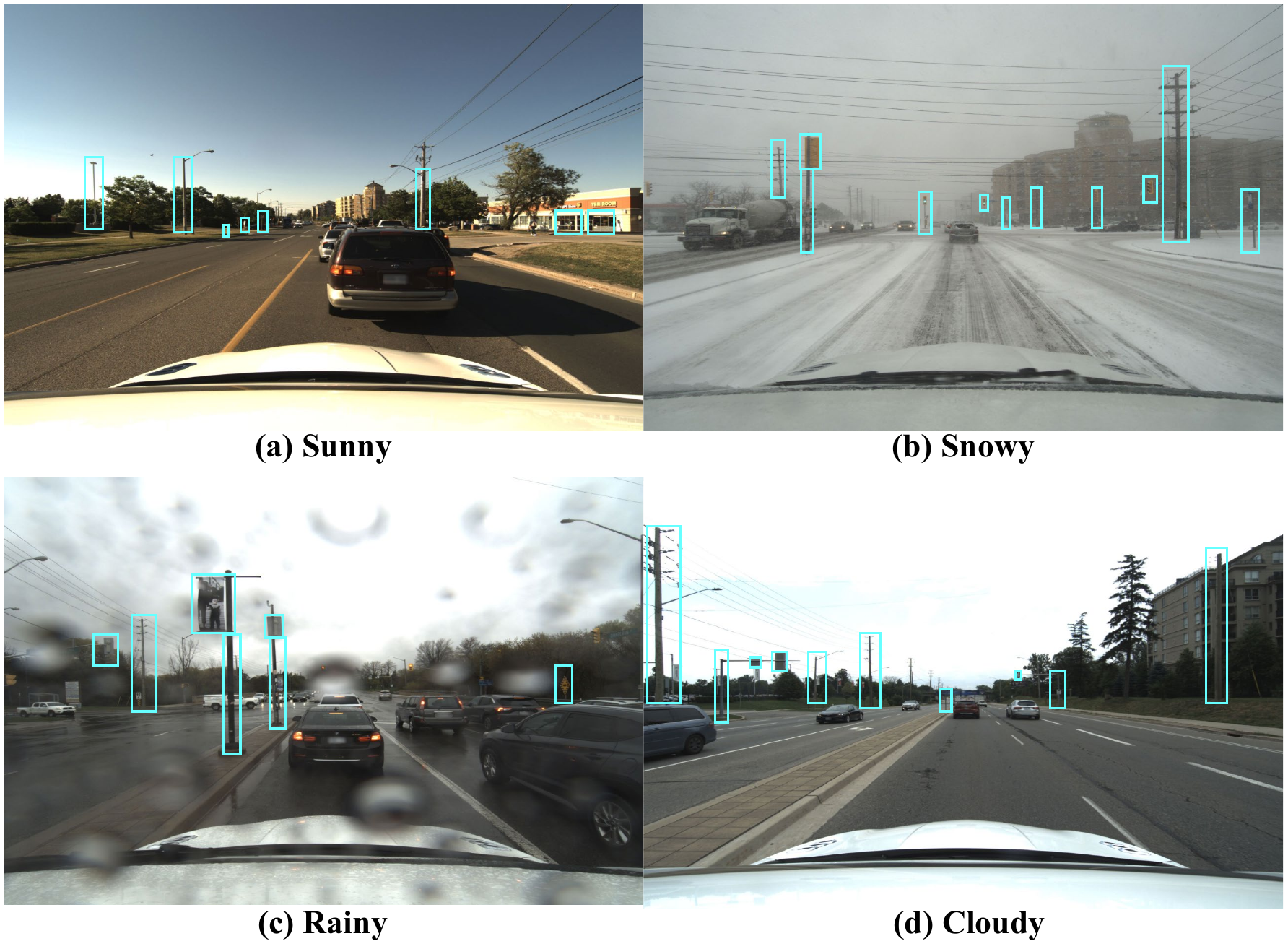}\hfill
\caption{
Examples of the ground truth landmark bounding box labels for the Boreas datasets under sunny, snowy, rainy and cloudy scenarios.
}
\label{fig:label_boreas}
\end{figure}

\textbf{Boreas Sunny (Snowy/Rainy/Cloudy) Landmark Patches.} 
To extract the Boreas Sunny (Snowy/Rainy/Cloudy) Landmark Patches, we follow a manual labeling process, similar to the procedure used for the Oxford dataset, to annotate landmarks such as traffic lights, traffic signs, poles, and windows. The Boreas dataset comprises $1500$ frames for the sunny scenario and $500$ frames each for the snowy, rainy, and cloudy scenarios, as illustrated in \cref{fig:label_boreas}.
For training and validation purposes, we utilize $1500$ frames from the sunny Boreas dataset, and for testing, $500$ frames each from the snowy, rainy, and cloudy Boreas datasets are employed. 
Employing operations similar to those used for the KITTI and Oxford datasets, we obtain the final landmark patches for the Boreas datasets under various scenarios, including sunny, snowy, rainy, and cloudy conditions. 

%%%%==========================================================================
\begin{table*}[!htb]
%\scriptsize
%\footnotesize
%\vspace{-0.8cm}
\caption{Matching performance comparison with the keypoint/pixel-level baselines containing neighborhood information. The test dataset is the KITTI Noisy Landmark Patches. The best and the second-best results under different metrics are highlighted in \textbf{bold} and \underline{underlined}, respectively.} 
\label{table-keypoint-neighbor}
% \vspace{-0.3cm}
\centering
%\newcommand{\tabincell}[2]{\begin{tabular}{@{}#1@{}}#2\end{tabular}}
%\resizebox{\textwidth}{18mm}{
\begin{tabular}{c|c|c|c|c}
\hline\hline
%{\bf Test Datasets} & 
{\bf Methods}  & {\bf Precision} & {\bf Recall} & {\bf $F_1$-Score} & {\bf AUC}
\\ \hline
%\multirow{5}{*}{KITTI} &
D2-Net (neighbor)              & \underline{0.9254} \scriptsize{$\pm$ 0.0032}  & 0.8011 \scriptsize{$\pm$ 0.0196} & 0.8587 \scriptsize{$\pm$ 0.0123}                         & 0.8037 \scriptsize{$\pm$ 0.0113} \\ 
% & 
SuperGlue (neighbor)           & 0.8711 \scriptsize{$\pm$ 0.0023}              & 0.8437 \scriptsize{$\pm$ 0.0216} & 0.8571 \scriptsize{$\pm$ 0.0115}                         & 0.7347 \scriptsize{$\pm$ 0.0080} \\ 
% &
LoFTR (neighbor)                & 0.9179 \scriptsize{$\pm$ 0.0020}              & \underline{0.8528} \scriptsize{$\pm$ 0.0072} & \underline{0.8841} \scriptsize{$\pm$ 0.0041}   & \underline{0.8120} \scriptsize{$\pm$ 0.0042} \\ %\hline
% & 
RobustMat (GAT-PDE) \&  Trained on Oxford
                   & {\textbf{0.9408}} \scriptsize{$\pm$ 0.0024}   & {\textbf{0.8560}} \scriptsize{$\pm$ 0.0174} & {\textbf{0.8963}} \scriptsize{$\pm$ 0.0101} & {\textbf{0.8472}} \scriptsize{$\pm$ 0.0093} \\ 
\hline\hline
%%------------------------------------------------------------------------------------------------
%\multirow{5}{*}{Oxford} & 
D2-Net (neighbor) + Denoising   
& 0.9229 \scriptsize{$\pm$ 0.0042}             & \underline{0.8757} \scriptsize{$\pm$ 0.0179}  & \underline{0.8987} \scriptsize{$\pm$ 0.0112} & \underline{0.8283} \scriptsize{$\pm$ 0.0127} \\ 
% &
SuperGlue (neighbor) + Denoising      
& 0.8976 \scriptsize{$\pm$ 0.0036}             & 0.8411 \scriptsize{$\pm$ 0.0075}              & 0.8684 \scriptsize{$\pm$ 0.0051}       
        & 0.7765 \scriptsize{$\pm$ 0.0072} \\ 
% &
LoFTR (neighbor) + Denoising                 
& \underline{0.9255} \scriptsize{$\pm$ 0.0012}  & 0.8619 \scriptsize{$\pm$ 0.0109}             & 0.8925 \scriptsize{$\pm$ 0.0057}       
        & 0.8269 \scriptsize{$\pm$ 0.0041} \\ %\hline
% &
RobustMat (GAT-PDE) + Denoising \& Trained on Oxford
& \textbf{0.9327} \scriptsize{$\pm$ 0.0037}  & \textbf{0.8869} \scriptsize{$\pm$ 0.0191} & \textbf{0.9091} \scriptsize{$\pm$ 0.0113} & \textbf{0.8475} \scriptsize{$\pm$ 0.0120}\\ 
%\tabincell{l}{* \\ *}
\hline\hline
\end{tabular}%}
% \vspace{-0.3cm}
\end{table*}

\subsection{Model Setting Details} \label[Appendix]{app.model_setting}
The detailed model setting of RobustMat is provided in \cref{model-setting}. 
The ODE solution package for the neural ODE and PDE modules is the ``odeint'' from the Python library ``torchdiffeq''\footnote{https://github.com/saheya/torchdiffeq}.

\subsection{More Experimental Results}\label[Appendix]{more_results}

In the main paper, we have compared our RobustMat with the baselines with keypoint/pixel level correspondence, including D2-Net, SuperGlue and LoFTR.  
However, since our method considers neighborhood information, we also introduce it into the keypoint-level baselines, including D2-Net, SuperGlue and LoFTR , on the KITTI and Oxford Noisy Landmark Patches. 
Specifically, for a given patch, we sort its neighbors according to increasing distances from it, where the distances are based on the coordinates of the central pixels of the patches. We then use each baseline method to compare not only the patch pair but also the pairs of their corresponding neighbors with the same sort order. Then, we calculate the average of the predicted scores for the patch pair and its neighbor pairs. Finally, we decide whether there is a match based on a threshold, which is a hyperparameter tuned separately to achieve the best performance for each baseline. 
From \cref{table-keypoint-neighbor}, we observe that our method still surpasses the other methods with neighborhood information. 
This implies that the neighborhood information has been used more efficiently in our method.

To assess the runtime performance, we conducted experiments on an NVIDIA RTX A5000 GPU to evaluate RobustMat. 
\cref{time-table} presents the inference runtime measured in seconds (s), denoting the mean time required for processing one pair of frames during the testing phase, for various RobustMat variants incorporating different graph diffusion modules. We also compare the inference runtime of our RobustMat with that of other baselines including SOLAR, DELG and CVNet. 
For each pair of landmark patches, the average inference runtime of our RobustMat is less than $0.2$ seconds. Although the inference runtime of our RobustMat is longer than that of the baselines, this achieved processing speed is considered acceptable for real-world applications, including place recognition and autonomous driving.

\begin{table}[!htb]
\caption{
Inference runtime of matching methods on Boreas Snowy Landmark Patches.
}
\label{time-table}
\vspace{-0.3cm}
\begin{center}
\newcommand{\tabincell}[2]{\begin{tabular}{@{}#1@{}}#2\end{tabular}}
\begin{tabular}{c|c}
\hline\hline 
{\bf Methods}      & \tabincell{c}{{\bf Inference} {\bf Runtime (s)} \\} \\
\hline
\tabincell{c}{SOLAR\\} & 0.0565 \\
\tabincell{c}{DELG\\}  & 0.0634 \\
\tabincell{c}{CVNet\\} & 0.0801 \\ 
\tabincell{c}{RobustMat (GCN-PDE) \\} & 0.1253 \\
\tabincell{c}{RobustMat (GAT-PDE) \\} & 0.1842 \\
\hline\hline 
\end{tabular}
\end{center}
%\vspace{-0.5cm}
\end{table}

\subsection{More Theoretical Basis for Loss}

Perturbations in the input lead to perturbations in the matched and unmatched conditional probability distributions. 
We next consider how the perturbations in these conditional probabilities influence $L_d$ in \cref{eq.L_ID_expect}. As the term $\hat{L}_d(f(x), g(\calG^y))$ in \cref{eq.Lce_ID} is symmetrical to $\hat{L}_d(f(y),g(\calG^x))$,
we discuss only $\hat{L}_d(f(x),g(\calG^y))$.

For simplicity, we assume that the pair $(f(x),g(\calG^y))$ are continuous variables randomly generated from a distribution $\P$ with probability density $p:\bbA \to \bbR_{+}$, where $(f(x),g(\mathcal{G}^{y}))$ take values in the set $\mathbb{A}$. We have $\int_{\mathbb{A}} p(f(x),g(\mathcal{G}^{y})) \ud (f(x),\allowbreak g(\mathcal{G}^{y}))=1$. 
 
Assume that $\P(x\leftrightarrow y)>0$ and $\P(x\nleftrightarrow y)>0$, \gls{wrt} any given landmark patch pairs $(x,y)$. 
The probability densities of $(f(x), g(\mathcal{G}^{y}))$ conditioned on $x \leftrightarrow y$ and $x \nleftrightarrow y$ are respectively denoted by  
\begin{align}
\pfgm & := \pfgm (f(x),g(\mathcal{G}^{y})) = p((f(x),g(\mathcal{G}^{y}))\mid {x \leftrightarrow y}), \\
\pfgu & := \pfgu(f(x),g(\mathcal{G}^{y})) = p((f(x),g(\mathcal{G}^{y}))\mid {x \nleftrightarrow y}).
\end{align}
The expectation of $\hat{L}_d(f(x),g(\calG^y))$ is given by 
\begin{align}
    & L_d\nn
    % & = \P(x\leftrightarrow y) \E[\log d(f(x),g(\mathcal{G}^y)) \given x \leftrightarrow y] \nn
    % & \quad + \P(x\nleftrightarrow y) \E[\log(1-d(f(x),g(\mathcal{G}^y))) \given x \nleftrightarrow y] \nn
    & = \int_{\mathbb{A}} \Big\{ \P(x\leftrightarrow y) \pfgm \log d(f(x),g(\mathcal{G}^y))\nn
    & \quad + \P(x\nleftrightarrow y) \pfgu \log(1-d(f(x), g(\mathcal{G}^{y}))) \Big\} \ud{(f(x),g(\mathcal{G}^y))}.
    \label{eq.L_ID_expect}
\end{align}
To minimize $\ell_{\mathrm{v\calG}}$ in the asymptotic regime $|\calM|\to\infty$ is equivalent to $\max\limits_{f,g,d} L_d$.

Let $\tildepfgm$ and  $\tildepfgu$ be in the $\epsilon$-neighborhoods (cf.\ ``An information-theoretic approach to universal feature selection in high-dimensional inference'') \gls{wrt} $\pfgm$ and $\pfgu$, respectively, i.e., 
\begin{align}
\tildepfgm %\nn
& = \pfgm + \epsilon_{x \leftrightarrow y}, \label{eq.tildepfgm} \\
\tildepfgu %\nn
& = \pfgu + \epsilon_{x \nleftrightarrow y}, \label{eq.tildepfgu}
\end{align}
where 
\begin{align}
& \epsilon_{x \leftrightarrow y} =  \epsilon \sqrt{\pfgm} \phi(f(x),g(\mathcal{G}^{y})), \\
& \epsilon_{x \nleftrightarrow y} = \epsilon \sqrt{\pfgu} \psi(f(x),g(\mathcal{G}^{y})), 
\end{align}
and $\phi(\cdot)$ and $\psi(\cdot)$ are perturbation functions so that 
\begin{align}
& \int_{\bbA}\epsilon_{x \leftrightarrow y} \ud (f(x), g(\calG^y))=0, \label{eq.sum_matched_eps}\\
& \int_{\bbA}\epsilon_{x \nleftrightarrow y} \ud (f(x), g(\calG^y))=0. \label{eq.sum_unmatched_eps}
\end{align}
If the conditional distributions are $\tildepfgm$ and $\tildepfgu$, we denote $L_d$ as ${\tilde L}_d$.

\begin{Theorem}[Effect of conditional probability perturbation]
\label{thm.disturbance_effect}
Assume the perturbed $\tildepfgm$ and $\tildepfgu$ are given by \cref{eq.tildepfgm} and \cref{eq.tildepfgu}. 
Then, we have 
\begin{align}
& |\max_{d} \allowbreak {\tilde L}_d -  \max_d L_d| 
=\calO(\epsilon).
\end{align}
\end{Theorem}
\begin{proof}
% See Appendix~\ref{app:thm.disturbance_effect}.
From \cref{eq.L_ID_expect}, $L_d$ is concave in $d(f(x), g(\mathcal{G}^{y}))$ for each $(x,y)$. By setting the first derivatives to zero, we obtain $d^*$ in \cref{eq.D*}.
Replacing $d$ in \cref{eq.L_ID_expect} with $d^{*}$, we obtain $L_{d^*}$.

Similarly, using $\tildepfgm$ and $\tildepfgu$ to replace $\pfgm$ and $\pfgu$ in $d^{*}(f(x), g(\mathcal{G}^{y}))$ and $L_{d^*}$ respectively, we obtain $\tilde{d}^{*}(f(x), g(\mathcal{G}^{y}))$ and ${\tilde L}_{\tilde{d}^*}$.
% , which are given by 
% 
% \begin{align}
%     {\tilde d}^{*}(f(x), g(\mathcal{G}^{y})) 
%     &= \frac{ {\P(x\leftrightarrow y)} \tildepfgm}
%     {{\tilde p}(f(x),g(\mathcal{G}^{y}))}, \label{eq.tilde_D*} 
% \end{align}
% and 
% \begin{align}
%     & {\tilde L}_{\tilde{d}^*}(f(x), g(\mathcal{G}^{y}))  \nn
%     & = \int_{\mathbb{A}} \Big\{ {\P(x\leftrightarrow y)}
%     \tildepfgm \cdot \log {\tilde{d}^*}(f(x),g(\mathcal{G}^y))
%     + {\P(x\nleftrightarrow y)} \tildepfgu \nn
%     & \qquad \cdot \log(1-{\tilde{d}^*}(f(x), g(\mathcal{G}^{y}))) \Big\}\ud (f(x),g(\mathcal{G}^y)),\label{eq.tilde_L_ID_original}
% \end{align}
% where 
% \begin{align}
% {\tilde p}(f(x),g(\mathcal{G}^{y})) = {\P(x\leftrightarrow y)} \tildepfgm + {\P(x\nleftrightarrow y)} \tildepfgu. 
% \end{align}

Using Taylor's series expansion, we substitute \cref{eq.tildepfgm} and \cref{eq.tildepfgu} into ${\tilde d}^{*}(f(x), g(\mathcal{G}^{y}))$ to obtain
\begin{align}
    & \log{\tilde d}^{*}(f(x),g(\mathcal{G}^y)) \nn
    % & = \log\big\{{\P(x\leftrightarrow y)}\pfgm 
    % + \epsilon {\P(x\leftrightarrow y)}\sqrt{\pfgm} \phi(f(x),g(\mathcal{G}^{y}))\big\} \nn
    % & \quad - \log\big\{p(f(x),g(\mathcal{G}^{y})) 
    % + \epsilon[{\P(x\leftrightarrow y)}\sqrt{\pfgm} \phi(f(x),g(\mathcal{G}^{y})) \nn
    % & \qquad + {\P(x\nleftrightarrow y)}\sqrt{\pfgu} \psi(f(x),g(\mathcal{G}^{y}))] \big\} \nn
    & = \log\bigg\{\frac{{\P(x\leftrightarrow y)}\pfgm}{p(f(x),g(\mathcal{G}^{y}))}\bigg\}
    + \epsilon \bigg\{ [\pfgm]^{-\frac{1}{2}}\phi(f(x),g(\mathcal{G}^{y})) \nn
    & \qquad - \frac{{\P(x\leftrightarrow y)}\sqrt{\pfgm} \phi(f(x),g(\mathcal{G}^{y}))}{p(f(x),g(\mathcal{G}^{y}))} \nn
    & \qquad - \frac{{\P(x\nleftrightarrow y)}\sqrt{\pfgu}\psi(f(x),g(\mathcal{G}^{y}))}{p(f(x),g(\mathcal{G}^{y}))}\bigg\} 
    + \calO(\epsilon^2),\label{eq.log_tilded*} \\
    &\log(1-{\tilde d}^{*}(f(x),g(\mathcal{G}^y))) \nn
    % & = \log\big\{{\P(x\nleftrightarrow y)}\pfgu  
    %  + \epsilon {\P(x\nleftrightarrow y)}\sqrt{\pfgu} \psi(f(x),g(\mathcal{G}^{y}))\big\} \nn
    % & \quad - \log\big\{p(f(x),g(\mathcal{G}^{y}))
    %  + \epsilon[{\P(x\leftrightarrow y)}\sqrt{\pfgm} \phi(f(x),g(\mathcal{G}^{y})) \nn
    % & \qquad + {\P(x\nleftrightarrow y)}\sqrt{\pfgu} \psi(f(x),g(\mathcal{G}^{y}))] \big\} \nn
    & = \log\bigg\{\frac{{\P(x\nleftrightarrow y)}\pfgu}{p(f(x),g(\mathcal{G}^{y}))}\bigg\}
    + \epsilon \bigg\{ [\pfgu]^{-\frac{1}{2}}\psi(f(x),g(\mathcal{G}^{y})) \nn
    & \qquad - \frac{{\P(x\leftrightarrow y)}\sqrt{\pfgm} \phi(f(x),g(\mathcal{G}^{y}))}{p(f(x),g(\mathcal{G}^{y}))} \nn
    & \qquad - \frac{{\P(x\nleftrightarrow y)}\sqrt{\pfgu}\psi(f(x),g(\mathcal{G}^{y}))}{p(f(x),g(\mathcal{G}^{y}))}\bigg\}
    + \calO(\epsilon^2). \label{eq.log_1-tilded*}
\end{align} 

Substituting \cref{eq.log_tilded*} and \cref{eq.log_1-tilded*} into ${\tilde L}_{\tilde{d}^*}$, we have 
\begin{align}
    & {\tilde L}_{{\tilde d}^{*}} - L_{d^*} \nn
    & = \epsilon \int_{\mathbb{A}}\bigg\{ {\P(x\leftrightarrow y)}\sqrt{\pfgm} \phi(f(x),g(\mathcal{G}^{y})) \nn
    & \qquad \cdot \log\bigg( \frac{{\P(x\leftrightarrow y)}\pfgm}{p(f(x),g(\mathcal{G}^{y}))} \bigg) \nn
    & \quad + {\P(x\nleftrightarrow y)}\sqrt{\pfgu} \psi(f(x),g(\mathcal{G}^{y})) \nn
    & \qquad \log\bigg( \frac{{\P(x\nleftrightarrow y)}\pfgu}{p(f(x),g(\mathcal{G}^{y}))} \bigg) 
    \bigg\} \ud (f(x),g(\mathcal{G}^y)) 
    + \calO(\epsilon^2), 
\end{align}
and the proof is completed.
\end{proof}

\begin{Remark}
As the neural diffusion modules introduce limited noise in the feature representations, the perturbations in the matched and unmatched conditional probabilities are also controlled.
From \cref{thm.disturbance_effect}, we conclude that the perturbation in the loss is also controlled, which leads to more accurate feature learning. 
\end{Remark}

To minimize $\ell_{\mathrm{v\calG}}$ in \cref{eq.Lce_ID}, we seek to maximize $L_d$. The following result provides an upper bound for $L_d$. 
\begin{Theorem}[Upper bound for $L_d$.]
\label{thm.optimal_solution}
For all discriminators $d$, we have 
\begin{align}
    L_d \le 0.
\end{align}
A sufficient condition for equality to hold is when $\pfgm$ and $\pfgu$ are mutually singular measures and the optimal discriminator given by
\begin{align}
    & d^{*}(f(x), g(\mathcal{G}^{y})) 
    = \frac{\P(x\leftrightarrow y) \pfgm}
    {p(f(x),g(\mathcal{G}^{y}))} \label{eq.D*} 
\end{align}
is used.
\end{Theorem}

\begin{proof}
% See Appendix~\ref{app:thm.optimal_solution}.
With $d^*$ the optimal discriminator given by \cref{eq.D*}, we have $\max_{\{f,g,d\}} L_d=\max_{\{f,g\}} L_{d^*}$, 
where $L_{d^*}$ is obtained by substituting \cref{eq.D*} into \cref{eq.L_ID_expect}:
\begin{align}
     L_{d^*} 
    & = \P(x\leftrightarrow y)
    \E[\log \frac{\pfgm}{ p(f(x),g(\mathcal{G}^{y}))} \given x \leftrightarrow y] \nn
    & \quad + \P(x\nleftrightarrow y)
    \E[\log \frac{\pfgu}{p(f(x),g(\mathcal{G}^{y}))} \given x \nleftrightarrow y] \nn
    & \quad + \P(x\leftrightarrow y)\log\P(x\leftrightarrow y) \nn
    & \quad + \P(x\nleftrightarrow y)\log\P(x\nleftrightarrow y). \label{eq.Ld*_standard} 
\end{align}

A $\cal F$-divergence is written as $D_{\cal F}(p(x) \| q(x)) \allowbreak = \int q(x) \allowbreak {\cal F}(\frac{p(x)}{q(x)}) \ud x$, where ${\cal F}(\cdot)$ is a convex and lower-semicontinuous function satisfying ${\cal F}(1)=0$.
The first term of \cref{eq.Ld*_standard} can be expressed as 
\begin{align}
    \P(x\leftrightarrow y) D_{\cal F_{S}}(\pfgu \| \pfgm),
\end{align}
where 
\begin{align}
    & D_{\cal F_{S}}(\pfgu \| \pfgm) \nn
    & = \int_{\bbA} \pfgm \Big\{ - \log\Big( \P(x\leftrightarrow y) \nn
    & \qquad + \P(x\nleftrightarrow y) \frac{\pfgu}{\pfgm}  \Big)  \Big\} \ud (f(x),g(\calG^y)), 
\end{align}
and the function ${\cal F_{S}}(\cdot)$ is given by
\begin{align}
    {\cal F_{S}}(u) = -\log(\P(x\leftrightarrow y) + \P(x\nleftrightarrow y) u), 
\end{align}
as well as $u = \frac{\pfgu}{\pfgm}$.

From the reference (titled ``Some bounds for skewed $\alpha$-Jensen-Shannon divergence''), 
$D_{\cal F}(p(x)\|q(x)) \le \lim_{t \to 0_{+}}({\cal F}(t)+{\cal F}^*(t))$, where ${\cal F}^*(t)$ is the conjugate function of ${\cal F}(t)$ and ${\cal F}^*(t)=t{\cal F}(\frac{1}{t})$. 
According to L'Hospital's rule, it is readily seen that $\lim_{u\to 0_{+}}{\cal F_{S}}^*(u) = {\cal F_{S}}'(\infty) = 0$. 
As a result, we have
\begin{align}
    & D_{\cal F_{S}}(\pfgu \| \pfgm) \le -\log(\P(x\leftrightarrow y)).
\end{align}

Similarly, we have 
\begin{align}
    & D_{\cal F_{S}}(\pfgm \| \pfgu) \le -\log(\P(x\nleftrightarrow y)).
\end{align}
Therefore,
\begin{align}
    L_{d^*}
    & \le \P(x\leftrightarrow y) \{-\log(\P(x\leftrightarrow y)) \} \nn
    & \quad + \P(x\nleftrightarrow y) \{-\log(\P(x\nleftrightarrow y)) \} \nn 
    & \quad + \P(x\leftrightarrow y)\log\P(x\leftrightarrow y) + \P(x\nleftrightarrow y)\log\P(x\nleftrightarrow y) \nn
    & \quad = 0, \label{eq.upper_bound}
\end{align}
which yields $ L_d \le L_{d^*} \le 0$. 

If $\pfgm$ and $\pfgu$ are mutually singular measures, we have
\begin{align}
    & L_{d^*} \nn
    & = \P(x\leftrightarrow y)  \log\frac{1}{\P(x\leftrightarrow y)} \int_{\mathbb{A}} \pfgm \ud (f(x),\allowbreak g(\mathcal{G}^{y})) \nn   
    & \qquad + \P(x\nleftrightarrow y) \log\frac{1}{\P(x\nleftrightarrow y)} \int_{\mathbb{A}} \pfgu \ud (f(x),\allowbreak g(\mathcal{G}^{y})) \nn
    & \qquad + \P(x\leftrightarrow y)\log\P(x\leftrightarrow y) + \P(x\nleftrightarrow y)\log\P(x\nleftrightarrow y) \nn
    & = \P(x\leftrightarrow y)  \log\frac{1}{\P(x\leftrightarrow y)}  
    + \P(x\nleftrightarrow y) \log\frac{1}{\P(x\nleftrightarrow y)}  \nn
    & \qquad + \P(x\leftrightarrow y)\log\P(x\leftrightarrow y) + \P(x\nleftrightarrow y)\log\P(x\nleftrightarrow y) \nn
    & = 0,
\end{align}
which achieves the upper bound in \cref{eq.upper_bound}. The proof is now complete.
\end{proof}

\Cref{thm.optimal_solution} gives the form of the ideal discriminator. However, the probability density $p$ is unknown a priori. Therefore, we propose to learn it using the structure given by \cref{eq:ddd}.
\Cref{thm.optimal_solution} indicates that in the \emph{ideal case}, the matched and unmatched conditional distributions are mutually singular. In other words, $(f(x), g(\calG^y))$ cannot be realized under both the matched and unmatched cases simultaneously.

\vfill

\end{document}

%% file: preamble.tex
\usepackage[T1]{fontenc}
\usepackage[utf8]{inputenc}
\usepackage{mathtools}

\usepackage{amssymb,mathrsfs}% Typical maths resource packages
\usepackage{amsthm}
\usepackage{bm}
\usepackage{scalerel}
\usepackage{nicefrac}
\usepackage{microtype} 
\usepackage[shortlabels]{enumitem}
\usepackage{graphicx}
\usepackage{epstopdf}
\DeclareGraphicsExtensions{.eps,.png,.jpg,.pdf}

\usepackage{url}
\usepackage{colortbl}
\usepackage{booktabs}
\usepackage{multirow}
\usepackage{colortbl,xcolor}
\usepackage[normalem]{ulem}
\usepackage{xparse,xstring}
\usepackage{calc}
\usepackage{etoolbox}

\makeatletter
\@ifpackageloaded{natbib}{
	\relax
}{
	\usepackage{cite}
}
\makeatother

\usepackage{array}
\newcolumntype{L}[1]{>{\raggedright\let\newline\\\arraybackslash\hspace{0pt}}m{#1}}
\newcolumntype{C}[1]{>{\centering\let\newline\\\arraybackslash\hspace{0pt}}m{#1}}
\newcolumntype{R}[1]{>{\raggedleft\let\newline\\\arraybackslash\hspace{0pt}}m{#1}}

\makeatletter
\let\MYcaption\@makecaption
\makeatother
\usepackage[font=footnotesize]{subcaption}
\makeatletter
\let\@makecaption\MYcaption
\makeatother

\usepackage{glossaries}
\makeatletter
\sfcode`\.1006

\let\oldgls\gls
\let\oldglspl\glspl

\newcommand\fussy@ifnextchar[3]{%
	\let\reserved@d=#1%
	\def\reserved@a{#2}%
	\def\reserved@b{#3}%
	\futurelet\@let@token\fussy@ifnch}
\def\fussy@ifnch{%
	\ifx\@let@token\reserved@d
		\let\reserved@c\reserved@a
	\else
		\let\reserved@c\reserved@b
	\fi
	\reserved@c}

\renewcommand{\gls}[1]{%
\oldgls{#1}\fussy@ifnextchar.{\@checkperiod}{\@}}
\renewcommand{\glspl}[1]{%
\oldglspl{#1}\fussy@ifnextchar.{\@checkperiod}{\@}}

\newcommand{\@checkperiod}[1]{%
	\ifnum\sfcode`\.=\spacefactor\else#1\fi
}

\robustify{\gls}
\robustify{\glspl}
\makeatother

\newacronym{wrt}{w.r.t.}{with respect to}
\newacronym{RHS}{R.H.S.}{right-hand side}
\newacronym{LHS}{L.H.S.}{left-hand side}
\newacronym{iid}{i.i.d.}{independent and identically distributed}
\newacronym{SOTA}{SOTA}{state-of-the-art}
\usepackage{float}

\ifx\notloadhyperref\undefined
	\ifx\loadbibentry\undefined
		\usepackage[hidelinks,hypertexnames=false]{hyperref}
	\else
		\usepackage{bibentry}
		\makeatletter\let\saved@bibitem\@bibitem\makeatother
		\usepackage[hidelinks,hypertexnames=false]{hyperref}
		\makeatletter\let\@bibitem\saved@bibitem\makeatother
	\fi
\else
	\ifx\loadbibentry\undefined
		\relax
	\else
		\usepackage{bibentry}
	\fi
\fi

\usepackage[capitalize]{cleveref}
\crefname{equation}{}{}
\Crefname{equation}{}{}
\crefname{claim}{claim}{claims}
\crefname{step}{step}{steps}
\crefname{line}{line}{lines}
\crefname{condition}{condition}{conditions}
\crefname{dmath}{}{}
\crefname{dseries}{}{}
\crefname{dgroup}{}{}

\crefname{Problem}{Problem}{Problems}
\crefformat{Problem}{Problem~#2#1#3}
\crefrangeformat{Problem}{Problems~#3#1#4 to~#5#2#6}

\crefname{Theorem}{Theorem}{Theorems}
\crefname{Corollary}{Corollary}{Corollaries}
\crefname{Proposition}{Proposition}{Propositions}
\crefname{Lemma}{Lemma}{Lemmas}
\crefname{Definition}{Definition}{Definitions}
\crefname{Example}{Example}{Examples}
\crefname{Assumption}{Assumption}{Assumptions}
\crefname{Remark}{Remark}{Remarks}
\crefname{Rem}{Remark}{Remarks}
\crefname{remarks}{Remarks}{Remarks}
\crefname{Appendix}{Appendix}{Appendices}
\crefname{Supplement}{Supplement}{Supplements}
\crefname{Exercise}{Exercise}{Exercises}
\crefname{Theorem_A}{Theorem}{Theorems}
\crefname{Corollary_A}{Corollary}{Corollaries}
\crefname{Proposition_A}{Proposition}{Propositions}
\crefname{Lemma_A}{Lemma}{Lemmas}
\crefname{Definition_A}{Definition}{Definitions}

\usepackage{crossreftools}
\ifx\notloadhyperref\undefined
\else
	\relax
\fi

\usepackage{algorithm}
\usepackage{algpseudocode}

\ifx\loadbreqn\undefined
	\relax
\else
	\usepackage{breqn}
\fi

\makeatletter
\def\cleartheorem#1{%
    \expandafter\let\csname#1\endcsname\relax
    \expandafter\let\csname c@#1\endcsname\relax
}
\def\clearthms#1{ \@for\tname:=#1\do{\cleartheorem\tname} }
\makeatother

\ifx\renewtheorem\undefined
	\ifx\useTheoremCounter\undefined
		\newtheorem{Theorem}{Theorem}
		\newtheorem{Corollary}{Corollary}
		\newtheorem{Proposition}{Proposition}
		
	\else
		\newtheorem{Theorem}{Theorem}

	\fi

	\newtheorem{Remark}{Remark}
	\newtheorem{Assumption}{Assumption}

\fi

\theoremstyle{remark}

\theoremstyle{plain}

\newcommand{\qednew}{\nobreak \ifvmode \relax \else
		\ifdim\lastskip<1.5em \hskip-\lastskip
			\hskip1.5em plus0em minus0.5em \fi \nobreak
		\vrule height0.75em width0.5em depth0.25em\fi}



\newcommand{\nn}{\nonumber\\ }

\NewDocumentCommand{\movedownsub}{e{^_}}{%
	\IfNoValueTF{#1}{%
		\IfNoValueF{#2}{^{}}% neither ^ nor _, do nothing; if no ^ but _, add ^{}
	}{%
		^{#1}% add superscript if present
	}%
	\IfNoValueF{#2}{_{#2}}% add subscript if present
}

\let\latexchi\chi
\RenewDocumentCommand{\chi}{}{\latexchi\movedownsub}

\newcommand{\Real}{\mathbb{R}}

\newcommand{\calE}{\mathcal{E}}

\newcommand{\calG}{\mathcal{G}}

\newcommand{\calL}{\mathcal{L}}
\newcommand{\calM}{\mathcal{M}}

\newcommand{\calO}{\mathcal{O}}

\newcommand{\calV}{\mathcal{V}}

\newcommand{\bbA}{\mathbb{A}}

\newcommand{\bbR}{\mathbb{R}}

\DeclareSymbolFont{bsfletters}{OT1}{cmss}{bx}{n}
\DeclareSymbolFont{ssfletters}{OT1}{cmss}{m}{n}
\DeclareMathSymbol{\bsfGamma}{0}{bsfletters}{'000}
\DeclareMathSymbol{\ssfGamma}{0}{ssfletters}{'000}
\DeclareMathSymbol{\bsfDelta}{0}{bsfletters}{'001}
\DeclareMathSymbol{\ssfDelta}{0}{ssfletters}{'001}
\DeclareMathSymbol{\bsfTheta}{0}{bsfletters}{'002}
\DeclareMathSymbol{\ssfTheta}{0}{ssfletters}{'002}
\DeclareMathSymbol{\bsfLambda}{0}{bsfletters}{'003}
\DeclareMathSymbol{\ssfLambda}{0}{ssfletters}{'003}
\DeclareMathSymbol{\bsfXi}{0}{bsfletters}{'004}
\DeclareMathSymbol{\ssfXi}{0}{ssfletters}{'004}
\DeclareMathSymbol{\bsfPi}{0}{bsfletters}{'005}
\DeclareMathSymbol{\ssfPi}{0}{ssfletters}{'005}
\DeclareMathSymbol{\bsfSigma}{0}{bsfletters}{'006}
\DeclareMathSymbol{\ssfSigma}{0}{ssfletters}{'006}
\DeclareMathSymbol{\bsfUpsilon}{0}{bsfletters}{'007}
\DeclareMathSymbol{\ssfUpsilon}{0}{ssfletters}{'007}
\DeclareMathSymbol{\bsfPhi}{0}{bsfletters}{'010}
\DeclareMathSymbol{\ssfPhi}{0}{ssfletters}{'010}
\DeclareMathSymbol{\bsfPsi}{0}{bsfletters}{'011}
\DeclareMathSymbol{\ssfPsi}{0}{ssfletters}{'011}
\DeclareMathSymbol{\bsfOmega}{0}{bsfletters}{'012}
\DeclareMathSymbol{\ssfOmega}{0}{ssfletters}{'012}

\makeatletter
\newcommand*\rel@kern[1]{\kern#1\dimexpr\macc@kerna}
\newcommand*\widebar[1]{%
  \begingroup
  \def\mathaccent##1##2{%
    \rel@kern{0.8}%
    \overline{\rel@kern{-0.8}\macc@nucleus\rel@kern{0.2}}%
    \rel@kern{-0.2}%
  }%
  \macc@depth\@ne
  \let\math@bgroup\@empty \let\math@egroup\macc@set@skewchar
  \mathsurround\z@ \frozen@everymath{\mathgroup\macc@group\relax}%
  \macc@set@skewchar\relax
  \let\mathaccentV\macc@nested@a
  \macc@nested@a\relax111{#1}%
  \endgroup
}
\makeatother

\DeclareMathOperator{\var}{var}

\DeclareMathOperator{\cov}{cov}

\newcommand{\ifbcdot}[1]{\ifblank{#1}{\cdot}{#1}}

\DeclarePairedDelimiterX\abs[1]{\lvert}{\rvert}{\ifbcdot{#1}}
\DeclarePairedDelimiterX\parens[1]{(}{)}{\ifbcdot{#1}}
\DeclarePairedDelimiterX\brk[1]{[}{]}{\ifbcdot{#1}}
\DeclarePairedDelimiterX\braces[1]{\{}{\}}{\ifbcdot{#1}}
\DeclarePairedDelimiterX\angles[1]{\langle}{\rangle}{\ifblank{#1}{\cdot,\cdot}{#1}}
\DeclarePairedDelimiterX\ip[2]{\langle}{\rangle}{\ifbcdot{#1},\ifbcdot{#2}}
\DeclarePairedDelimiterX\norm[1]{\lVert}{\rVert}{\ifbcdot{#1}}
\DeclarePairedDelimiterX\ceil[1]{\lceil}{\rceil}{\ifbcdot{#1}}
\DeclarePairedDelimiterX\floor[1]{\lfloor}{\rfloor}{\ifbcdot{#1}}

\DeclareFontFamily{U}{matha}{\hyphenchar\font45}
\DeclareFontShape{U}{matha}{m}{n}{
      <5> <6> <7> <8> <9> <10> gen * matha
      <10.95> matha10 <12> <14.4> <17.28> <20.74> <24.88> matha12
      }{}
\DeclareSymbolFont{matha}{U}{matha}{m}{n}
\DeclareFontSubstitution{U}{matha}{m}{n}

\DeclareFontFamily{U}{mathx}{\hyphenchar\font45}
\DeclareFontShape{U}{mathx}{m}{n}{
      <5> <6> <7> <8> <9> <10>
      <10.95> <12> <14.4> <17.28> <20.74> <24.88>
      mathx10
      }{}
\DeclareSymbolFont{mathx}{U}{mathx}{m}{n}
\DeclareFontSubstitution{U}{mathx}{m}{n}

\DeclareMathDelimiter{\vvvert}{0}{matha}{"7E}{mathx}{"17}
\DeclarePairedDelimiterX\vertiii[1]{\vvvert}{\vvvert}{\ifbcdot{#1}}

\DeclarePairedDelimiterXPP\trace[1]{\operatorname{Tr}}{(}{)}{}{\ifbcdot{#1}} % column vector
\DeclarePairedDelimiterXPP\col[1]{\operatorname{col}}{\{}{\}}{}{\ifbcdot{#1}} % column vector
\DeclarePairedDelimiterXPP\row[1]{\operatorname{row}}{\{}{\}}{}{\ifbcdot{#1}} % row vector
\DeclarePairedDelimiterXPP\erf[1]{\operatorname{erf}}{(}{)}{}{\ifbcdot{#1}}
\DeclarePairedDelimiterXPP\erfc[1]{\operatorname{erfc}}{(}{)}{}{\ifbcdot{#1}}
\DeclarePairedDelimiterXPP\KLD[2]{D}{(}{)}{}{\ifbcdot{#1}\, \delimsize\|\, \ifbcdot{#2}} % KL divergence
\DeclarePairedDelimiterXPP\op[2]{\operatorname{#1}}{(}{)}{}{#2} % general operator

\newcommand{\ud}{\,\mathrm{d}} % for integrals like \int f(x) \ud x
\DeclarePairedDelimiterXPP\indicate[1]{{\bf 1}}{\{}{\}}{}{\ifbcdot{#1}}
\newcommand{\indicator}[1]{{\bf 1}_{\braces*{\ifbcdot{#1}}}}

\newcommand{\tc}[1]{^{(#1)}}
\NewDocumentCommand\ofrac{s m}{%
	\IfBooleanTF#1%
	{\dfrac{1}{#2}}%
	{\frac{1}{#2}}%
}
\NewDocumentCommand\ddfrac{s m m}{%
	\IfBooleanTF#1%
	{\dfrac{\mathrm{d} {#2}}{\mathrm{d} {#3}}}%
	{\frac{\mathrm{d} {#2}}{\mathrm{d} {#3}}}%
}
\NewDocumentCommand\ppfrac{s m m}{%
	\IfBooleanTF#1%
	{\dfrac{\partial {#2}}{\partial {#3}}}%
	{\frac{\partial {#2}}{\partial {#3}}}%
}

\providecommand\given{}
\DeclarePairedDelimiterX\Set[2]\{\}{%
\renewcommand\given{\SetSymbol[\delimsize]{#1}}
#2
}
\DeclarePairedDelimiterX\Setc[1]\{\}{%
\renewcommand\given{\SetSymbol{:}}
#1
}

\NewDocumentCommand\set{s o m}{%
	\IfBooleanTF#1%
	{\IfValueTF{#2}{\Set*{#2}{#3}}{\Setc*{#3}}}%
	{\IfValueTF{#2}{\Set{#2}{#3}}{\Setc{#3}}}%
}

\NewDocumentCommand{\evalat}{ s O{\big} m e{_^} }{%
\IfBooleanTF{#1}%
{\left. #3 \right|}{#3#2|}%
\IfValueT{#4}{_{#4}}%
\IfValueT{#5}{^{#5}}%
}

\providecommand\given{}
\DeclarePairedDelimiterXPP\cprob[1]{}(){}{
\renewcommand\given{\nonscript\,\delimsize\vert\allowbreak\nonscript\,\mathopen{}}%
#1%
}
\DeclarePairedDelimiterXPP\cexp[1]{}[]{}{
\renewcommand\given{\nonscript\,\delimsize\vert\allowbreak\nonscript\,\mathopen{}}%
#1%
}

\DeclareDocumentCommand \P { s e{_^} d() g } {%
	\mathbb{P}%
	\IfBooleanTF{#1}%
		{
			\IfValueT{#2}{_{#2}}%
			\IfValueT{#3}{^{#3}}%
			\IfValueTF{#5}{\cprob{#4 \given #5}}{\IfValueT{#4}{\cprob{#4}}}%
		}%
		{
			\IfValueT{#2}{_{#2}}%
			\IfValueT{#3}{^{#3}}%
			\IfValueTF{#5}{\cprob*{#4 \given #5}}{\IfValueT{#4}{\cprob*{#4}}}%
		}%
}

\DeclareDocumentCommand \E { s e{_^} o g } {%
	\mathbb{E}%
	\IfBooleanTF{#1}%
		{
			\IfValueT{#2}{_{#2}}%
			\IfValueT{#3}{^{#3}}%
			\IfValueTF{#5}{\cexp{#4 \given #5}}{\IfValueT{#4}{\cexp{#4}}}%
		}%
		{
			\IfValueT{#2}{_{#2}}%
			\IfValueT{#3}{^{#3}}%	
			\IfValueTF{#5}{\cexp*{#4 \given #5}}{\IfValueT{#4}{\cexp*{#4}}}%		
		}%
}

\DeclareDocumentCommand \Var { s e{_^} d() g } {%
	\var%
	\IfBooleanTF{#1}%
		{
			\IfValueT{#2}{_{#2}}%
			\IfValueT{#3}{^{#3}}%
			\IfValueTF{#5}{\cprob{#4 \given #5}}{\IfValueT{#4}{\cprob{#4}}}%
		}%
		{
			\IfValueT{#2}{_{#2}}%
			\IfValueT{#3}{^{#3}}%	
			\IfValueTF{#5}{\cprob*{#4 \given #5}}{\IfValueT{#4}{\cprob*{#4}}}%		
		}%
}

\DeclareDocumentCommand \Cov { s e{_^} d() g } {%
	\cov%
	\IfBooleanTF{#1}%
		{
			\IfValueT{#2}{_{#2}}%
			\IfValueT{#3}{^{#3}}%
			\IfValueTF{#5}{\cprob{#4 \given #5}}{\IfValueT{#4}{\cprob{#4}}}%
		}%
		{
			\IfValueT{#2}{_{#2}}%
			\IfValueT{#3}{^{#3}}%	
			\IfValueTF{#5}{\cprob*{#4 \given #5}}{\IfValueT{#4}{\cprob*{#4}}}%		
		}%
}

\ExplSyntaxOn
\NewDocumentCommand \dist {m o o} {%
\mathrm{#1}\left(%
	\IfValueT{#3}{%
		\tl_if_blank:nTF{ #3 }{\cdot\, \middle|\, }{#3\, \middle|\, }%
	}
	\IfValueT{#2}{#2}%
\right)%
}
\ExplSyntaxOff

\NewDocumentCommand {\cbrace} {t+ D[]{black} D(){\widthof{#5}} m m } {%
	\begingroup%
		\color{#2}
		\IfBooleanTF{#1}{%
			\overbrace{#4}^%
		}{
			\underbrace{#4}_%
		}%
		{\parbox[c]{#3}{\centering\footnotesize{#5}}}%
	\endgroup% 
}

\let\oldforall\forall
\renewcommand{\forall}{\oldforall \, }

\let\oldexist\exists
\renewcommand{\exists}{\oldexist \, }

\makeatletter

\newcommand{\rankcolor}[2]{%
	\expandafter\renewcommand\csname #1\endcsname[1]{%
		\ifblank{##1}{%
			{\color{#2} \textbf{#2}}%
		}{%
			\ifmmode
				\textcolor{#2}{\bm{##1}}%
			\else%
				{\color{#2} \textbf{##1}}%
			\fi	
		}%
	}
}

\rankcolor{first}{red}
\rankcolor{second}{blue}
\rankcolor{third}{cyan}
\makeatother

\graphicspath{{./Figures/}{./figures/}}
\DeclareDocumentCommand{\includeCroppedPdf}{ o O{./Figures/} m }{
	\IfFileExists{#2#3-crop.pdf}{}{%
		\immediate\write18{pdfcrop #2#3.pdf #2#3-crop.pdf}}%
	\includegraphics[#1]{#2#3-crop.pdf}
}

\makeatletter
\newcommand*{\addFileDependency}[1]{% argument=file name and extension
  \typeout{(#1)}
  \@addtofilelist{#1}
  \IfFileExists{#1}{}{\typeout{No file #1.}}
}
\makeatother

\definecolor{gray90}{gray}{0.9}
\def\colorlist{red,blue,brown,cyan,darkgray,gray,lightgray,green,lime,magenta,olive,orange,pink,purple,teal,violet,white,yellow}

\makeatletter
\def\startcomment{[}
\ifx\nohighlights\undefined
	\newcommand{\createcolor}[1]{%
			\expandafter\newcommand\csname #1\endcsname[1]{{\color{#1} ##1}}%
	}
	\newcommand{\msout}[1]{\text{\color{green} \sout{\ensuremath{#1}}}}
	\newcommand{\del}[1]{{\color{green}\ifmmode \msout{#1}\else\sout{#1}\fi}}
\else
	\newcommand{\createcolor}[1]{%
			\expandafter\newcommand\csname #1\endcsname[1]{%
				\noexpandarg%
				\StrChar{##1}{1}[\firstletter]%
				\if\firstletter\startcomment%
					\relax
				\else%
					##1
				\fi
			}%
	}
	\newcommand{\msout}[1]{}
	\newcommand{\del}[1]{}
\fi

\def\@tempa#1,{%
    \ifx\relax#1\relax\else
        \createcolor{#1}%
        \expandafter\@tempa
    \fi
}
\expandafter\@tempa\colorlist,\relax,
\makeatother

\newcommand{\hhide}[1]{}
\ifx\diagnoselabel\undefined
	\relax
\else
	\makeatletter
	\def\@testdef #1#2#3{%
		\def\reserved@a{#3}\expandafter \ifx \csname #1@#2\endcsname
			\reserved@a  \else
			\typeout{^^Jlabel #2 changed:^^J%
				\meaning\reserved@a^^J%
				\expandafter\meaning\csname #1@#2\endcsname^^J}%
			\@tempswatrue \fi}
	\makeatother
\fi